\pdfoutput=1
\documentclass[journal]{IEEEtran}
\usepackage{amsmath,amssymb, amsthm}

\usepackage{epstopdf}

\usepackage{subcaption}
\usepackage{amsfonts}
\usepackage{verbatim}
\usepackage{tikz}
\usetikzlibrary{decorations.pathreplacing,calc}
\usepackage{enumitem}

\usepackage{layout}
\usepackage{amsmath, bm}
\usepackage{amsthm}
\usepackage{amssymb} 
\usepackage{url}
\usepackage{color}

\usepackage{algorithmicx}
\usepackage{algorithm}
\usepackage{algpseudocode}

\usepackage{epsfig}
\usepackage{caption}
\usepackage{stfloats}

\usepackage{url}
\usepackage{breakurl}

\newcommand{\stepsize}{ h}
\newcommand{\decrease}{\rho}
\newcommand{\primal}{P}
\newcommand{\z}{{z}}
\newcommand{\x}{{x}}
\newcommand{\y}{{y}}

\newcommand{\setn}{[n]}
\newcommand{\nz}{\text{nnz}}
\newcommand{\proxl}{\text{prox}^{l}}

\newcommand\tagthis{\addtocounter{equation}{1}\tag{\theequation}}

\usepackage[normalem]{ulem}

\newcommand{\add}[1]{#1}

\newcommand{\eqdef}{\stackrel{\text{def}}{=}}

\newcommand{\R}{\mathbb{R}}

\newcommand{\vc}[2]{#1^{#2}}

\DeclareMathOperator{\prox}{prox}         

\DeclareMathOperator{\Exp}{\mathbf{E}}           

\DeclareMathOperator{\dom}{dom}         

\DeclareMathOperator{\support}{support}

\newtheorem{assumption}{Assumption}
\newtheorem{lemma}{Lemma}

\newtheorem{theorem}{Theorem}

\theoremstyle{plain}

\newtheorem{cor}[theorem]{Corollary}

\theoremstyle{definition}

\begin{document}

\add{\bstctlcite{IEEEexample:BSTcontrol}}
%
\title{Mini-Batch Semi-Stochastic Gradient Descent \\ in the Proximal Setting}
%
%
%

\author{Jakub~Kone\v{c}n\'{y}, 
         Jie~Liu,
         Peter~Richt\'{a}rik, 
         Martin~Tak\'a\v{c}
\thanks{Jakub~Kone\v{c}n\'{y} and Peter Richt\'{a}rik are with the School of Mathematics, University of Edinburgh, United Kingdom, EH9 3FD.}
\thanks{Jie~Liu and Martin Tak\'a\v{c} are with the Department of Industrial and Systems Engineering,
Lehigh University, Bethlehem, PA 18015, USA.}
\thanks{Manuscript received April 15, 2015; revised ---.}}

\maketitle

\begin{abstract}
We propose mS2GD: a method incorporating a mini-batching scheme for improving the theoretical complexity and practical performance of semi-stochastic gradient descent  (S2GD). We consider the problem of minimizing a strongly convex  function represented as the sum of an average of a large number of smooth convex functions, and a simple nonsmooth convex regularizer. Our method first performs a deterministic step (computation of the gradient of the objective function  at the starting point), followed by a large number of stochastic steps. The process is repeated a few times with the last iterate becoming the new starting point. The novelty of our method is in introduction of mini-batching into the computation of stochastic steps. In each step, instead of choosing a single function, we sample $b$ functions, compute their gradients, and compute the direction based on this. We analyze the complexity of the method and show that it benefits from two speedup effects. First,  we prove that as long as $b$ is below a certain threshold, we can reach any predefined accuracy with less overall work than without mini-batching. Second, our mini-batching scheme admits a simple parallel implementation, and hence is suitable for further acceleration by parallelization.
\end{abstract}


%
\IEEEpeerreviewmaketitle


\section{Introduction} \label{sec:Intro}
%
%
%
%

\add{

\IEEEPARstart{I}{n} this work we are concerned with the problem of minimizing the sum of two convex functions,
\begin{equation}\label{Px1}
\min_{\x\in\R^d} \{\primal(\x) := F(\x) + R(\x)\},
\end{equation}
where the first component, $F$, is smooth,  and the second component, $R$, is  possibly nonsmooth (and extended real-valued, which allows for the modeling of constraints).

In the last decade, an intensive amount of research was conducted into algorithms for  solving problems of the form  \eqref{Px1}, largely motivated by the realization that the underlying problem has a considerable modeling power. One of the most popular and practical methods  for \eqref{Px1}  is the accelerated proximal gradient method of Nesterov \cite{nesterov2007acc}, with its most successful variant being FISTA \cite{fista}.

In many applications in optimization, signal processing and machine learning,  $F$ has an additional structure. In particular, it is often the case that $F$ is the {\em average of a   number of  convex functions:}
\begin{equation}\label{Px2}
F(\x) = \frac{1}{n} \sum_{i=1}^n f_i(\x).
\end{equation}

Indeed, even one of the most basic optimization problems---least squares regression---lends itself to a natural representation of the form \eqref{Px2}. 

\subsection{Stochastic methods.}

For problems of the form \eqref{Px1}+\eqref{Px2}, and especially when $n$ is large and when a solution of low to medium accuracy is sufficient, deterministic methods do not perform as well as classical {\em stochastic}\footnote{Depending on conventions used in different communities, the terms {\em randomized} or {\em sketching} are used instead of the word {\em stochastic}. In signal processing, numerical linear algebra and theoretical computer science, for instance, the terms {\em sketching} and {\em randomized} are used more often. In machine learning and optimization, the terms {\em stochastic} and {\em randomized} are used more often. In this paper, stochasticity does not refer to a data generation process, but to randomization embedded in an algorithm which is applied to a deterministic problem. Having said that, the deterministic problem {\em can} and {\em often does} arise as a sample average approximation of stochastic problem (average replaces an expectation), which further blurs the lines between the terms.}  methods. The prototype method in this category is stochastic gradient descent (SGD), dating back to the 1951 seminal work of Robbins and Monro~\cite{RM1951}. SGD  selects an index $i\in \{1,2,\dots,n\}$ uniformly at random, and then updates the variable $\x$ using $\nabla f_i(\x)$ --- a stochastic estimate of $\nabla F(\x)$. Note that the computation of $\nabla f_i$ is $n$ times cheaper than the computation of the full gradient $\nabla F$. For problems where $n$ is very large, the per-iteration savings  can be extremely large, spanning several orders of magnitude. 

These savings do not come for free, however (modern methods, such as the one we propose, overcome this -- more on that below). Indeed, the stochastic estimate of the gradient embodied by $\nabla f_i$ has a non-vanishing variance. To see this, notice that even when started from an optimal solution $x_*$, there is no reason for $\nabla f_i(x_*)$ to be zero, which means that SGD drives  away from the optimal point. Traditionally, there have been two ways of dealing with this issue. The first one consists in choosing a decreasing sequence of stepsizes. However, this means that a much larger number of iterations is needed. A second approach is to use  a subset (``minibatch'') of indices $i$, as opposed to a single index, in order to form a better stochastic estimate of the gradient. However, this results in a method which performs more work per iteration. In summary, while traditional approaches manage to decrease the variance in the stochastic estimate, this comes at a cost.

\subsection{Modern stochastic methods}

Very recently, starting with the SAG \cite{SAG}, SDCA \cite{SDCA}, SVRG \cite{SVRG} and S2GD \cite{S2GD}  algorithms from year 2013, it has transpired that neither decreasing stepsizes nor mini-batching are necessary to resolve the non-vanishing variance issue inherent in the vanilla SGD methods. Instead, these modern stochastic\footnote{These methods are randomized algorithms. However, the term ``stochastic''  (somewhat incorrectly) appears in their names for historical reasons, and quite possibly due to their aspiration to improve upon {\em stochastic} gradient descent (SGD).} method are able to dramatically improve upon SGD in various different ways, but without having to resort to the usual variance-reduction techniques (such as decreasing stepsizes or mini-batching) which carry with them considerable costs drastically reducing their power. Instead, these modern methods were able to improve upon SGD  without any unwelcome side effects. This development led to a revolution in the area  of first order methods for solving problem \eqref{Px1}+\eqref{Px2}. Both the theoretical complexity and practical efficiency of these modern methods vastly outperform prior gradient-type methods.

In order to achieve $\epsilon$-accuracy, that is,
\begin{equation}
\label{eq:epsilonaccuracy}
\Exp [\primal(x_k) - \primal(x_*)] \leq \epsilon [ \primal(x_0) - \primal(x_*) ],
\end{equation}
modern stochastic methods such as SAG, SDCA, SVRG and S2GD  require only \begin{equation}\label{eq:modern_complexity} {\cal O}((n+\kappa)\log(1/\epsilon))\end{equation} units of work, where $\kappa$ is a condition number associated with $F$, and one unit of work corresponds to the computation of the gradient of $f_i$ for a random index $i$, followed by a call to a prox-mapping involving $R$.
More specifically, $\kappa=L/\mu$, where $L$ is a uniform bound on the Lipschitz constants of the gradients of functions $f_i$ and $\mu$ is the strong convexity constant of $P$. These quantities will be defined precisely in Section~\ref{sec:analysis}.

The complexity bound \eqref{eq:modern_complexity} should be contrasted with that of proximal gradient descent (e.g., ISTA), which requires $O(n\kappa \log(1/\epsilon))$ units of work, or FISTA, which requires $O(n\sqrt{\kappa} \log(1/\epsilon))$ units of work\footnote{However, it should be remarked that the condition number $\kappa$ in these latter methods is slightly different from that appearing in the bound \eqref{eq:modern_complexity}.}.  Note that while all these methods enjoy linear convergence rate, the modern stochastic methods can be many orders of magnitude faster than classical deterministic methods. Indeed, one can have \[n+\kappa \ll n\sqrt{\kappa} \leq n\kappa.\] Based on this, we see that these modern methods always beat (proximal) gradient descent ($n+\kappa \ll n\kappa$), and also  outperform FISTA as long as $\kappa \leq {\cal O}(n^2)$.  In machine learning, for instance, one usually has $\kappa \approx n$, in which case the improvement is by a factor of $\sqrt{n}$ when compared to FISTA,  and by a factor of $n$ over ISTA. For applications where $n$ is massive, these improvements are indeed dramatic.

For more information about modern dual and primal   methods we refer the reader to the literature on  randomized coordinate descent methods \cite{nesterovRCDM, richtarik, PCDM, necoara2014random, approx,SDCA, marecek2014distributed,necoara2013distributed, richtarik2013distributed, fercoq2014fast,ALPHA, combettes2015} and stochastic gradient methods \cite{SAG,  zhangsgd, ma2015adding, COCOA, takac2013ICML, nitanda, ALPHA, rosasco2014}, respectively.

\subsection{Linear systems and sketching.}

In the case when $R\equiv 0$, all stationary points (i.e., points satisfying $\nabla F(x)=0$) are optimal for   \eqref{Px1}+\eqref{Px2}. In the special case when the functions $f_i$ are convex quadratics of the form $f_i(x)= \tfrac{1}{2}(a_i^T x - b_i)$, the equation $\nabla F(x)=0$ reduces to the linear system $A^T Ax = A^T b$, where $A=[a_1,\dots,n]$. Recently, there has been considerable interest in designing and analyzing  randomized methods for solving linear systems; also known under the name of {\em sketching} methods. Much of this work was done independently from the developments in (non-quadratic) optimization, despite the above connection between optimization and linear systems. A randomized version of the classical Kaczmarz method was studied in a seminal paper by Strohmer and Vershynin \cite{Strohmer2009}. Subsequently, the method was extended and improved upon in several ways \cite{Needell2010, Zouzias2012, 
Ma2015, Oswald2015}. The randomized Kaczmarz method is equivalent to SGD with a specific stepsize choice \cite{NeedellWard2015,GowerRichtarik2015-linear}. The first randomized coordinate descent method, for linear systems, was analyzed  by Lewis and Leventhal~\cite{Leventhal2008}, and subsequently generalized in various ways by numerous authors (we refer the reader to \cite{ALPHA} and the references therein).
Gower and Richt\'{a}rik \cite{GowerRichtarik2015-linear} have recently studied randomized iterative methods for linear systems in a general {\em sketch and project} framework, which in special cases includes randomized Kaczmarz, randomized coordinate descent, Gaussian descent, randomized Newton, their block variants, variants with importance sampling, and also an infinite array  of new specific  methods. For approaches of a combinatorial flavour, specific to diagonally dominant systems, we refer to the influential work of Spielman and Teng  \cite{Spielman2006}.

\section{Contributions}

In this paper we equip moderns stochastic methods---methods which already enjoy the fast rate 
\eqref{eq:modern_complexity}---with the ability to process data in {\em mini-batches}. None of the {\em primal}\footnote{By a primal method we refer to an algorithm which operates directly to solve \eqref{Px1}+\eqref{Px2} without explicitly operating on the dual problem. {\em Dual} methods have very recently   been analyzed in the mini-batch setting. For a review of such methods we refer the reader to  the paper describing the QUARTZ method \cite{QUARTZ} and the references therein.} modern methods have been analyzed in the mini-batch setting. This paper fills this gap in the literature. 

While we have argued above that the modern methods, S2GD included, do not have the ``non-vanishing variance'' issue that SGD does, and hence do not {\em need} mini-batching for that purpose,  mini-batching is still useful.  In particular,  we develop and analyze the complexity of mS2GD (Algorithm~\ref{alg:mS2GD}) --- a mini-batch proximal variant of {\em semi-stochastic gradient descent} (S2GD) \cite{S2GD}. While the S2GD method was analyzed in the $R=0$ case only, we develop and analyze our method in the proximal\footnote{Note that the Prox-SVRG method   \cite{proxsvrg} can also handle the composite problem \eqref{Px1}.} setting  \eqref{Px1}. We show that mS2GD enjoys several benefits when compared to previous modern methods. First, it trivially admits a parallel implementation, and hence enjoys a speedup in clocktime in an HPC environment. This is critical for applications with massive datasets and is the main motivation and advantage of our method. Second, our results show that in order to attain a  specified accuracy $\epsilon$, mS2GD can get by with fewer gradient evaluations than S2GD. This is formalized in Theorem~\ref{thm:optimalM}, which predicts more than linear speedup up to a certain threshold mini-batch size after which the complexity deteriorates. Third, compared to \cite{proxsvrg}, our method does not need to average the iterates produced in each inner loop;  we instead simply continue from the last one. This is the approach employed in S2GD \cite{S2GD}.

}


\section{The Algorithm}
\label{sec:algorithms}

In this section we first briefly motivate the mathematical setup of deterministic and stochastic proximal gradient methods in Section \ref{sec:algorithms}A, followed by the introduction of semi-stochastic gradient descent in Section~\ref{sec:algorithms}B. We will the be ready to describe the mS2GD method in Section~\ref{sec:algorithms}C.

\subsection{Deterministic and stochastic proximal gradient methods}

The classical {\em deterministic proximal gradient} approach \cite{fista,combettes2011,parikh2014} to solving \eqref{Px1}  is to form a sequence $\{\y_{t}\}$ via 
\begin{align*}
\y_{t+1} = \arg \min_{\x\in \R^d} U_t(x),
\end{align*} 
where $U_t(x) \eqdef F(\y_t) + \nabla F (\y_{t})^T (\x-\y_{t}) + \tfrac{1}{2\stepsize} \|\x-\y_{t}\|^2 + R(x)$.
Note that in view of Assumption~\ref{ass1}, which we shall use in our analysis in Section~\ref{sec:analysis},  $U_{t}$ is an upper bound on $\primal$  whenever $\stepsize>0$ is a stepsize parameter satisfying $1/\stepsize \geq L$. This procedure can be compactly written using the {\em proximal operator} as follows:
\begin{equation*}
\y_{t+1} = \prox_{\stepsize R} (\y_{t} - \stepsize \nabla F (\y_{t})),
\end{equation*} 
where
\begin{equation*}
\prox_{hR}(\z) \eqdef \arg\min_{\x \in\R^d} \{ \tfrac 12  \|\x-\z\|^2 + h R(\x)\}.
\end{equation*}

In a large-scale setting it is more efficient to instead consider  the \emph{stochastic proximal  gradient} approach, in which the proximal operator is applied to a stochastic gradient step:
\begin{equation}\label{eqn:prox-svrg update}
y_{t+1} = \prox_{\stepsize R}(y_{t} - \stepsize G_{t}),
\end{equation}
where $G_{t}$ is a stochastic estimate of the gradient $\nabla F(y_t)$. 

\subsection{Semi-stochastic methods}

Of particular relevance to our work are the SVRG \cite{SVRG}, S2GD \cite{S2GD} and Prox-SVRG  \cite{proxsvrg} methods where the stochastic estimate of $\nabla F(y_t)$ is of the form
\begin{equation}\label{eq:sjs8js}
G_{t} = \nabla F (x) + \tfrac{1}{nq_{i_t}}(\nabla f _{i_t}(\y_{t}) - \nabla f_{i_t}(x)),
\end{equation}
where  $x$ is an ``old'' reference point for which the gradient $\nabla F(x)$ was already computed in the past, and $i_t \in \setn \eqdef \{1,2,\dots,n\}$ is a random index equal to $i$ with probability $q_i>0$.  Notice that $G_{t}$ is an unbiased estimate of the gradient of $F$ at $y_t$: \[\Exp_{i_t} [G_{t}] \overset{\eqref{eq:sjs8js}}{=} \nabla F(x) + \sum_{i=1}^n q_i \tfrac{1}{n q_i} (\nabla f_i(y_t) - \nabla f_i(x)) \overset{\eqref{Px2}}{=} \nabla F(y_t).\]

Methods such as  S2GD, SVRG, and  Prox-SVRG update the points $y_t$ in an inner loop, and the reference point $x$ in an outer loop (``epoch'') indexed by $k$. With this new outer iteration counter we will have $x_k$ instead of $x$, $y_{k,t}$ instead of $y_t$ and $G_{k,t}$ instead of $G_t$. This is the notation we will use in the description of our algorithm in Section~\ref{sec:algorithms}C. The outer loop ensures that the squared norm of $G_{k,t}$ approaches zero as $k,t\to \infty$ (it is easy to see that this is equivalent to saying that the stochastic estimate $G_{k,t}$ has a diminishing variance), which ultimately leads to extremely fast convergence.

\subsection{Mini-batch S2GD}

We are now ready to describe the mS2GD method\footnote{\add{
A more detailed algorithm and the associated analysis (in which we benefit from the knowledge of lower-bound on the strong convexity parameters of the functions $F$ and $R$) can be found in the arXiv preprint~\cite{konecny2015mini}.
The more general algorithm mainly differs in $t_k$ being chosen according to a geometric  probability law which depends on the estimates of the convexity constants.
}}
 (Algorithm~\ref{alg:mS2GD}).

\begin{algorithm}[H]
\caption{mS2GD}
\label{alg:mS2GD}
\begin{algorithmic}[1]
\State \textbf{Input:} $m$ (max \# of stochastic steps per epoch); $\stepsize>0$ (stepsize); $\x_0 \in\R^d$ (starting point);  mini-batch size $b \in \setn$ 
\For {$k=0,1, 2, \dots$}
    \State Compute and store $g_{k} \leftarrow \nabla F(\x_{k}) = \tfrac{1}{n}\sum_{i} \nabla f_i(\x_{k})$ 
    \State Initialize the inner loop: $\y_{k,0} \gets \x_{k}$
    \State Choose $t_k  \in \{1,2,\dots,m\}$  \add{uniformly at random}
    \For {$t=0$ to $t_k-1$}
	    \State Choose mini-batch $A_{kt}\subseteq \setn$ of size $b$, \newline {\color{white}.} \qquad \qquad uniformly at random 
        \State 
\label{step:8}        
        Compute a stochastic estimate of $\nabla F(\y_{k,t})$: \newline
        {\color{white}ij} \quad \quad $G_{k,t} \gets g_k + \frac{1}{b}\sum_{i\in A_{kt}}(\nabla f_{i}(\y_{k,t}) - \nabla f_{i}(\x_{k})) $
        \State
\label{step:9}        
         $\y_{k,t+1} \gets \prox_{\stepsize R}(\y_{k,t} - \stepsize G_{k,t})$
     \EndFor
    \State  Set $\x_{k+1} \leftarrow \y_{k,t_k}$
\EndFor
\noindent
\end{algorithmic}
\end{algorithm}


The algorithm includes an outer loop, indexed by epoch counter $k$, and an inner loop, indexed by $t$. Each epoch is started by computing $g_k$, which is the  (full) gradient of $F$ at $x_k$. It then immediately  proceeds to the inner loop. The inner loop is run for $t_k$ iterations,  
\add{where $t_k$ is chosen uniformly at random from  $ \{1,\dots,m\}$.} Subsequently, we run $t_k$ iterations in the inner loop (corresponding to Steps 6--10). Each new iterate is given by the proximal update~\eqref{eqn:prox-svrg update}, however with the stochastic estimate of the gradient $G_{k,t}$ in~\eqref{eq:sjs8js}, which is formed by using a {\em mini-batch} of examples $A_{kt}\subseteq [n]$ of size $|A_{kt}|=b$. Each inner iteration requires $2b$ units of work\footnote{It is possible to finish each iteration with only $b$ evaluations for  component gradients, namely $\{\nabla f_i(y_{k,t})\}_{i\in A_{kt}}$, at the cost of having to store $\{\nabla f_i(x_k)\}_{i\in\setn}$, which is exactly the way that SAG~\cite{SAG} works. This speeds up the algorithm; nevertheless, it is impractical for big $n$.}.

\section{Analysis}
\label{sec:analysis}
 
In this section, we lay down the assumptions, state our main complexity result, and  comment on how to optimally choose the parameters of the method.

\subsection{Assumptions}

Our analysis is performed under the following two assumptions.

\begin{assumption}\label{ass1}
Function  $R:\R^d\to \R\cup \{+\infty\}$ (regularizer/proximal term) is  convex and closed. The functions  $f_i:\R^d\to \R$  have Lipschitz continuous gradients with constant $L > 0$. That is, 
$
\| \nabla f _i(\x)-\nabla f _i(\y) \|\leq L \|\x-\y\|,
$
for all $x,y\in \R^d$, where $\|\cdot\|$ is the $\ell_2$-norm.
\end{assumption}

Hence,  the gradient of $F$ is also Lipschitz continuous with the same constant $L$. 

\begin{assumption}\label{ass2}
 $\primal$ is strongly convex with parameter $\mu>0$. That is for all $\x,\y \in \dom(R)$ and $ \xi \in \partial{\primal(\x)}$, 
\begin{equation}\label{strongconv}
\primal(\y) \geq \primal(\x) + \xi^T(\y-\x) + \tfrac\mu2 \|\y-\x\|^2,
\end{equation}
 where $\partial \primal(\x)$ is the subdifferential of $\primal$ at $\x$.
\end{assumption}

Lastly, by $\mu_F \geq 0$ and $\mu_R\geq 0$ we denote the strong convexity constants of $F$ and $R$, respectively. We allow both of these quantities to be equal to $0$, which simply means that the functions are convex (which we already assumed above). Hence, this is not necessarily an additional assumption.

\subsection{Main result}

We are now ready to formulate our complexity result.
 
\begin{theorem}\label{s2convergence}
Let Assumptions~\ref{ass1} and~\ref{ass2} be satisfied, let $\x_* \eqdef \arg\min_\x \primal(\x)$ and choose $b\in \{1,2\dots,n\}$. Assume  that
\add{ $0<\stepsize \leq 1/L$, $4h L \alpha(b)<1$ and that $m,h$ are further chosen so that
\begin{equation}\label{s2rho}
\decrease \eqdef 
\tfrac{    1
  }
  {
  m 
\stepsize \mu
(
1
-
  4\stepsize  L \alpha(b)  
)
  }
+
\tfrac{      
  4\stepsize  L \alpha(b) 
\left(      
 m
+ 
  1
\right)  
  }
  {
  m 
(
1
-
  4\stepsize   L \alpha(b)  
)
  } < 1,
\end{equation}}
where $\alpha(b) \eqdef \frac{n-b}{b (n-1)} $. Then mS2GD  has linear convergence in expectation with rate $\decrease$:
\begin{equation*}
\Exp [\primal(\x_k) - \primal(\x_*)] \leq \decrease^k [\primal(\x_0)-\primal(\x_*)].
\end{equation*}
\end{theorem}

\add{
Notice that for any fixed $b$, by properly adjusting the parameters $h$ and $m$ we can force $\rho$ to be arbitrarily small. Indeed, the second term can be made arbitrarily small by choosing $h$ small enough. Fixing the resulting $h$, the first term can then be made arbitrarily small by choosing $m$ large enough. This may look surprising, since this means that only a single outer loop ($k=1$) is needed in order  to obtain a solution of any prescribed accuracy. While this is indeed the case, such a choice of the parameters of the method ($m$, $h$, $k$) would not be optimal -- the resulting workload would to be too high as the complexity of the method would  depend sublinearly on $\epsilon$. In order to obtain a logarithmic dependence on $1/\epsilon$, i.e., in order to obtain linear convergence, one needs to perform $k=O(\log(1/\epsilon))$ outer loops, and set the parameters $h$ and $m$ to appropriate values (generally, $h={\cal O}(1/L)$ and $m={\cal O}(\kappa)$).

} 

\subsection{Special cases: $b=1$ and $b=n$}

In the special case with  $b=1$ (no mini-batching), we get $\alpha(b)=1$, and the rate given by \eqref{s2rho} exactly recovers the rate achieved by Prox-SVRG \cite{proxsvrg} (in the case when the Lipschitz constants of $\nabla f_i$ are all equal). The rate is also identical to the rate of S2GD \cite{S2GD} (in the case of $R=0$, since S2GD was only analyzed in that case). \add{If we set the number of outer iterations to  $k =\lceil \log(1/\epsilon)\rceil$, choose the stepsize as $h =  \tfrac{1}{(2+4e)L}$, where $e=\exp(1)$, and choose $m=43 \kappa$, then the total workload of mS2GD  for achieving \eqref{eq:epsilonaccuracy}  is $(n+43 \kappa)\log(1/\epsilon)$ units of work. Note that this recovers the fast rate \eqref{eq:modern_complexity}.  

In the batch setting, that is when $b=n$, we have $\alpha(b)=0$ and hence $\rho = 1/(mh \mu)$. By choosing $k=\lceil \log (1/\epsilon)\rceil$, $h=1/L$, and $m=2\kappa$, we obtain the rate ${\cal O}\left(n\kappa \log(1/\epsilon)\right)$. This is the standard rate of (proximal) gradient descent. 

Hence, by modifying the mini-batch size $b$ in mS2GD, we interpolate between the fast rate  of S2GD and the slow rate of GD.}


\subsection{Mini-batch speedup} \label{sec:speedup}

\add{

In this section we will derive formulas for good choices  of the parameter $m,h$ and $k$ of our method as a function of $b$. Hence, throughout this section we shall consider $b$ fixed. 


Fixing $0<\rho<1$, it is easy to see that in order for $x_k$ to be an $\epsilon$-accurate solution (i.e., in order for \eqref{eq:epsilonaccuracy} to hold), it suffices to choose $k\geq (1-\rho)^{-1}\log(\epsilon^{-1})$. Notice that the total workload  mS2GD will do in order to arrive at $x_k$ is \[ k(n+2m) \approx (1-\rho)^{-1}\log(\epsilon^{-1}) (n+2m)\] units of work. If we now consider $\rho$ fixed (we may try to optimize for it later), then clearly the total workload is proportional to $m$. The free parameters of the method are the stepsize $h$ and the inner loop size $m$. Hence, in order to set the parameters so as to minimize the workload (i.e., optimize the complexity bound), we would like to (approximately) solve the optimization problem
\[\min m \quad \text{subject to} \quad 0< h \leq \tfrac{1}{L}, \; h < \tfrac{1}{4L\alpha(b)}, \; \rho \; \text{is fixed}.\]

Let $(\stepsize_*^b, m_*^b)$ denote the optimal pair (we highlight the dependence on $b$ as it will be useful). Note that if $m_*^b \leq m_*^1 / b$ for some $b>1$, then mini-batching can help us reach the $\epsilon$-solution with smaller overall workload. The following theorem presents the formulas for $\stepsize_*^b$ and $m_*^b$.

\begin{theorem}
\label{thm:optimalM}
Fix $b$ and $0<\rho<1$ and let
\begin{equation*}
\tilde \stepsize^b\ \eqdef\ \sqrt{ \left( \tfrac{1+\decrease}{\decrease\mu} \right)^2 + \tfrac1{4\mu\alpha(b)L}} - \tfrac{1+\decrease}{\decrease\mu}.
\end{equation*}
If $\tilde \stepsize^b \leq \frac1L$, then $\stepsize_*^b = \tilde \stepsize^b$ and
\begin{align*}
 m_*^b 
=
\tfrac{2\kappa}{\rho} \left\{ \left( 1+\tfrac1\rho \right)4\alpha(b) + \sqrt{
  \tfrac{4\alpha(b)}{ \kappa} +   \left( 1+\tfrac1\rho \right)^2[4\alpha(b)]^2}
 \right\},
 \tagthis\label{eq:vfewdfwaefvawvgfeefewafa}
\end{align*}
where $\kappa \eqdef \frac{L}{\mu}$ is the condition number. If $\tilde{h}^b > \tfrac{1}{L}$, then  $\stepsize_*^b = \frac1L$
and
\begin{equation}\label{eq:fasfawefwafewa}
m_*^b
  = \tfrac{\kappa + 4 \alpha(b) }
  { \rho -4 \alpha(b)   (1+\rho)}.
\end{equation}

\end{theorem}

Note that if $b=1$, then 

Equation~\eqref{eq:vfewdfwaefvawvgfeefewafa} suggests that as long as the condition $\tilde \stepsize^b \leq \frac1L$ holds, $m_*^b$ is decreasing at a rate  faster than $1/b$. Hence, we can  find the solution with less overall work when using a minibatch of size $b$ than when using a minibatch of size $1$.

}

\subsection{Convergence rate}

\add{
In this section we study the total workload of mS2GD in the regime of small mini-batch sizes.

\begin{cor}\label{thm:minibatch}
Fix $\epsilon\in (0,1)$, choose the number of outer iterations equal to \[ k = \left\lceil \log(1 / \epsilon) \right\rceil,\]  and fix the target decrease in Theorem~\ref{thm:optimalM} to satisfy $\rho = \epsilon^{1 / k}$. Further, pick a mini-batch size satisfying $1 \leq b \leq 29$,  let the stepsize $h$ be as  in~\eqref{eqn:stepsizehb} and let $m$ be as in \eqref{eqn:maxiter}. Then in order for mS2GD to find $x_k$ satisfying \eqref{eq:epsilonaccuracy},  mS2GD needs at most
\begin{equation}\label{eqn:complexity}
(n + 2b m^b) \lceil \log(1 / \epsilon) \rceil
\end{equation}
units of work, where $bm^b = {\cal O}(\kappa)$, which leads to the overall complexity of
$$ \mathcal{O} \left( (n + \kappa) \log(1 / \epsilon) \right)$$
units of work.
\end{cor}
\begin{proof}
Available in Appendix \ref{thm:proof3}.
\end{proof}

This result shows that as long as the mini-batch size is small enough, the total work performed by mS2GD is the same as in the  $b=1$ case. If the $b$ updates can be performed in parallel, then this leads to linear speedup. 
}

 \begin{figure*}[htbp]
 \centering
  \includegraphics[width=5.5cm,height=4.5cm]{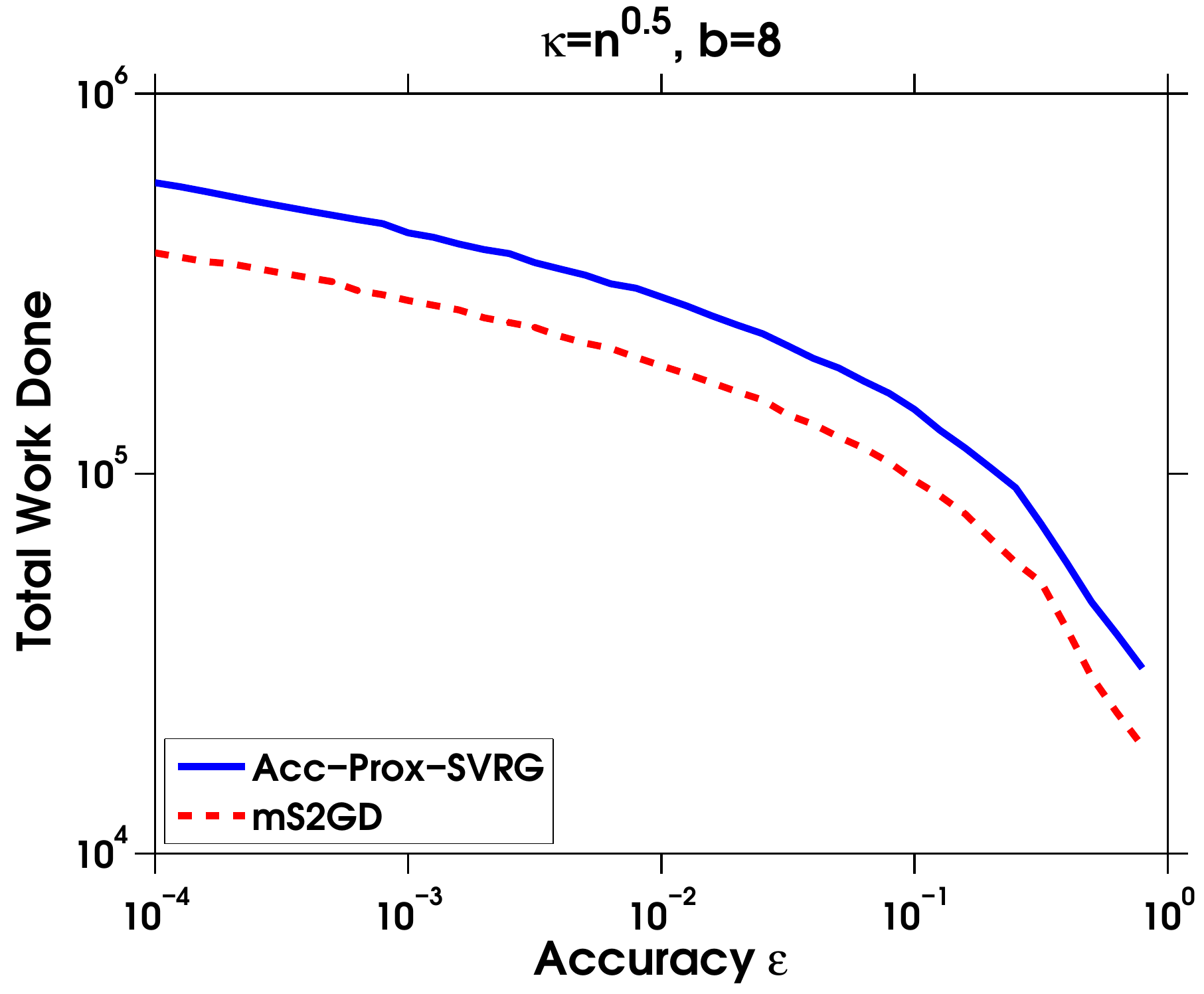}
  \includegraphics[width=5.5cm,height=4.5cm]{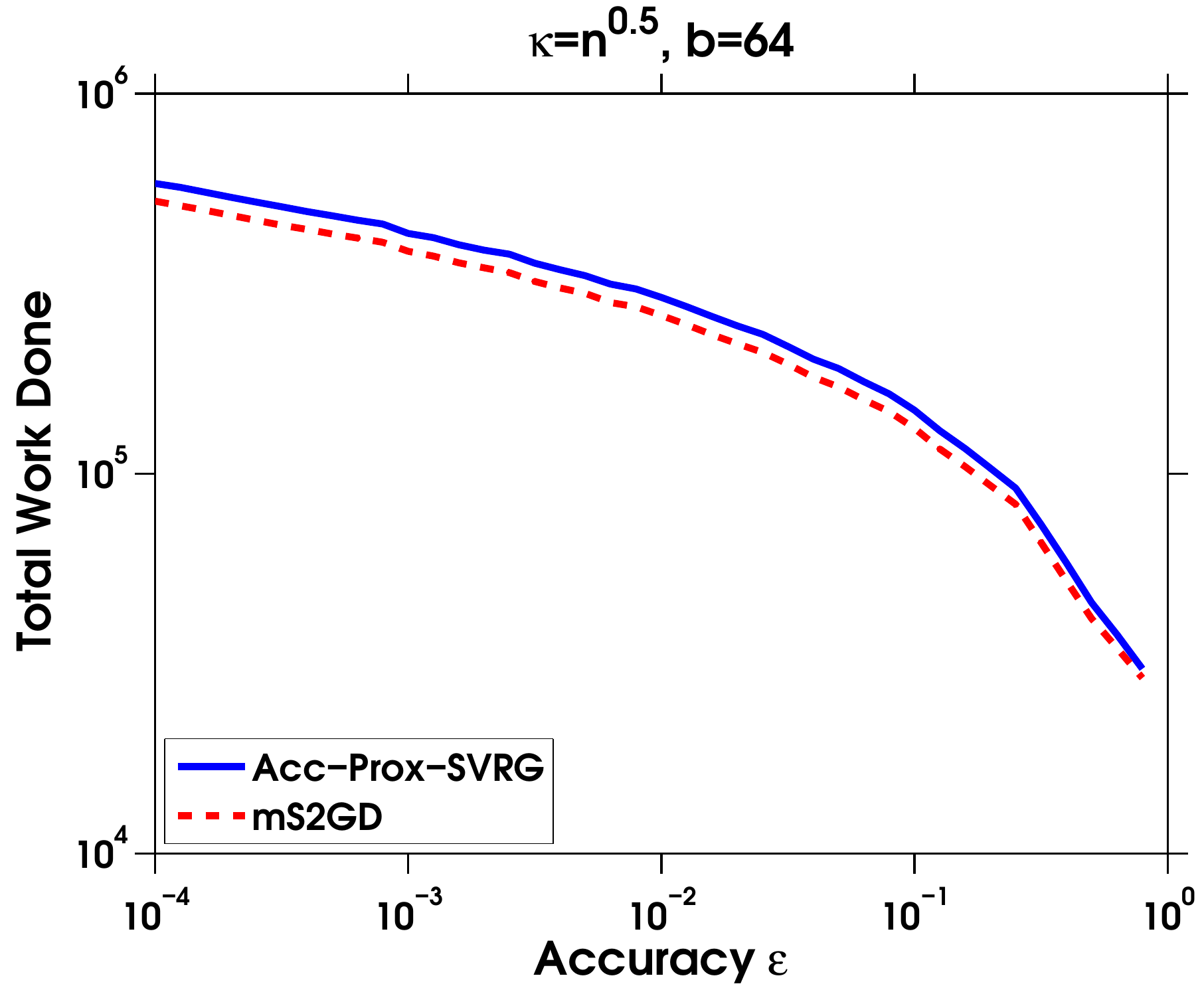}
  \includegraphics[width=5.5cm,height=4.5cm]{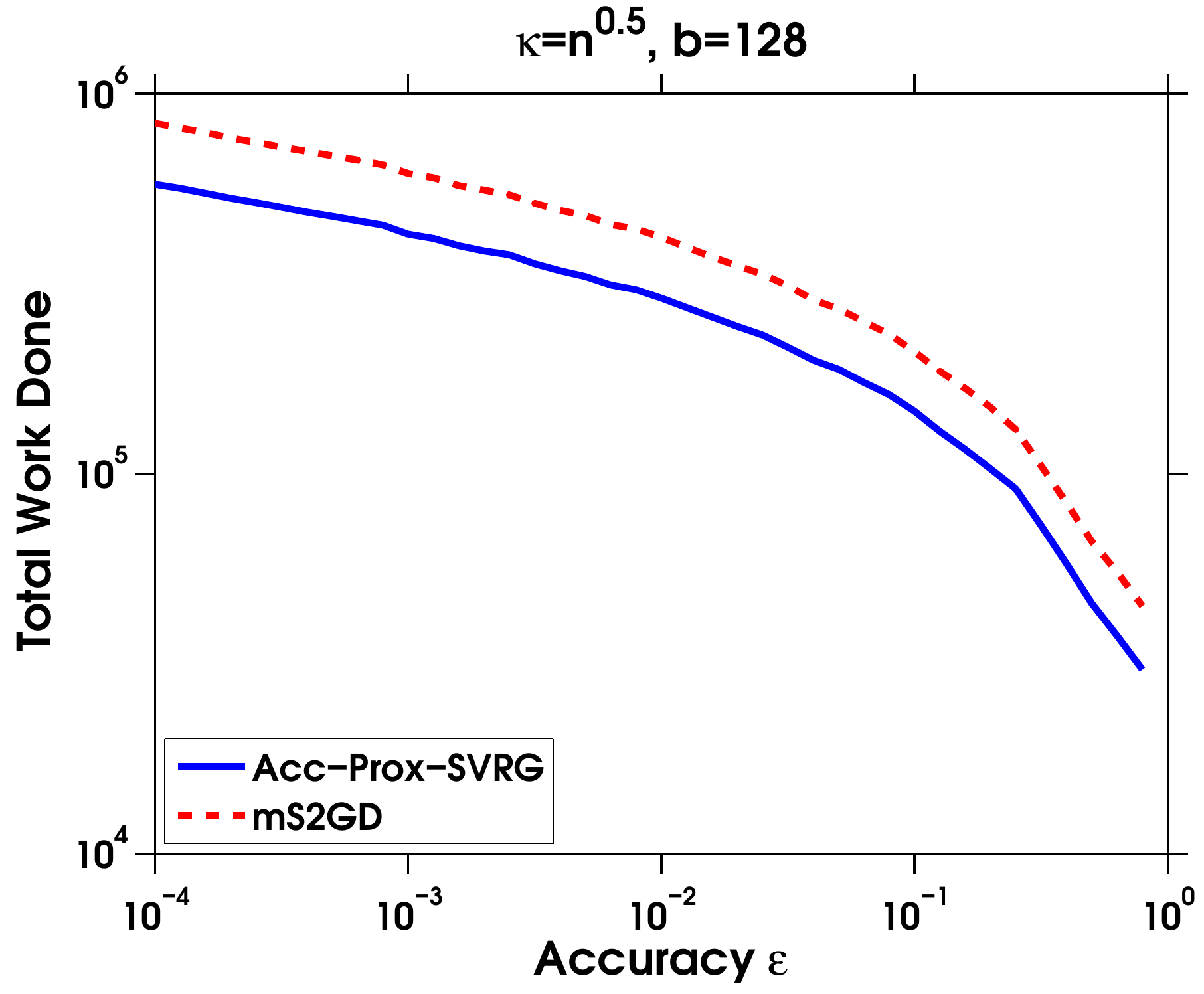}
  \includegraphics[width=5.5cm,height=4.5cm]{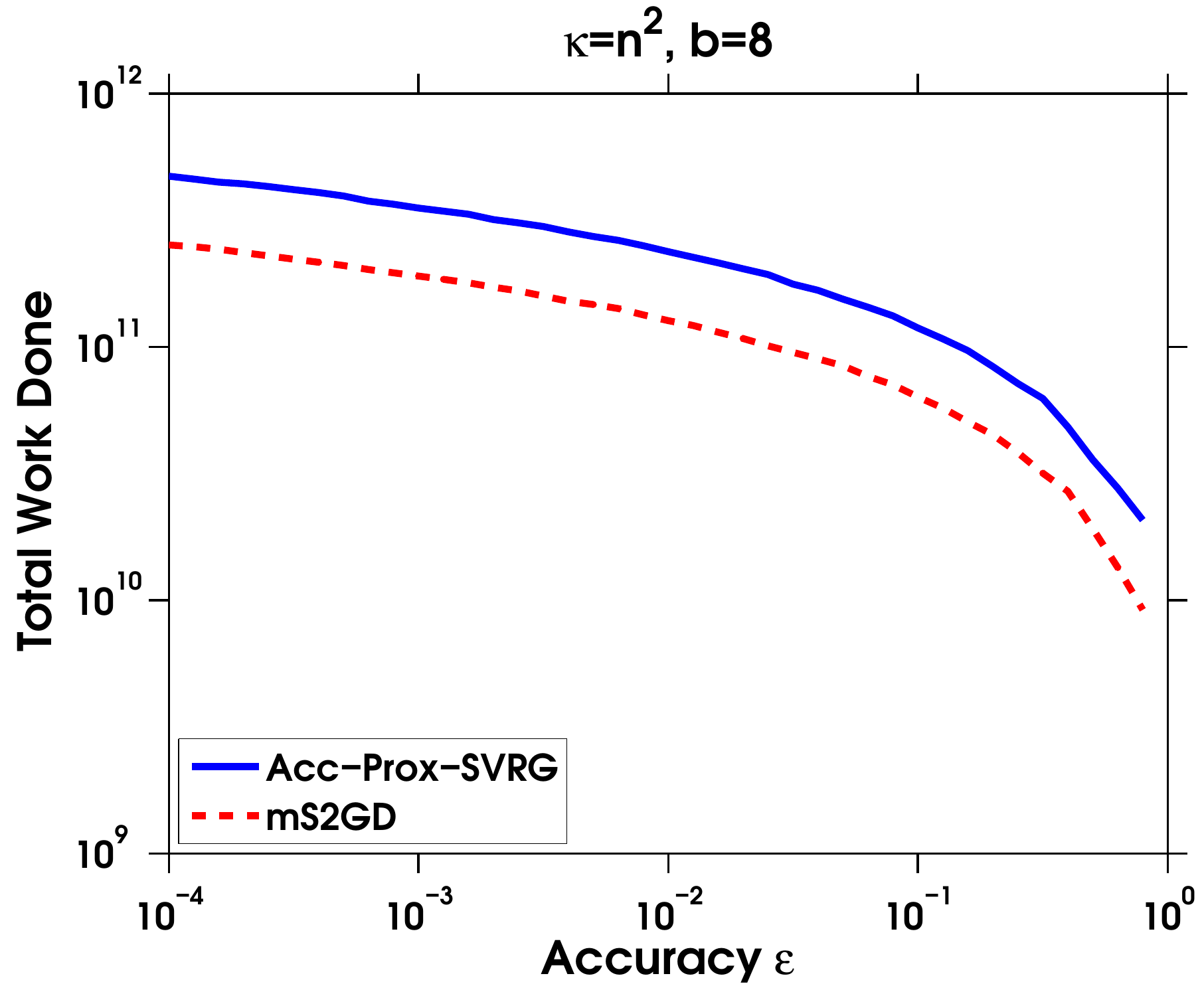}
  \includegraphics[width=5.5cm,height=4.5cm]{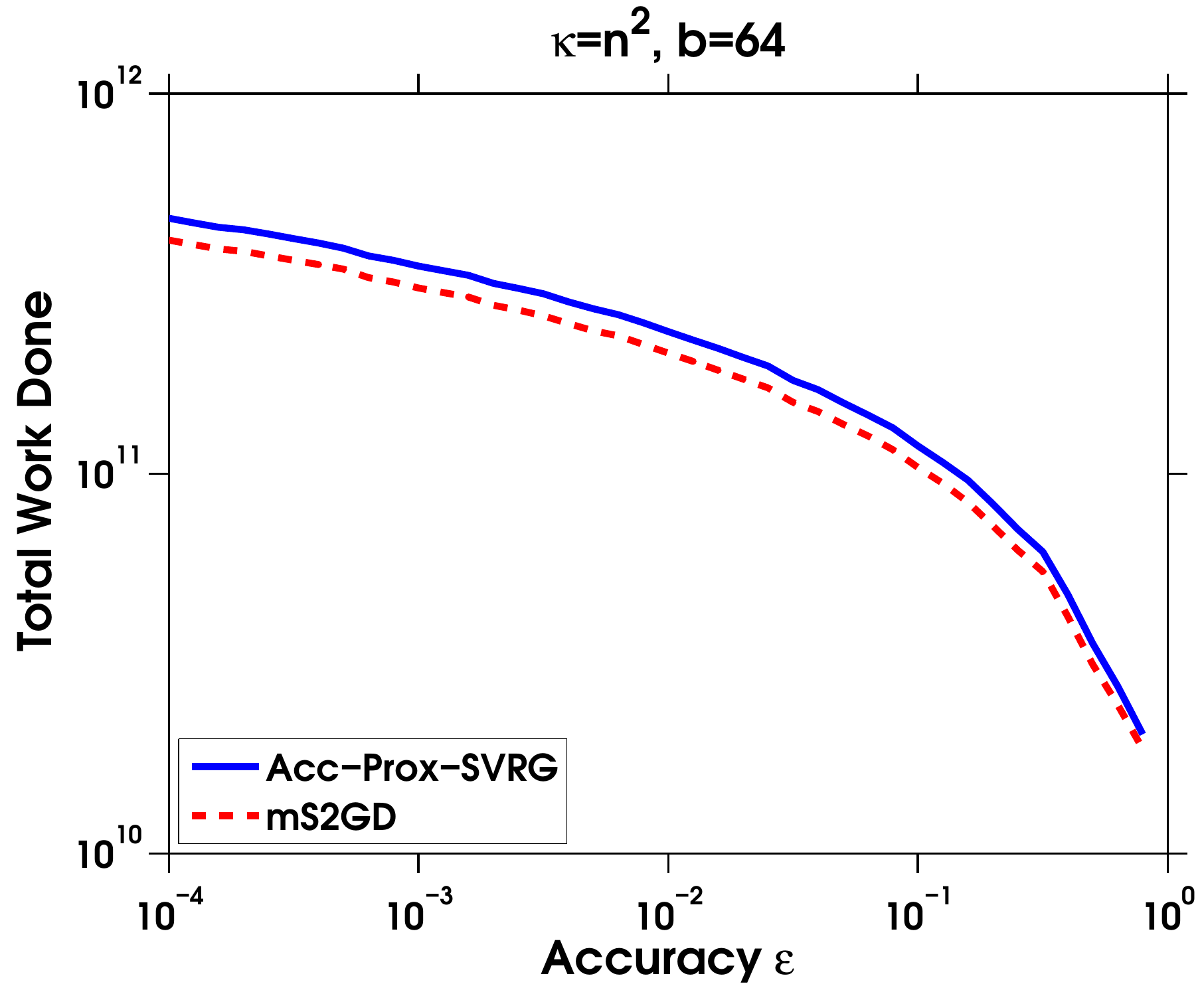}
  \includegraphics[width=5.5cm,height=4.5cm]{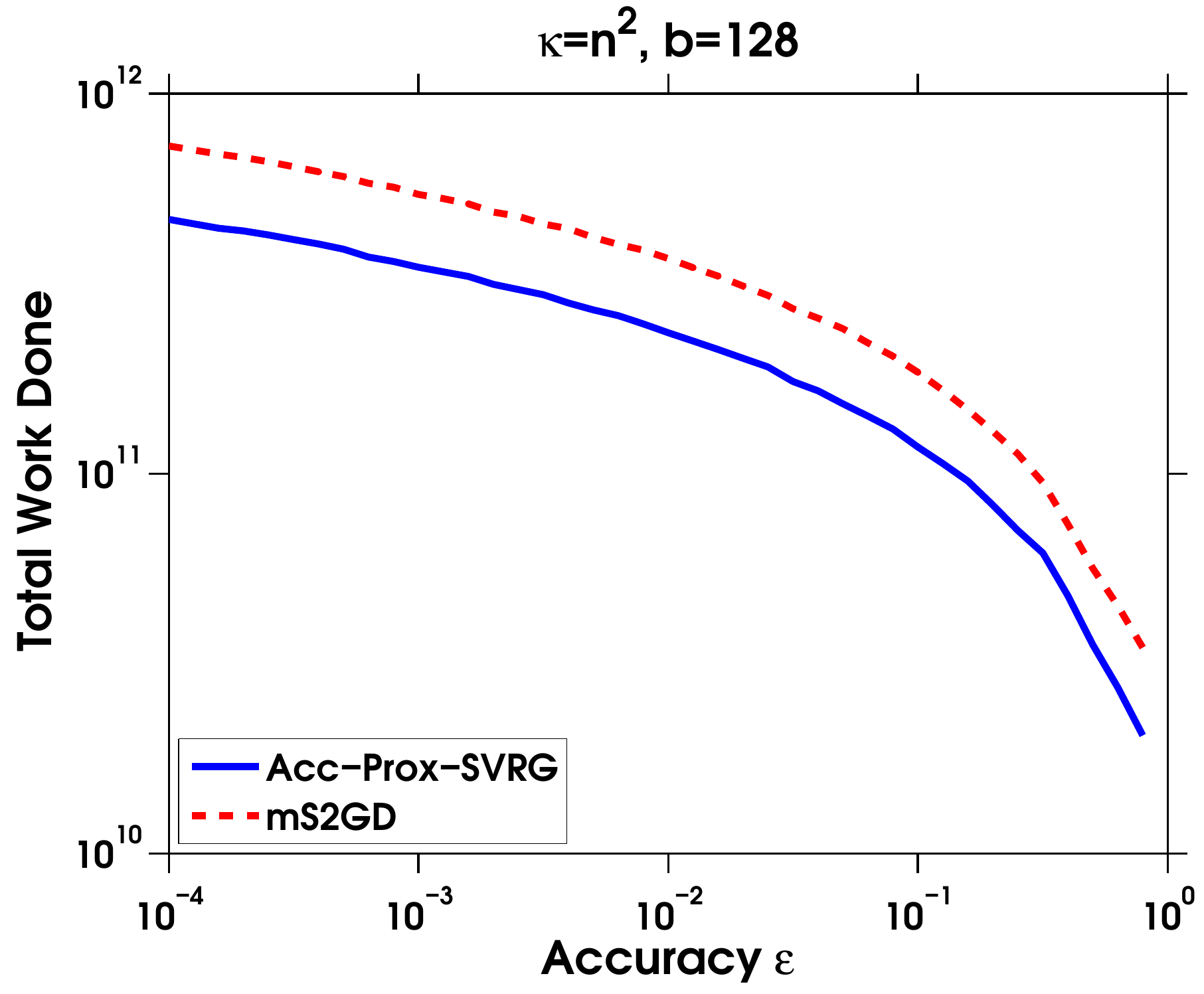}
 \caption{\footnotesize Complexity of Acc-Prox-SVRG and mS2GD in terms of total work done \add{for  $n=10,000$, and small ($\kappa=\sqrt{n}$; top row) and large ($\kappa = n^2$; bottom row) condition number.}}
 \label{figure:acc-prox-svrg}
 \end{figure*}

\subsection{Comparison with Acc-Prox-SVRG}
\add{
The Acc-Prox-SVRG~\cite{nitanda} method of Nitanda, which  was not available online before the first version of this paper appeared  on arXiv,  incorporates both a mini-batch scheme and Nesterov's acceleration~\cite{nesterov2004convex, nesterov2007acc}. 
The author claims that when $b< \lceil b_0 \rceil$, with the threshold $b_0$ defined as $\frac{8\sqrt{\kappa}n}{\sqrt{2} p(n-1) +8\sqrt{\kappa}}$, the overall complexity of the method is 
\[\mathcal{O} \left( 
\left( 
n + \tfrac{n-b}{n-1}\kappa 
\right) 
\log(1/\epsilon) 
\right);\]
and otherwise it is
\[\mathcal{O} \left( 
\left( 
n + b\sqrt{\kappa} 
\right) 
\log(1/\epsilon)  
\right).\]
 This suggests that acceleration will only be realized when the mini-batch size is large, while for small $b$, Acc-Prox-SVRG achieves the same overall complexity, $\mathcal{O} \left((n+\kappa)\log(1/\epsilon) \right)$, as mS2GD.
 
We will now take a closer look at  the theoretical results given by Acc-Prox-SVRG and mS2GD, for each $\epsilon\in(0,1)$. In particular, we shall numerically minimize the total work of mS2GD, i.e.,
\[\left(
 n+2b\lceil m^b\rceil
 \right)
 \left\lceil
\log(1/\epsilon) / \log(1/\rho)
  \right\rceil,\]
  over $\rho\in(0,1)$ and $h$ (compare this with \eqref{eqn:complexity}); and compare these results with similar fine-tuned quantities for Acc-Prox-SVRG.\footnote{$m^b$ is the best choice of $m$ for Acc-Prox-SVRG and mS2GD, respectively. Meanwhile, $h$ is within the safe upper bounds for both methods.}

Fig.~\ref{figure:acc-prox-svrg} illustrates these theoretical complexity bounds  for both ill-conditioned and well-conditioned data. With small-enough mini-batch size $b$, mS2GD is better than Acc-Prox-SVRG. However, for a large mini-batch size $b$, the situation reverses because of the acceleration inherent in Acc-Prox-SVRG.\footnote{We have experimented with different values for $n, b$ and $\kappa$, and this result always holds.} Plots with $b=64$ illustrate the cases where we cannot observe any differences between the methods.

Note however that accelerated methods are very prone to error accumulation. Moreover, it is not clear that an efficient implementation of Acc-Prox-SVRG is possible for sparse data.  As shall show  in the next section, mS2GD allows for such an implementation.
}


\section{Efficient implementation for sparse data}\label{sec:implementation_sparse}


\add{
Let us make the following assumption about the structure of functions $f_i$ in \eqref{Px2}.
\begin{assumption}
The functions $f_i$ arise as the composition of a univariate smooth function $\phi_i$ and an inner product with a datapoint/example $a_i \in \R^d$:
$f_i(x)  = \phi(a_i^Tx)$  for $i=1,\dots,n$.
 \label{asm:structure}
\end{assumption}
Many functions of common practical interest satisfy this assumption including linear and logistic regression. 
Very often, especially for large scale datasets, the data are extremely sparse, i.e.
the vectors $\{a_i\}$ contains many zeros.
Let us denote the number of non-zero coordinates of $a_i$ by $\omega_i = \|a_i\|_0 \leq d $
and the set of indexes corresponding to  non-zero coordinates by
$\support(a_i)=\{j : a_i^{j} \neq 0\} $,
where $a_i^{j}$ denotes the $j^{th}$ coordinate of vector $a_i$.
\begin{assumption} \label{asm:structure2}
The regularization function $R$ is separable.
\end{assumption}
This includes the most commonly used  regularization functions as 
$\frac{\lambda}{2} \|x\|^2$
or $\lambda \|x\|_1$.
}

Let us take a brief detour and look at the classical SGD algorithm
with $R = 0$. The update 
would be of the form
\begin{equation}
  x_{j+1} \gets x_{j}
   -\stepsize \phi'_i(a_i^T x_j) a_i = x_{j} - \stepsize \nabla f_i(x_j). 
   \label{adsfawlfcawefcava}
\end{equation}
If evaluation of the univariate function $\phi'_i$  takes $O(1)$ amount of work, the computation of
$\nabla f_i$ will account for $O(\omega_i)$ work.
Then the update 
\eqref{adsfawlfcawefcava}
would cost $O(\omega_i)$ too, which implies that the classical SGD method can naturally benefit from sparsity of data.

Now, let us get back to the Algorithm~\ref{alg:mS2GD}.
Even under the sparsity assumption and structural Assumption \ref{asm:structure}
the Algorithm~\ref{alg:mS2GD} suggests that each inner iteration will cost $O(\omega + d ) \sim O(d)$
because $g_k$ is in general fully dense and hence in 
Step \ref{step:9} of Algorithm~\ref{alg:mS2GD}
we have to update all $d$ coordinates.

However, in this Section, we will 
introduce and describe the implementation trick 
which is based on ``lazy/delayed'' updates.
The main idea of this trick is not to perform Step  \ref{step:9}
of Algorithm~\ref{alg:mS2GD}
for all coordinates, but only for 
coordinates $j \in \cup_{i \in A_{kt}} \support(a_i)$.
The algorithm is described in  Algorithm 
\ref{alg:mS2GDsparse}. 
\begin{algorithm}[H]
\caption{"Lazy" updates for mS2GD \\
(these replace steps 6--10 in Algorithm ~\ref{alg:mS2GD})}
\label{alg:mS2GDsparse}
\begin{algorithmic}[1]
    \State $\chi^{(j)} \gets 0$ for $j=1,2,\dots,d$
    \For {$t=0$ to $t_k-1$}
	\State Choose mini-batch $A_{kt}\subseteq \setn$ of size $b$, \newline {\color{white}.} \qquad \qquad uniformly at random
        \For {$i\in A_{kt}$} 
            \For {$j\in \support(a_i)$}
            \label{asdfasfsafdsafa}
                \State $y_{k,t}^{j} \gets\prox^{t-\vc{\chi}{j}}  [ y_{k,{\chi^{j} }}^{j}, g_k^{j},  R, h]$
                \State $\vc{\chi}{j} \gets t$
            \EndFor
            \label{asdfasfsafdsafa2}
        \EndFor 
        \State $\y_{k,t+1} \gets \y_{k,t} -\frac{\stepsize}{b}\sum_{i\in A_{kt}}a_i( \phi'_{i}(\y_{k,t}^Ta_i) -  \phi'_{i}(\x_{k}^Ta_i))$ 
     \EndFor
        \For {$j = 1$ to $d$} 
        \label{asdfsafdasfa}
            \State $\vc{y_{k,t_k}}{j} \gets \prox^{t_k-\vc{\chi}{j}}
            [\vc{y_{k,\vc{\chi}{j}}}{j}, \vc{g_k}{j}, R, h]$
        \EndFor 
\label{asdfsafdasfa2}        
\end{algorithmic}
\end{algorithm}
To explain the main idea behind the lazy/delayed updates,
consider that it happened that
during the fist $\tau$ iterations, 
the value of the fist coordinate in all datapoints which we have used was
0.
Then given the values of $\vc{y_{k,0}}{1}$
and $\vc{g_k}{1}$ we can compute the  true value
of  $\vc{y_{k,t}}{1}$ easily.
We just need to apply the $\prox$ operator $\tau$ times, i.e. 
$
\vc{y_{k,\tau}}{1}
 = \prox^\tau_1 [ y_{k,0} , g_k ,R,h],
$
where 
the function 
$\prox^\tau_1$ is described in Algorithm \ref{asdfsafdsaf}.
\begin{algorithm}
\caption{$\prox^{\tau}_j[y,g,R,h]$}
\label{asdfsafdsaf}
\begin{algorithmic}
\State $\tilde y_0 = y$
\For{
    $s=1,2,\dots, \tau$}
\State $\tilde y_s
  \gets 
  \prox_{hR}( \tilde y_{s-1}  
   - h 
   g  )$
  \EndFor
\State {\bf return}   $\vc{\tilde y_\tau}{j}$
\end{algorithmic}
\end{algorithm}

The vector $\chi$ in Algorithm \ref{alg:mS2GDsparse}
is enabling us to keep track of the iteration when corresponding coordinate of $y$ was updated for the last time.
E.g. if in iteration $t$ we will be updating the $1^{st}$ coordinate for the first time, $\vc{\chi}{1}=0$
and after we compute and update the true value of $\vc{y}{1}$, its value will be set to 
$\vc{\chi}{1}=t$.
Lines \ref{asdfasfsafdsafa}-\ref{asdfasfsafdsafa2}
in 
 Algorithm \ref{alg:mS2GDsparse}
 make sure 
 that the coordinates of $y_{k,t}$ which will be read and used afterwards
 are up-to-date. 
At the end of the inner loop, we
 will updates all coordinates of $y$ to the most recent value (lines \ref{asdfsafdasfa}-\ref{asdfsafdasfa2}). Therefore, those lines make sure that the $y_{k,t_k}$ of Algorithms \ref{alg:mS2GD} and \ref{alg:mS2GDsparse} will be the same.

However, one could claim that we are not saving any work, as when needed, we still have to compute the proximal operator many times. Although this can be true for a general function $R$, for particular cases,
$R(x) = \frac{\lambda}{2} \|x\|^2$
and $R(x) = \lambda \|x\|_1^2$, we provide following Lemmas which give a closed form expressions for the $\prox^\tau_j$ operator. 

\begin{figure*}[!htbp]
   \centering
    \includegraphics[width=0.40\textwidth]{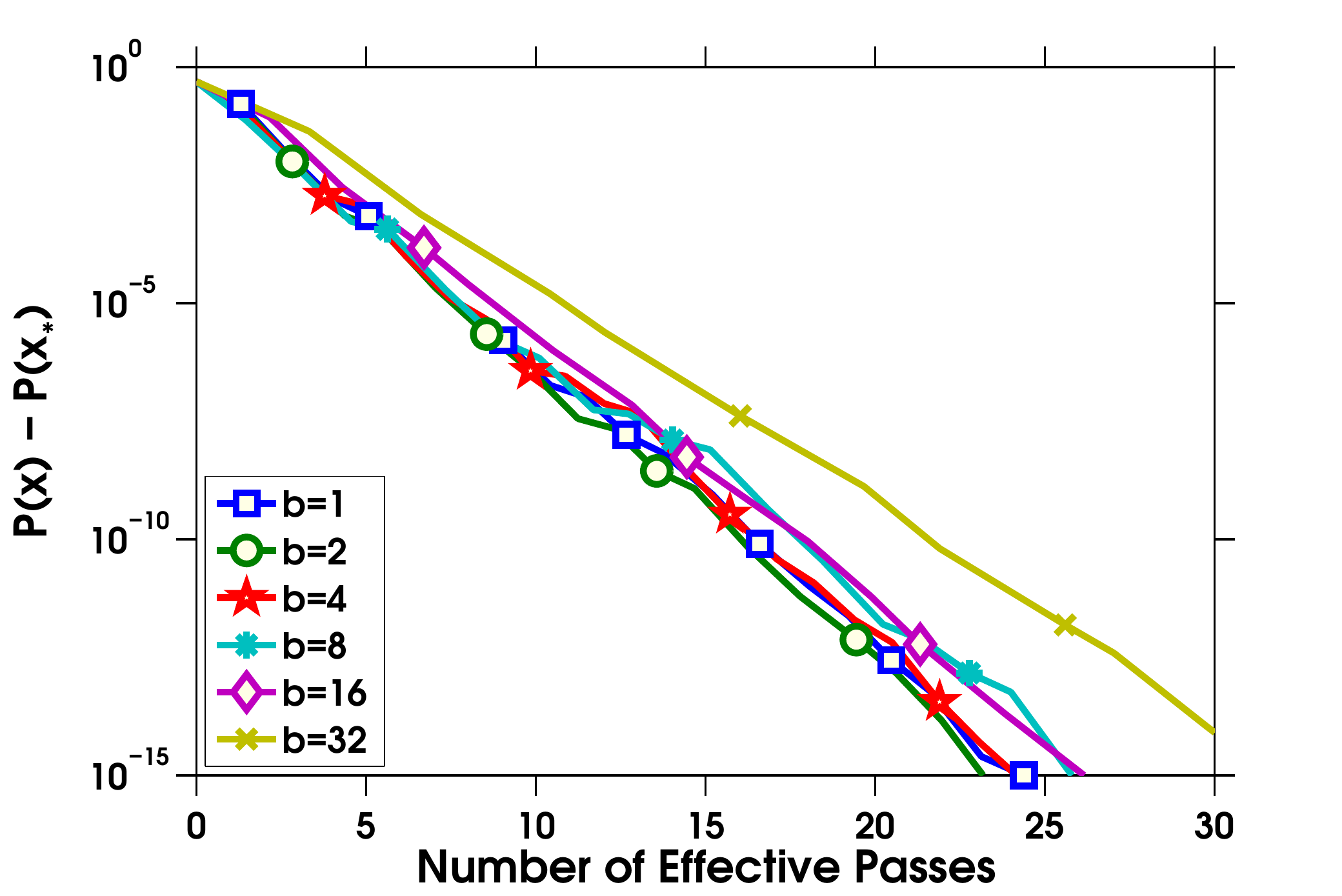}
    \includegraphics[width=0.40\textwidth]{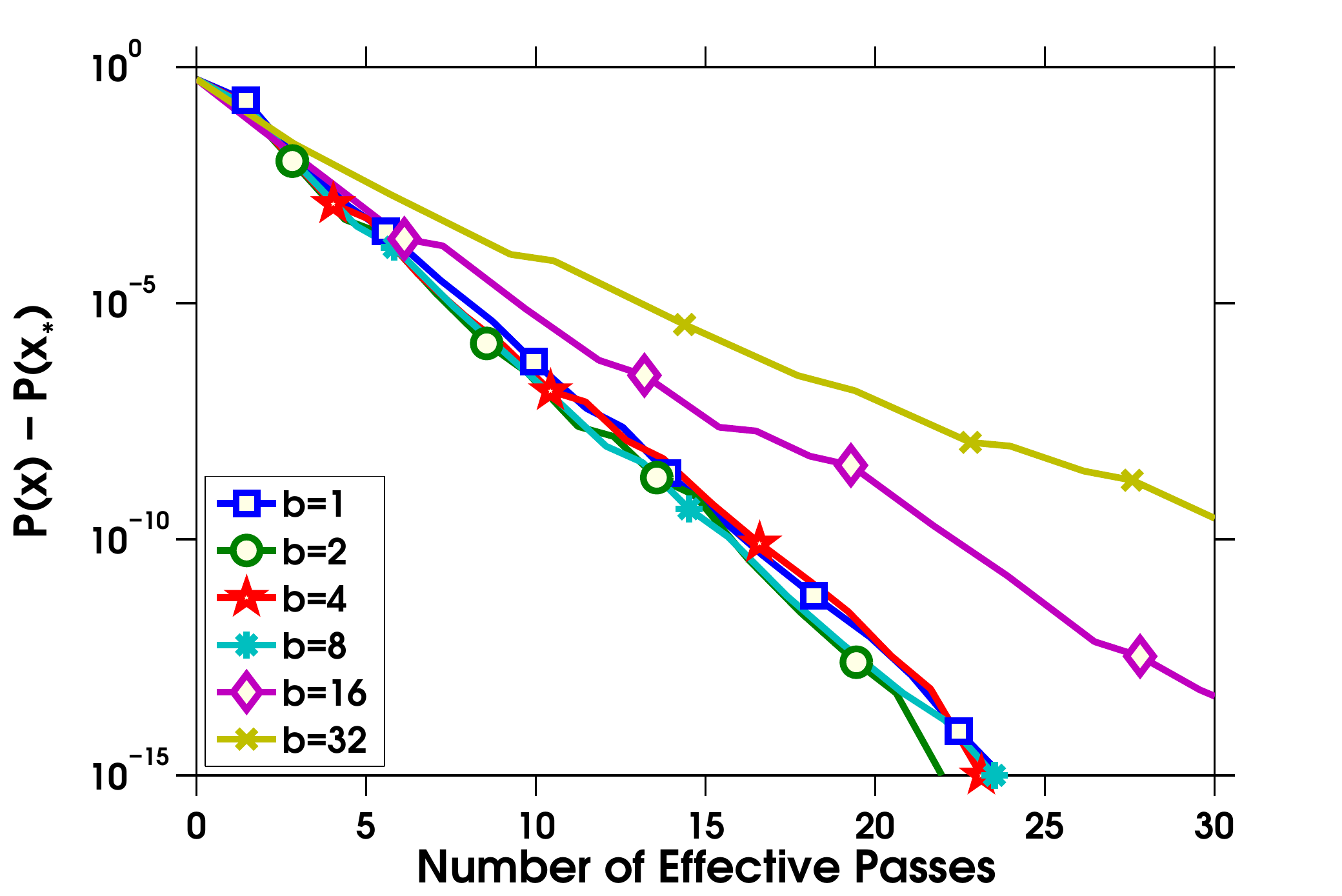}
    \caption{\footnotesize Comparison of mS2GD with different mini-batch sizes on \emph{rcv1} (left) and \emph{astro-ph} (right).}
  \label{fig:MBSpeedup} 
\end{figure*}

\begin{figure*}[!htbp]
    \centering
    \includegraphics[width=0.40\textwidth]{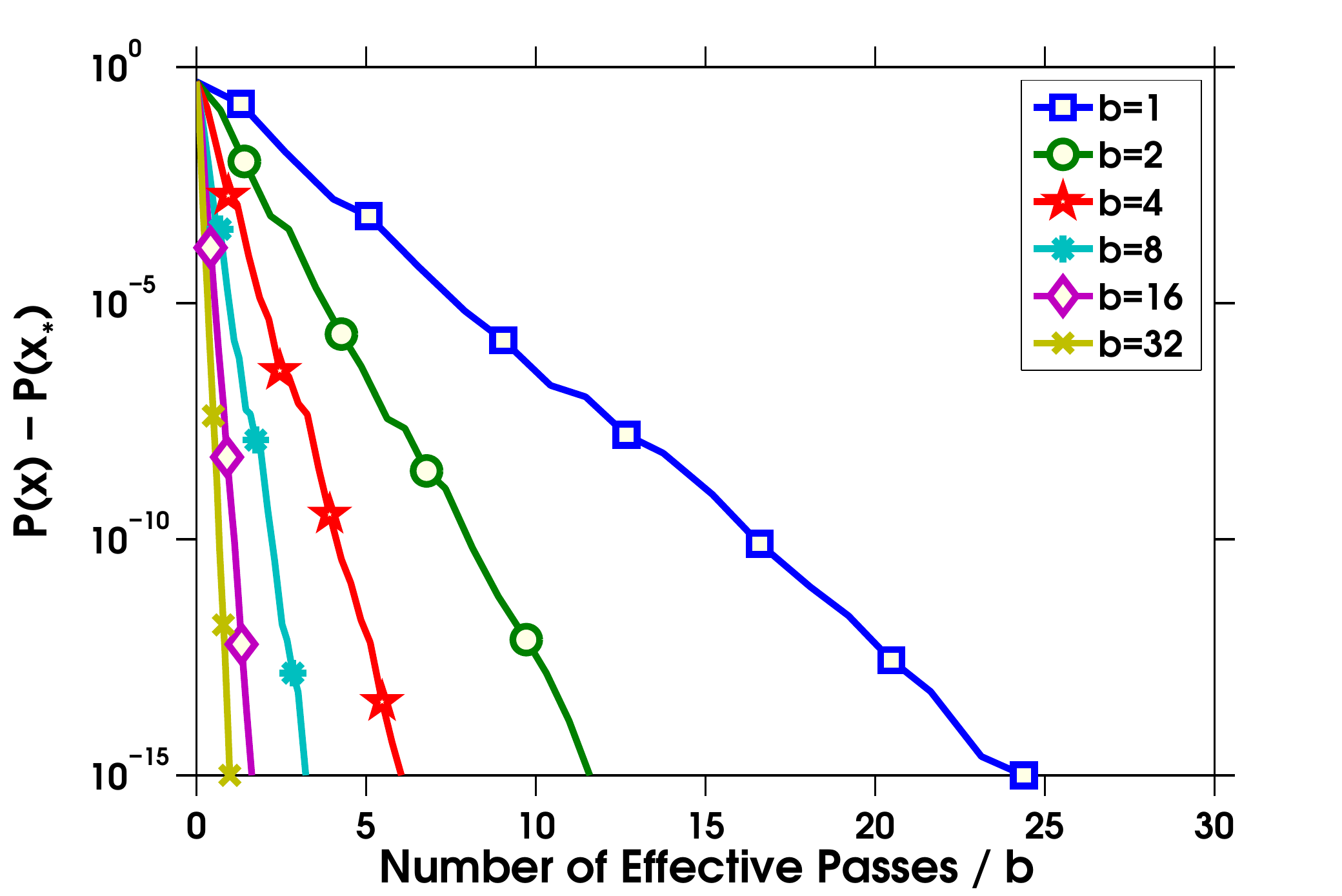}
    \includegraphics[width=0.40\textwidth]{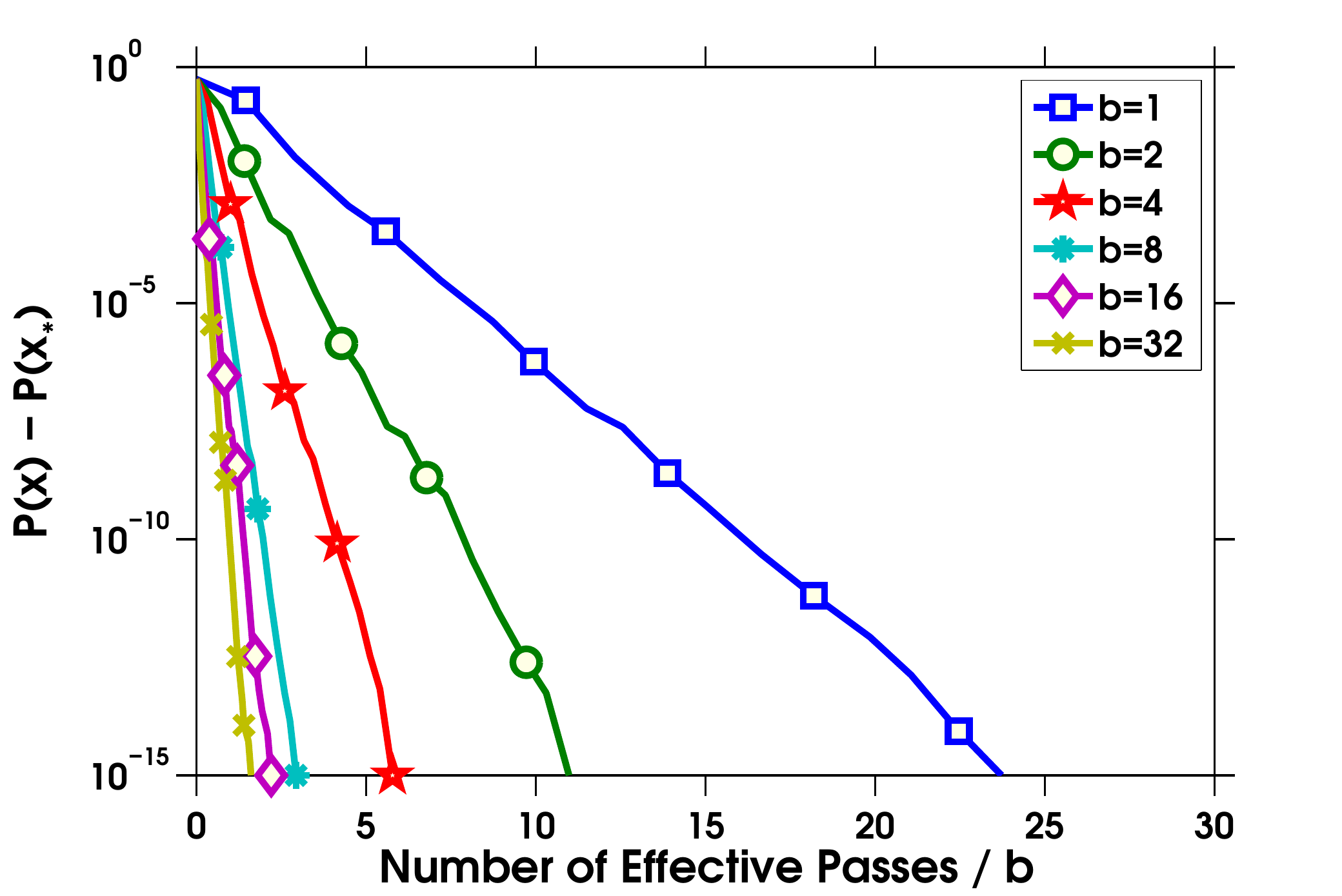}
  \caption{\footnotesize Parallelism speedup   for \emph{rcv1} (left) and \emph{astro-ph} (right) in theory (unachievable in practice).}
  \label{fig:MBSpeedup with parallelism} 
\end{figure*}

\begin{lemma}[Proximal Lazy Updates with $\ell_2$-Regularizer]\label{lemma:L2regularizer}
If $R(x) = \frac{\lambda}{2} \|x\|^2$
with $\lambda > 0$ then
\begin{align*}
\prox_j^\tau[ y, g, R, h]
= \beta^{\tau}  y^j - \tfrac{\stepsize\beta}{1-\beta}
                \left(1 - \beta^{\tau} \right)  \vc{g}{j} ,
\end{align*}
where  $\beta \eqdef 1/(1+\lambda \stepsize)$.
\end{lemma}

\begin{lemma}[Proximal Lazy Updates with \add{$\ell_1$-}Regularizer]\label{lemma:L1regularizer}
Assume that $R(x) =  \lambda  \|x\|_1 $
with $\lambda > 0$.
Let us define 
$M$ and $m$ as follows,
\begin{equation*}
M=[\lambda+ \vc{g}{j}]h,  \qquad m = -[\lambda - \vc{g}{j}]h,
\end{equation*}
and let $[\cdot]_+ \eqdef \max\{\cdot, 0\}$.
Then the value of
$\prox_j^\tau[ y, g, R, h]$
can be expressed based on one of the 3 situations described below:
\begin{enumerate}
\item If $\vc{g}{j} \geq \lambda$, then by letting $p \eqdef \left\lfloor \frac{\vc{y}{j}}{M}\right\rfloor$, the operator can be defined as
\begin{flalign*}
&\prox_j^\tau[ y, g, R, h] 
\\&=\begin{cases}
\vc{y}{j} - \tau M, &\text{if }p\geq\tau,\\
\min\{y^j - [p]_+ M, m\} -(\tau - [p]_+)m, &\text{if }p<\tau.
\end{cases} 
\end{flalign*}

\item If $-\lambda < \vc{g}{j} < \lambda$, then the operator can be defined as
\begin{flalign*}
\prox_j^\tau[ y, g, R, h] &
&=\begin{cases}
\max\{ \vc{y}{j} - \tau M, 0\}, & \text{if } \vc{y}{j}\geq 0,\\
\min\{ \vc{y}{j} - \tau m, 0\}, & \text{if }\vc{y}{j}< 0.
\end{cases}
\end{flalign*}

\item If $\vc{g}{j}  \leq -\lambda$, then by letting $q \eqdef \left\lfloor\frac{\vc{y}{j}}{m}\right\rfloor$, the operator can be defined as
\begin{flalign*}
&\prox_j^\tau[ y, g, R, h]  
\\&=\begin{cases}
\vc{y}{j} - \tau m, &\text{if }q\geq\tau,\\
 \max\{y^j - [q]_+ m, M\} -(\tau-[q]_+)M, &\text{if }q<\tau.
\end{cases} 
\end{flalign*}
\end{enumerate}

\end{lemma}

The proofs of Lemmas~\ref{lemma:L2regularizer} and~\ref{lemma:L1regularizer} are available in APPENDIX~\ref{section:proxl}.

{\em Remark:} Upon completion of the paper, we learned that similar ideas of lazy updates were proposed in \cite{langford2009} and \cite{carpenter2008} for online learning and multinomial logistic regression, respectively. However, our method can be seen as a more general result applied to a stochastic gradient method and its variants under Assumptions  \ref{asm:structure} and  \ref{asm:structure2}.

 \begin{figure*}[!htbp]
\centering
 \includegraphics[width=0.40\textwidth]{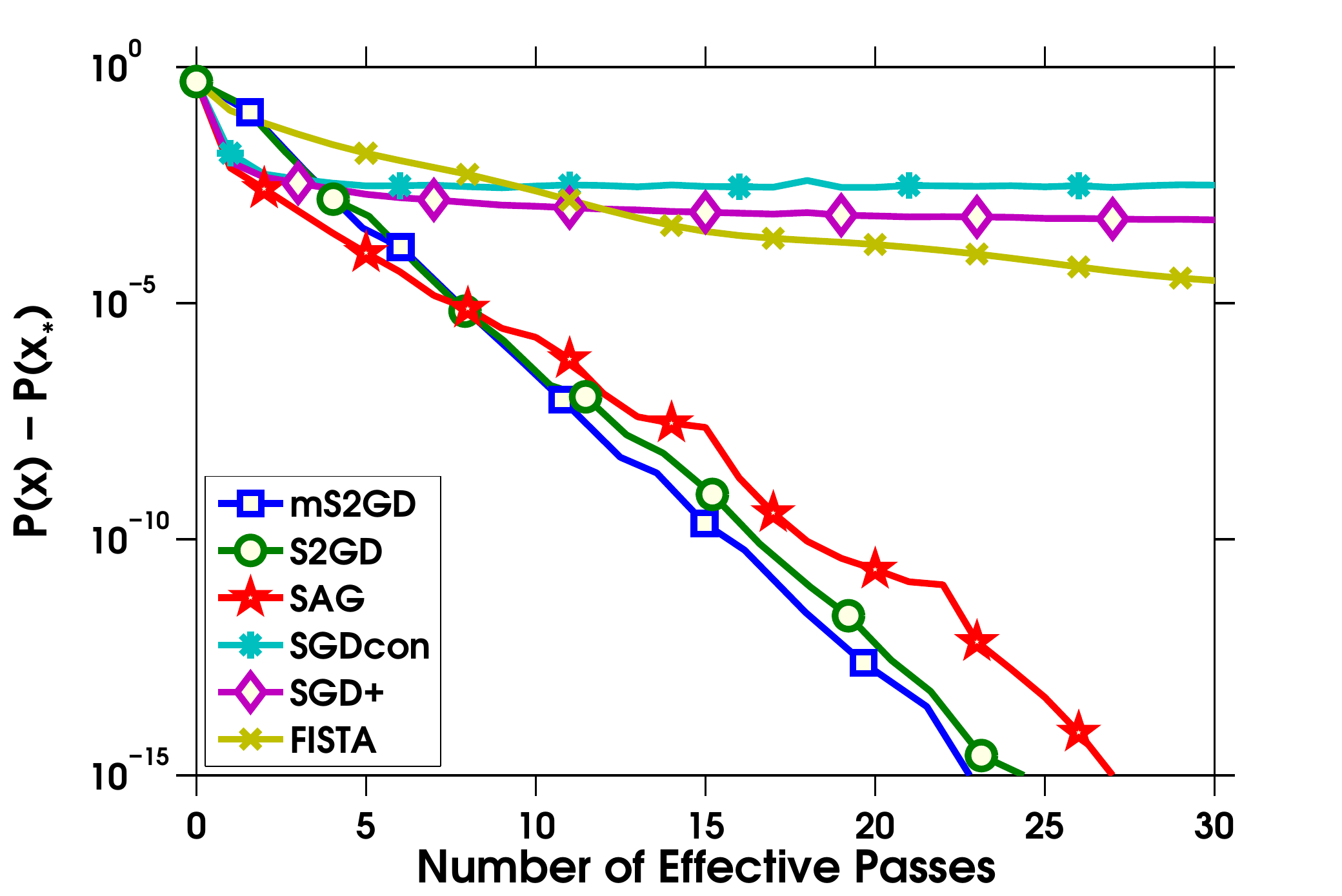}
 \includegraphics[width=0.40\textwidth]{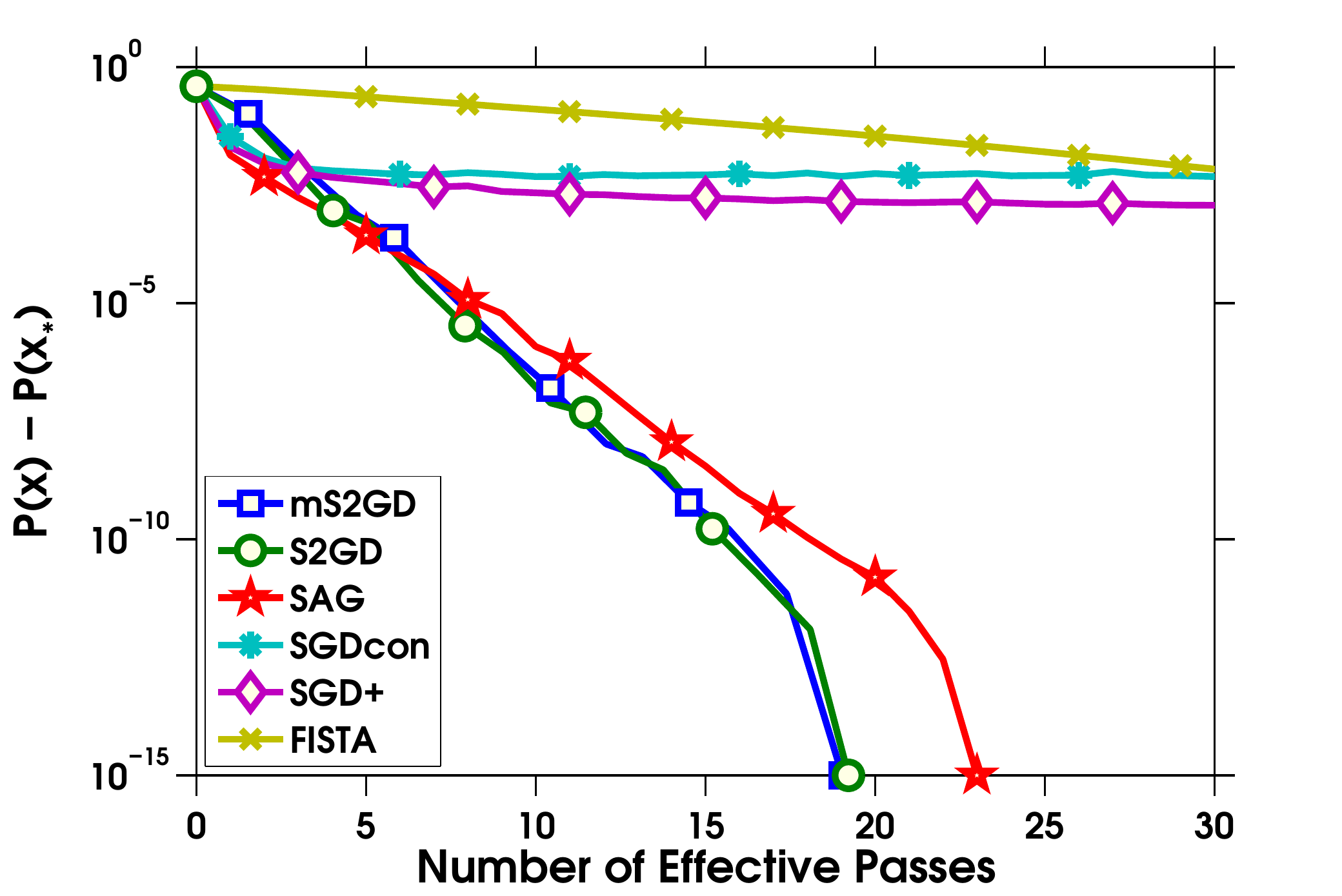}
 \\
 \includegraphics[width=0.40\textwidth]{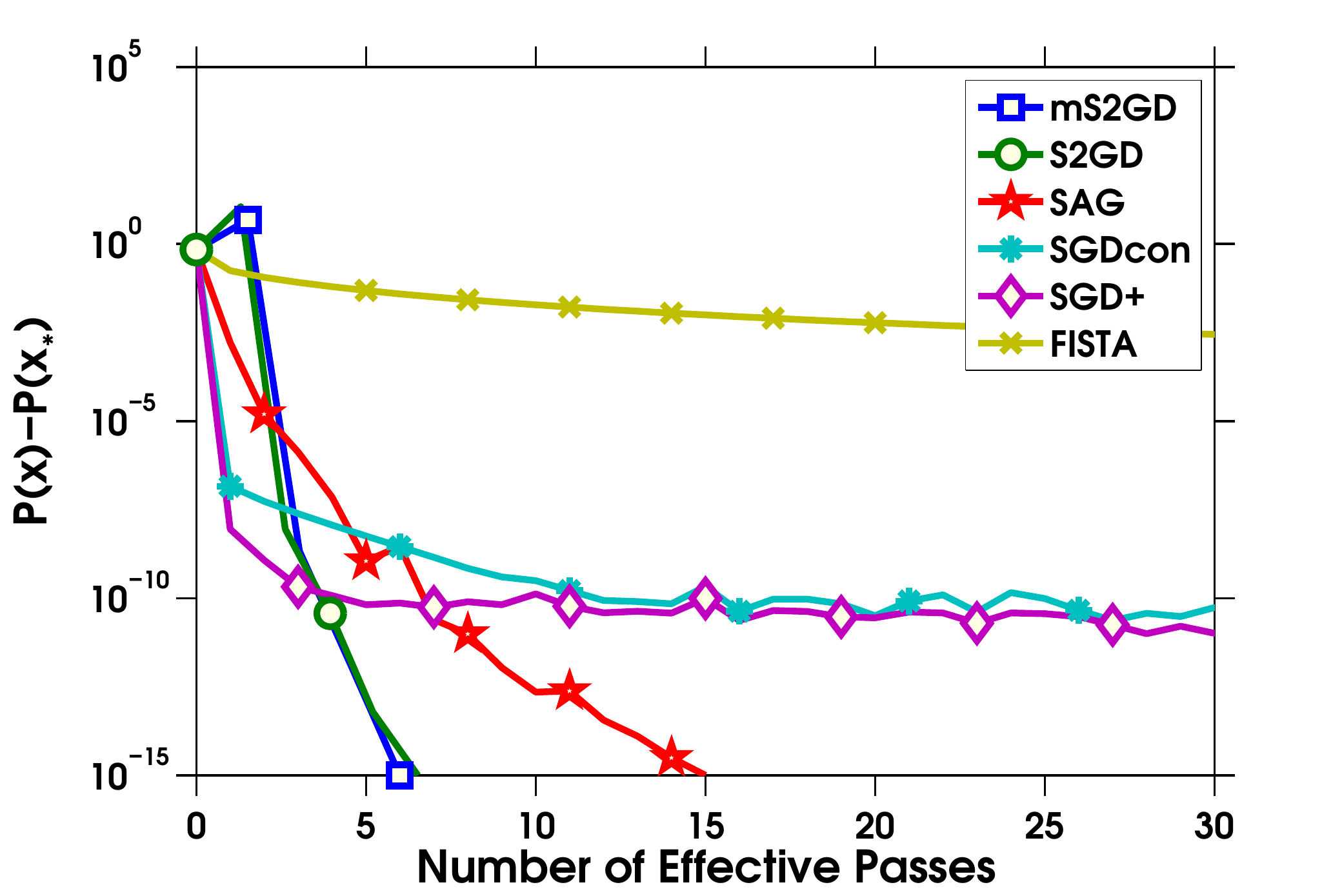}
 \includegraphics[width=0.40\textwidth]{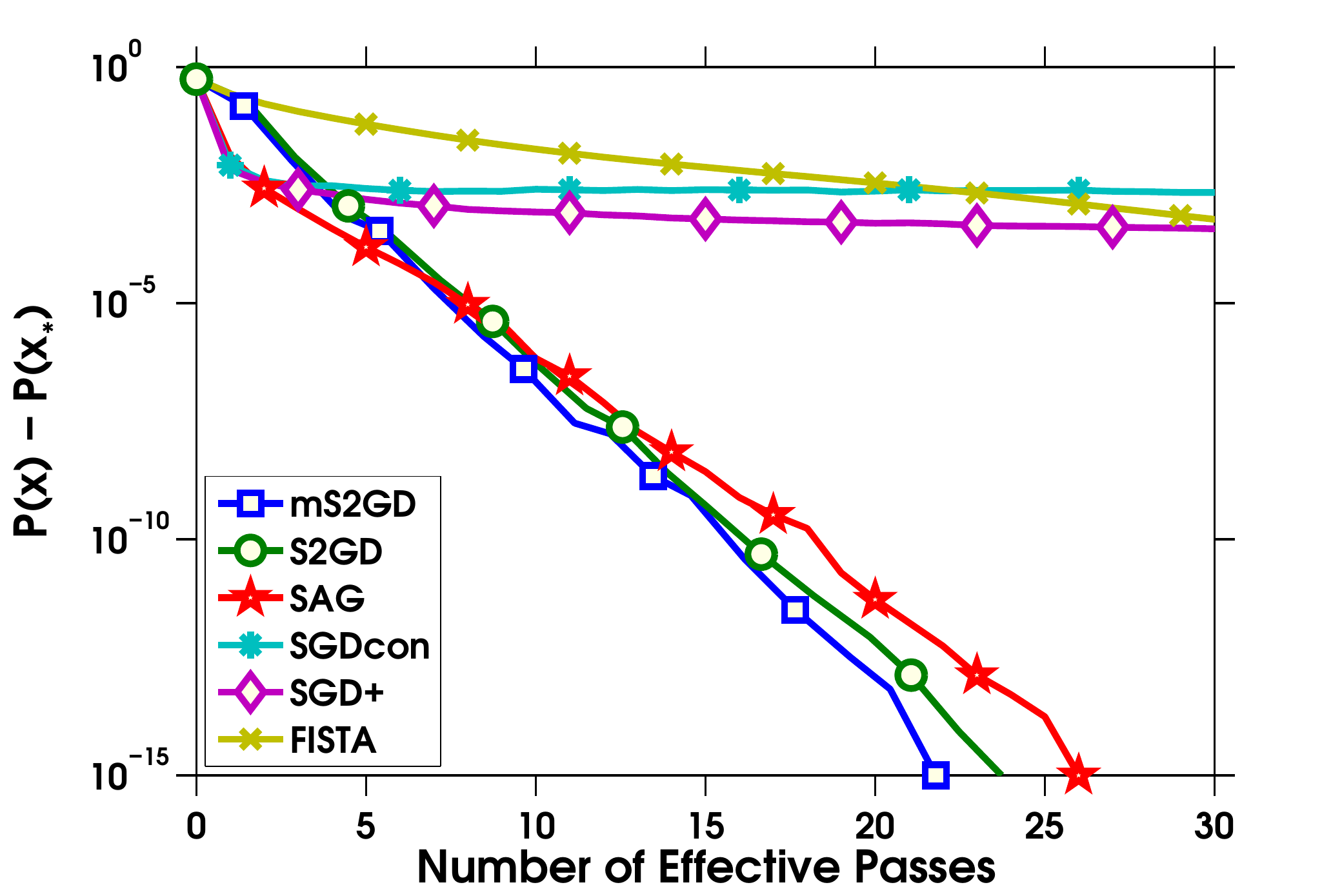}
 \caption{\footnotesize Comparison of several algorithms on four datasets: \emph{rcv1} (top left), \emph{news20} (top right), \emph{covtype} (bottom left) and \emph{astro-ph} (bottom right). We have used  mS2GD with $b=8$.}
   \label{fig:comparison_algorithms}
 \end{figure*}


\section{Experiments}

\add{In this section we perform numerical experiments to illustrate the properties and performance of our algorithm.  In Section~ \ref{mS2GD_properties} we study the total workload and parallelization speedup of mS2GD as a function of the mini-batch size $b$. In  Section~\ref{mS2GD_algorithms} we compare  mS2GD with several other algorithms.
Finally, in Section~\ref{sec:deblur}  we briefly illustrate that our method can be efficiently applied to a deblurring problem.
}

In Sections~\ref{mS2GD_properties} and \ref{mS2GD_algorithms} we conduct experiments with $R(x)=\tfrac{\lambda}{2}\|x\|^2$ and $F$ of the form \eqref{Px2}, where $f_i$ is the logistic loss function:
\begin{equation}\label{logisticL2}
f_i(x) = \log [1+\exp(-b_i a_i^Tx)].
\end{equation}
These functions are often used in machine learning, with  $(a_i, b_i) \in \R^d\times \{+1,-1\}$, $i=1,\dots,n$, being a training dataset of example-label pairs. The resulting optimization problem  \eqref{Px1}+\eqref{Px2} takes the form
\begin{equation}\label{L2problem}
\textstyle P(x) = \tfrac1n\sum_{i=1}^n f_i(x) + \tfrac{\lambda}{2}\|x\|^2,
\end{equation}
and is used in machine learning for binary classification. In these sections we have performed experiments on four publicly available binary classification datasets, namely \emph{rcv1,  news20, covtype}~\footnote{\emph{rcv1, covtype} and \emph{news20} are available at \burl{http://www.csie.ntu.edu.tw/~cjlin/libsvmtools/datasets/}.} and \emph{astro-ph}~\footnote{Available at \burl{http://users.cecs.anu.edu.au/~xzhang/data/}.}.


In the logistic regression problem, the Lipschitz constant of function $\nabla f_i$ is equal to $L_i = \|a_i\|^2/4$. Our analysis assumes (Assumption~\ref{ass1}) the same constant $L$ for all functions. Hence, we have $L = \max_{i\in\setn} L_i$. We set the regularization parameter $\lambda = \frac1n$ in our experiments, resulting in the problem having the condition number $\kappa = \frac{L}{\mu} = \mathcal{O}(n)$. In Table~\ref{table: datasets} we  summarize  the four datasets, including the sizes $n$, dimensions $d$, their sparsity levels as a proportion of nonzero elements, and the Lipschitz constants $L$.

\begin{table}[H]
\small
\centering
\begin{tabular}{|c|c|c|c|c|c|}
\hline
Dataset  & $n$ & $d$ & Sparsity & $L$  \\
\hline \hline 
\emph{rcv1} & 20,242 & 47,236 & 0.1568\% & 0.2500\\
\hline 
\emph{news20} & 19,996  & 1,355,191  & 0.0336\% & 0.2500\\
\hline 
\emph{covtype} & 581,012 &  54  & 22.1212\% & 1.9040\\
\hline 
\emph{astro-ph} & 62,369  & 99,757  & 0.0767\% & 0.2500\\
\hline
\end{tabular}
\captionsetup{justification=centering,margin=0.5cm}
\caption{Summary of datasets used for experiments.}
\label{table: datasets}
\end{table}
 

\subsection{Speedup of mS2GD}\label{mS2GD_properties}
Mini-batches allow mS2GD to be accelerated on a  computer with a parallel processor. In Section~\ref{sec:speedup}, we have shown in  that up to some threshold mini-batch size, the total workload of mS2GD remains unchanged. Figure~\ref{fig:MBSpeedup} compares the best performance of mS2GD used with various  mini-batch sizes on datasets \emph{rcv1} and \emph{astro-ph}. An effective  pass (through the data) corresponds to   $n$ units of work. Hence,  the evaluation of a gradient of $F$ counts as one effective pass. In both cases, by increasing the mini-batch size to $b=2, 4, 8$, the performance of mS2GD is the same or better than that of S2GD ($b=1$) without any parallelism. 

Although for larger mini-batch sizes mS2GD would be obviously worse, the results are still promising with parallelism. In Figure~\ref{fig:MBSpeedup with parallelism},we show the ideal speedup---one that  would be achievable if we could always evaluate the $b$ gradients  in parallel in exactly the same amount of time as it would take to evaluate a single gradient.\footnote{In practice, it is impossible to ensure that the times of evaluating different component gradients are the same.}.

\subsection{mS2GD vs other algorithms}\label{mS2GD_algorithms}
In this part, we implemented the following algorithms to conduct a numerical comparison:\\
1) \textbf{SGDcon}: Proximal stochastic gradient descent method with a constant step-size which gave the best performance in hindsight.\\
2) \textbf{SGD+}: Proximal stochastic gradient descent with variable step-size $h=h_0/(k+1)$, where $k$ is the number of effective passes, and $h_0$ is some initial constant step-size. \\
3) \textbf{FISTA}: Fast iterative shrinkage-thresholding algorithm proposed in~\cite{fista}. \\
4) \textbf{SAG}: Proximal version of the stochastic average gradient algorithm~\cite{SAG}. Instead of using $h=1/16L$, which is analyzed in the reference, we used a constant step size.\\
5) \textbf{S2GD}: Semi-stochastic gradient descent method proposed in~\cite{S2GD}. We applied proximal setting to the algorithm and used a constant stepsize.\\
6) \textbf{mS2GD}:  mS2GD with mini-batch size $b=8$. Although a safe step-size is given in our theoretical analyses in Theorem~\ref{s2convergence}, we ignored the bound, and used a constant step size.

In all cases, unless otherwise stated, we have used the best constant stepsizes in hindsight.

Figure~\ref{fig:comparison_algorithms} demonstrates the superiority of mS2GD over other  algorithms in the test pool on the four datasets described above. For mS2GD, the best choices of parameters with $b=8$ are given in Table~\ref{table: parameters}.
\begin{table}[H]
\small
\centering
\begin{tabular}{|c|c|c|c|c|c|}
\hline
Parameter & \emph{rcv1} & \emph{news20} & \emph{covtype} & \emph{astro-ph} \\
\hline \hline 
$ m$ & 0.11$n$ & 0.10$n$ & 0.07$n$ & 0.08$n$\\
\hline
  $h$ & $5.5/L$  & $6/L$  & $350/L$ & $6.5/L$\\
\hline
\end{tabular}
\captionsetup{justification=centering,margin=0.5cm}
\caption{Best choices of parameters in mS2GD.}
\label{table: parameters}
\end{table}

\subsection{Image deblurring}
\label{sec:deblur}
 \add{
In this section we utilize the Regularization Toolbox 
\cite{hansen2007regularization}
\footnote{Regularization Toolbox 
available for Matlab can be obtained from \url{http://www.imm.dtu.dk/~pcha/Regutools/} .} }
We use the \emph{blur} function 
available therein to obtain the original image and generate a blurred image (we choose following values of parameters for blur function: $N=256$, band=9, sigma=10). The purpose of the blur function is to generate a test problem with an atmospheric turbulence blur. In addition, an additive Gaussian white noise with stand deviation of $10^{-3}$ is added to the blurred image. This forms our testing image as a vector $b$. The image dimension of the test image is $256\times 256$, which means that $n=d=65,536$.  We would not expect our method to work particularly well on this problem since mS2GD works best when $d\ll n$. However, as we shall see, the method's performance is on a par with the performance of the best methods in our test pool.

Our  goal is to reconstruct (de-blur) the original image $x$ by solving a LASSO problem: $\min_x \| Ax - b\|_2^2+\lambda \|x\|_1$. We have chosen  $\lambda =10^{-4}$. In our implementation, we normalized the objective function by $n$, and hence our objective value  being optimized is in fact $\min_x \frac1n\| Ax - b\|_2^2+\lambda \|x\|_1$, where $\lambda = \frac{10^{-4}}{n}$,
similarly as was done in \cite{fista}.

Figure~\eqref{fig:afcewvcawfecvwa} shows the original test image (left) and a blurred image with added Gaussian noise (right).
Figure~\ref{fig:asfvfrawfrvwa} compares the mS2GD algorithm with SGD+, S2GD and FISTA. We run all algorithms for 100 epochs and plot the error. The plot suggests that SGD+ decreases the objective function very rapidly at beginning, but slows down after 10-20 epochs.

 \begin{figure} 
\centering

 \includegraphics[width=0.16\textwidth]{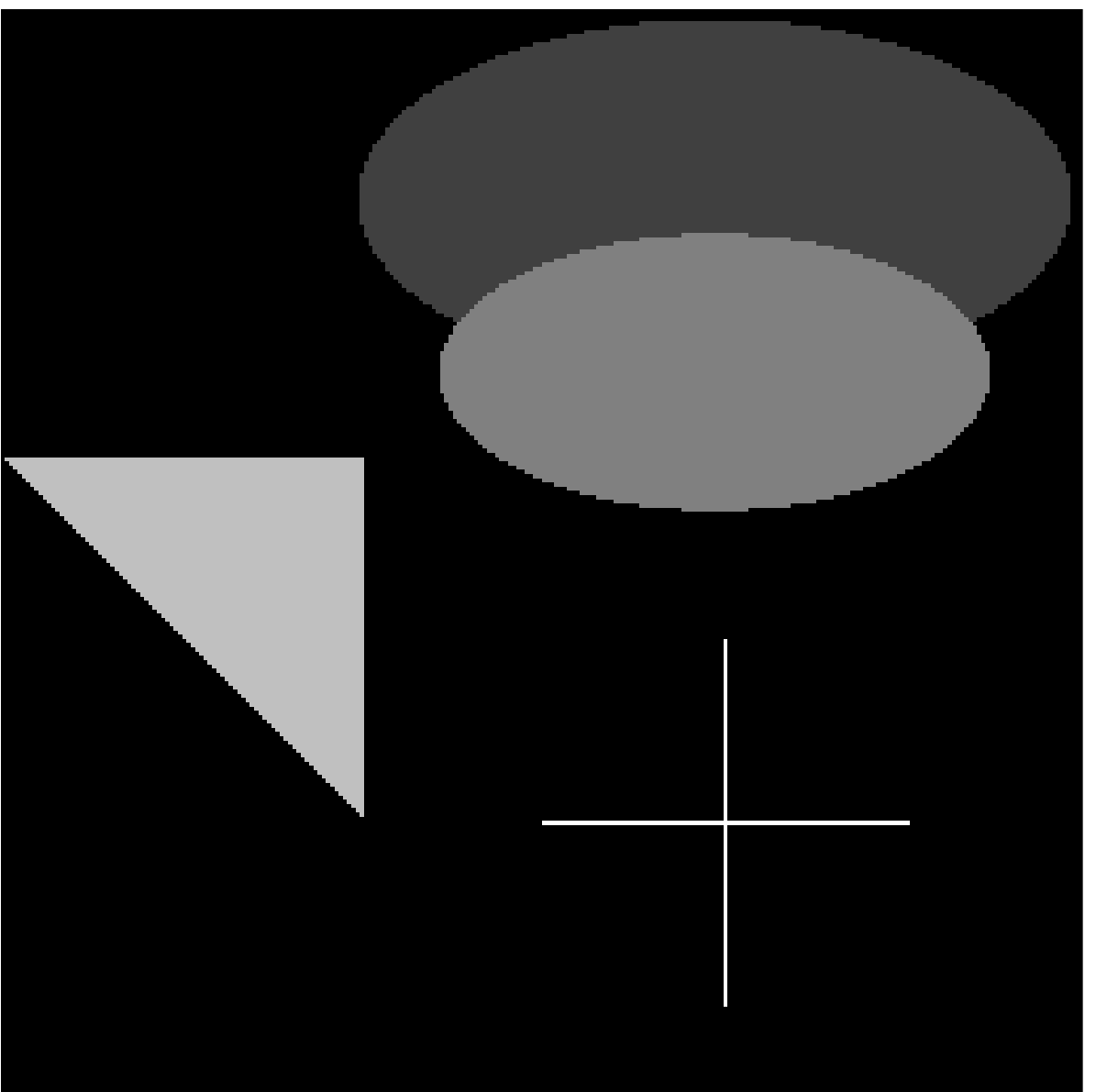}
 \includegraphics[width=0.16\textwidth]{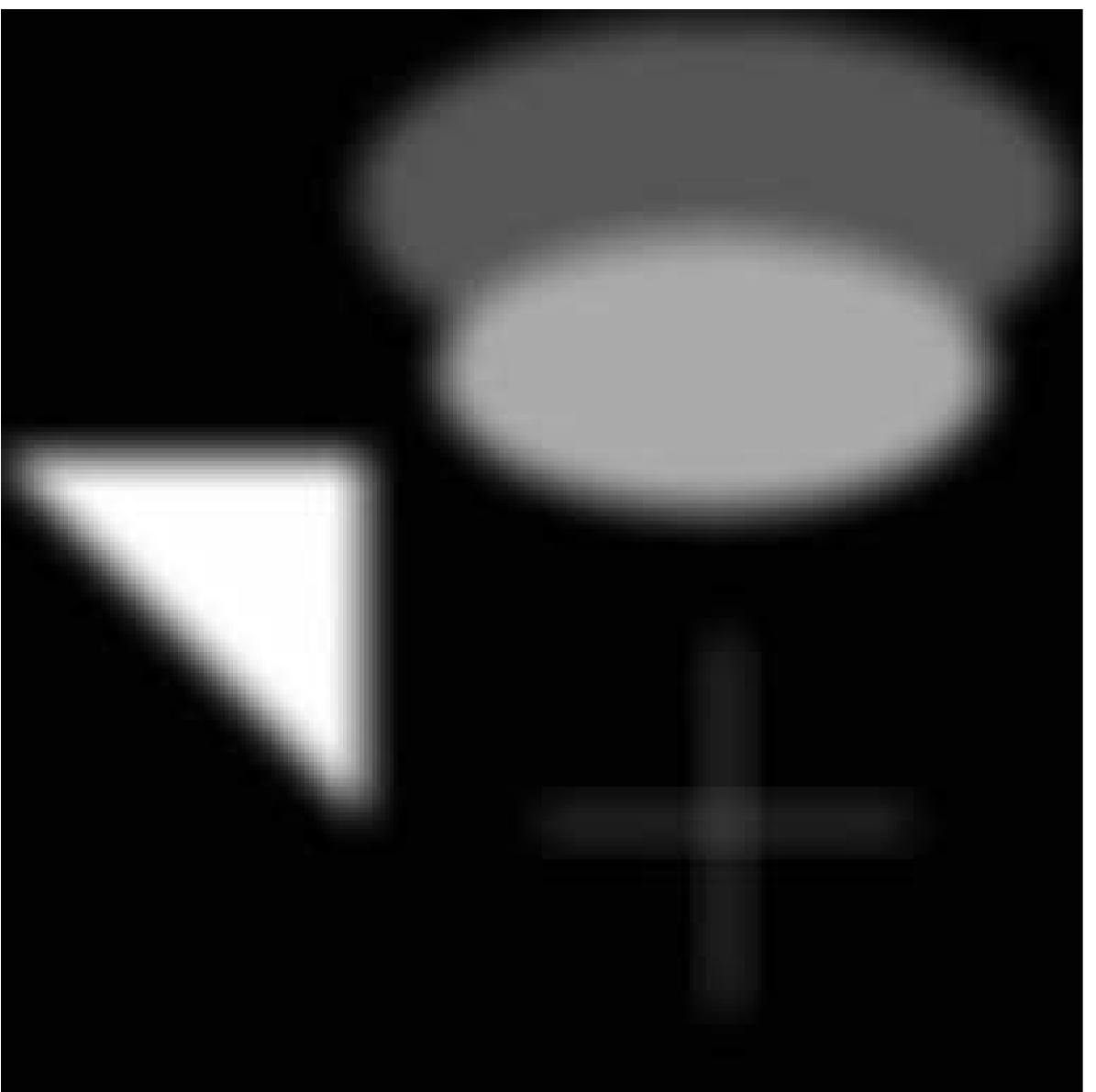}
  \caption{\footnotesize Original (left) and blurred \& noisy (right) test image.}
  \label{fig:afcewvcawfecvwa}
 \end{figure}

  \begin{figure} 

\centering
 \includegraphics[width=0.40\textwidth]{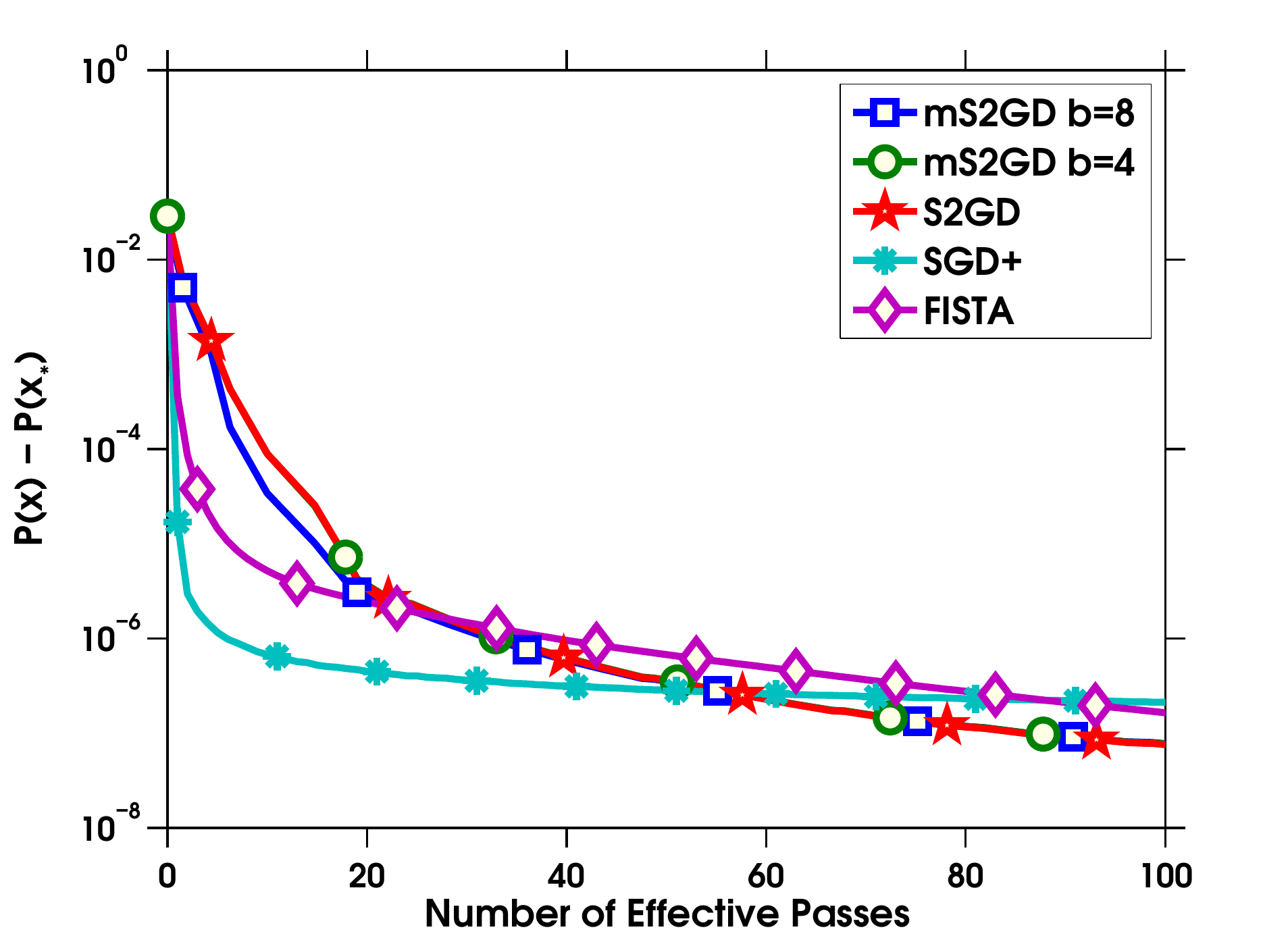}
\captionsetup{width=0.5\textwidth}
  \caption{\footnotesize Comparison of several algorithms for the de-blurring problem.}
 \label{fig:asfvfrawfrvwa}
 \end{figure}
 
 Finally, Fig.~\ref{fig:adsfaewfawefa}
 shows the reconstructed image after
 $T=20, 60, 100$ epochs.
 
  \begin{figure*}
\centering
    \begin{tabular}{r l}
    $T = 20$ \ &
        \begin{subfigure}{0.16\textwidth}
        \centering
        \caption*{FISTA}
        \includegraphics[width=1\textwidth]{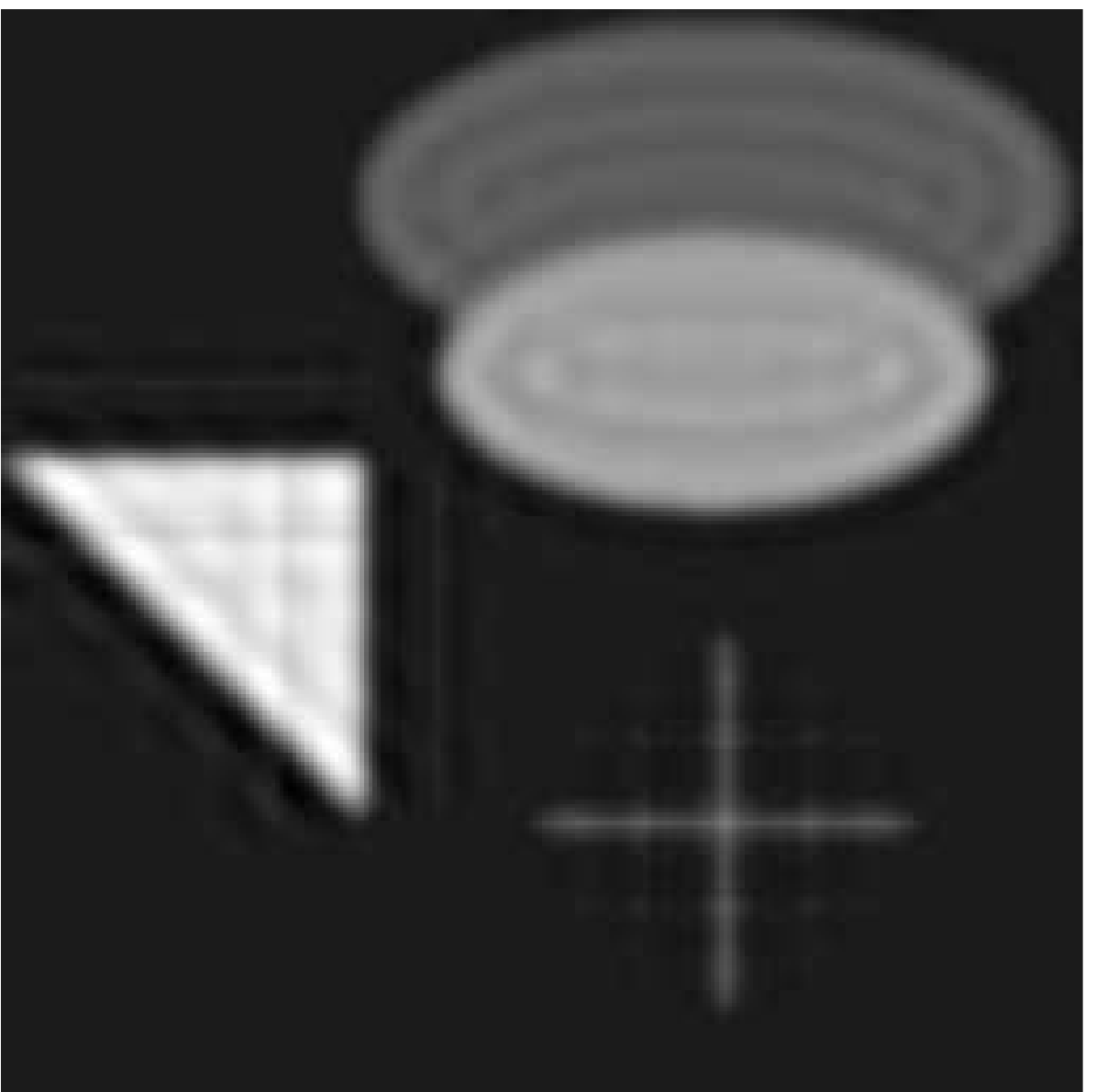}
        \end{subfigure}
        \begin{subfigure}{0.16\textwidth}
        \centering
        \caption*{SGD+}
        \includegraphics[width=1\textwidth]{FISTAsol20-eps-converted-to.pdf}
        \end{subfigure}
        \begin{subfigure}{0.16\textwidth}
        \centering
        \caption*{S2GD}
        \includegraphics[width=1\textwidth]{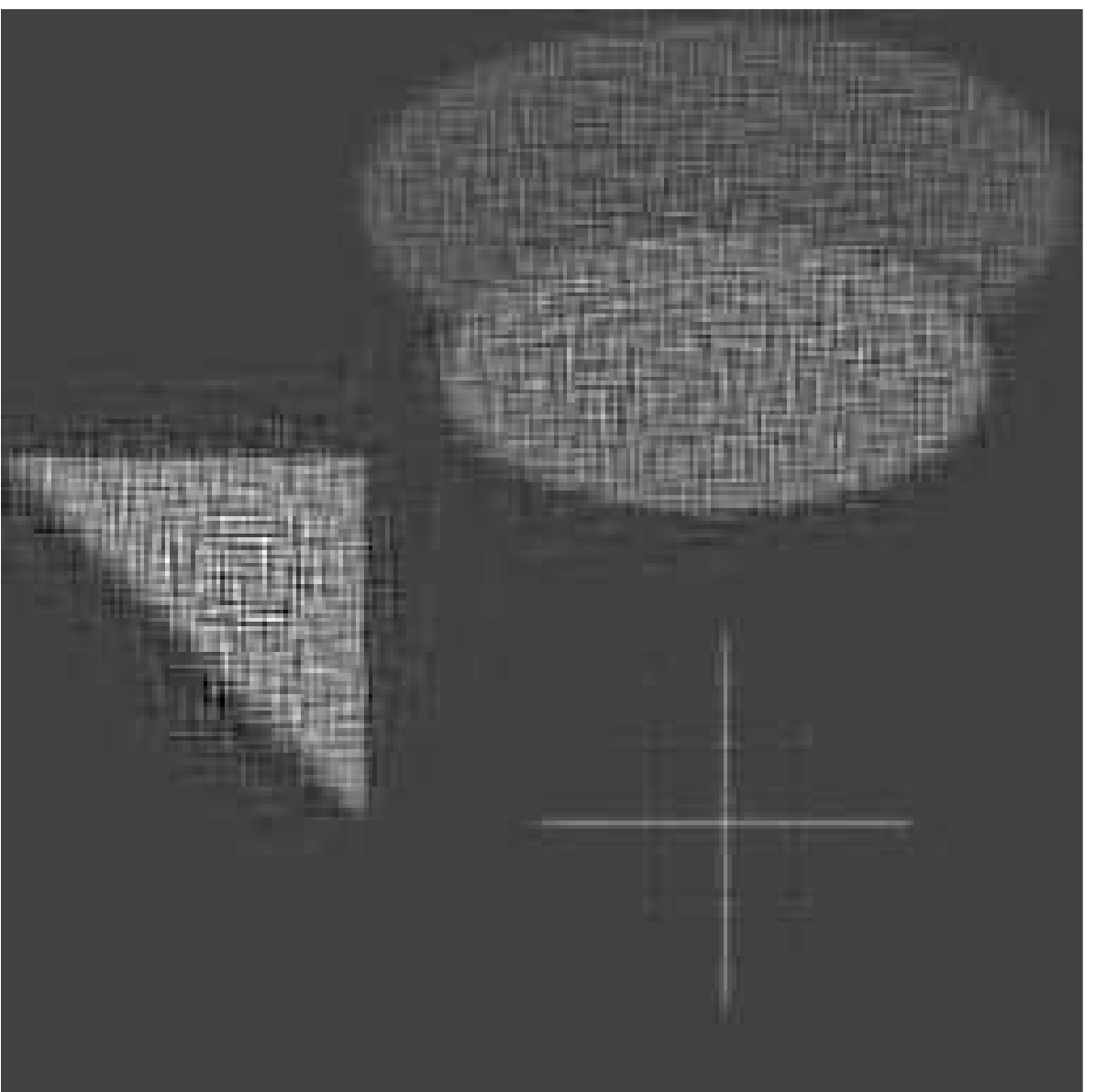}
        \end{subfigure}
        \begin{subfigure}{0.16\textwidth}
        \centering
        \caption*{mS2GD $(b=4)$}
        \includegraphics[width=1\textwidth]{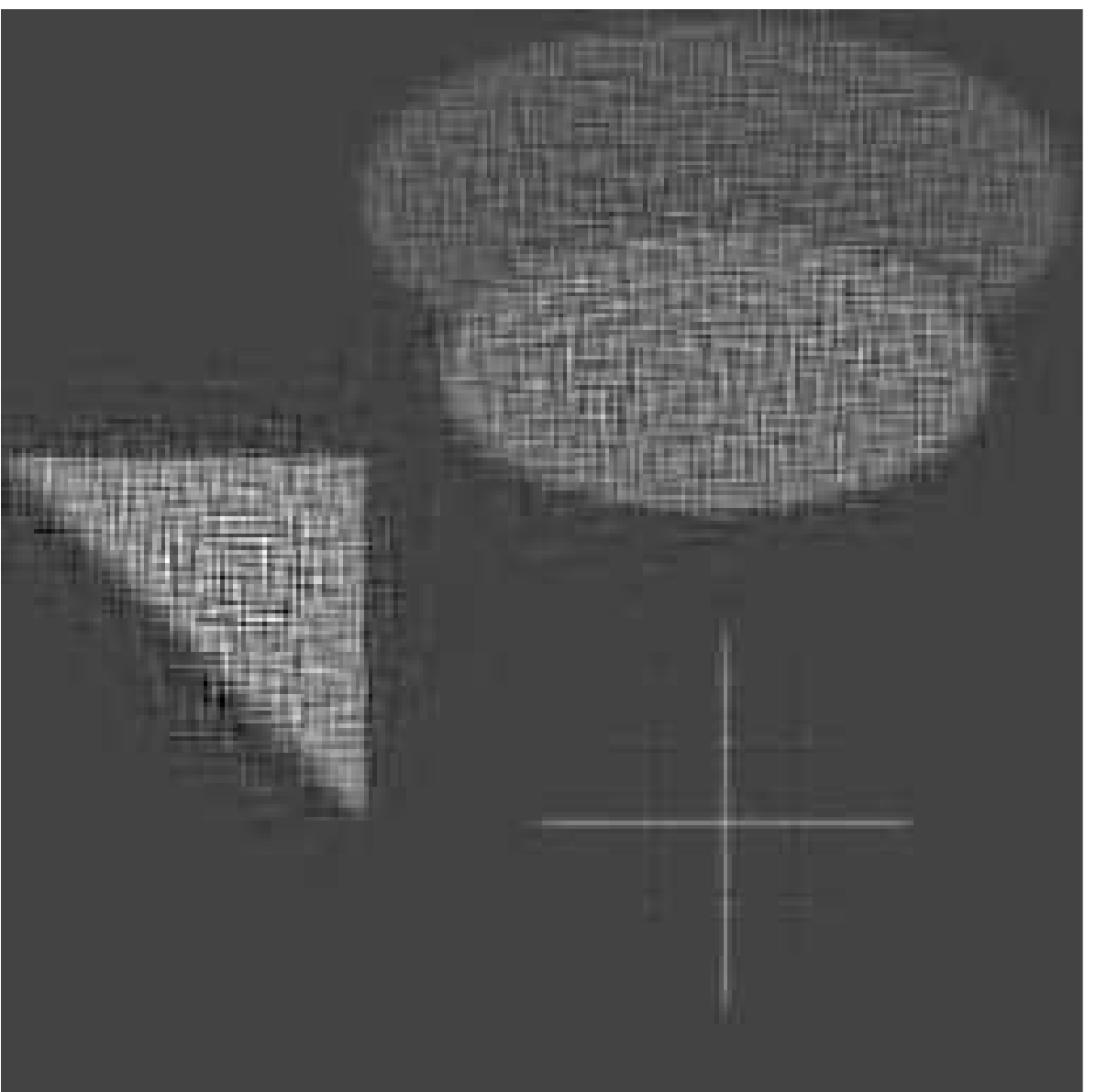}
        \end{subfigure}
        \begin{subfigure}{0.16\textwidth}
        \centering
        \caption*{mS2GD $(b=8)$}
        \includegraphics[width=1\textwidth]{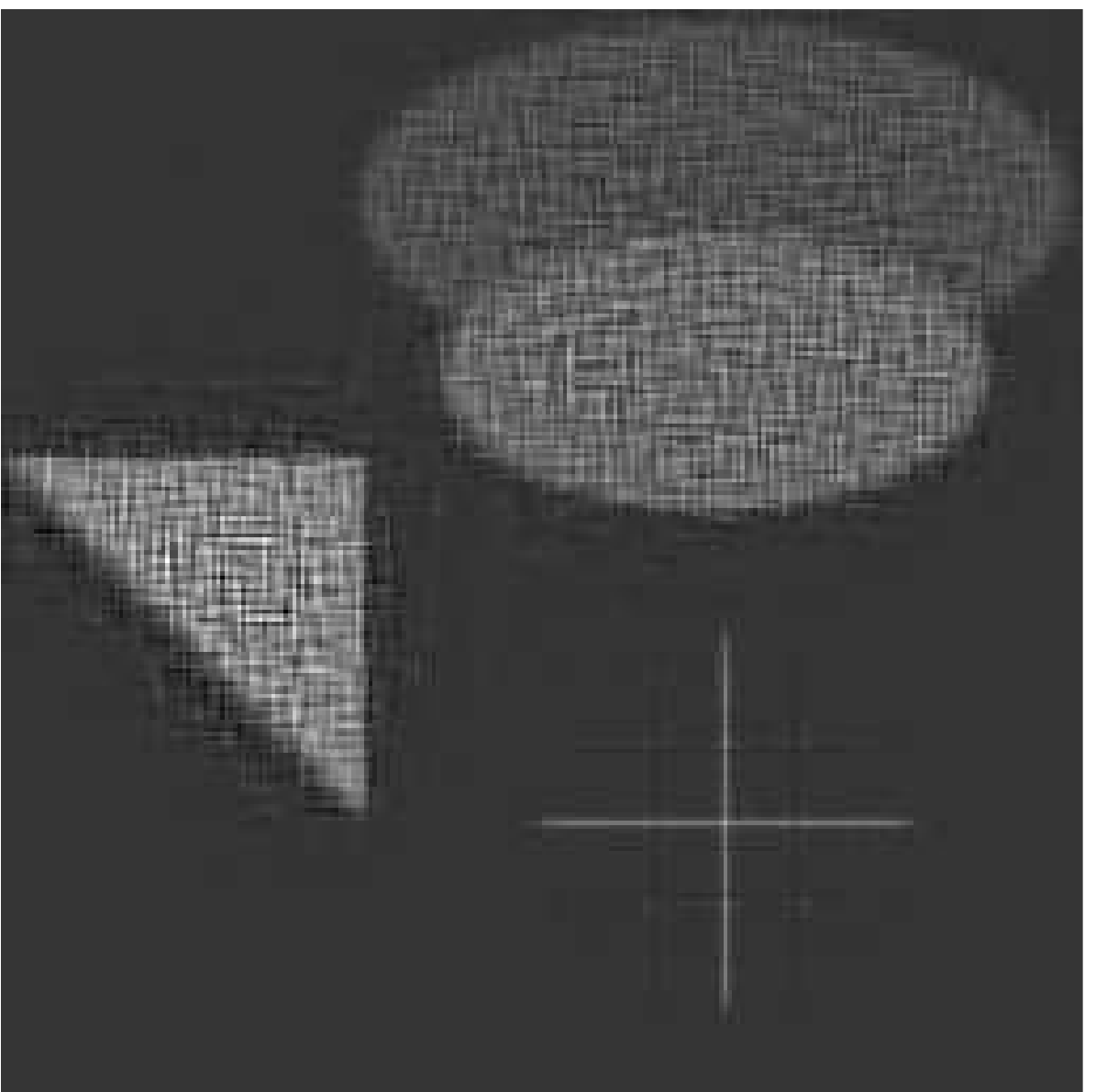}
        \end{subfigure}
  \\$T = 60$ \ &
        \begin{subfigure}{0.16\textwidth}
        \centering
        \includegraphics[width=1\textwidth]{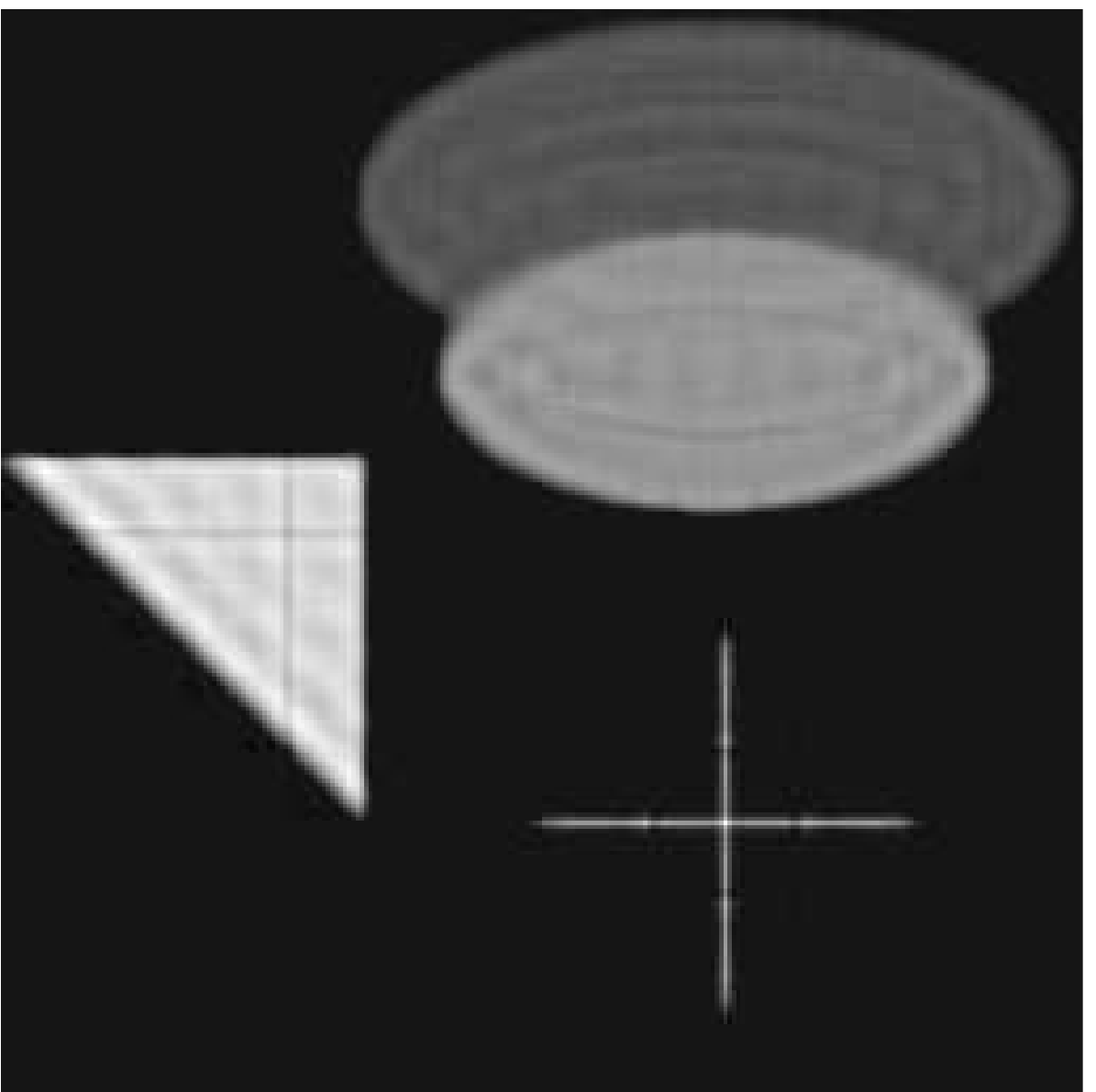}
        \end{subfigure}
        \begin{subfigure}{0.16\textwidth}
        \centering
        \includegraphics[width=1\textwidth]{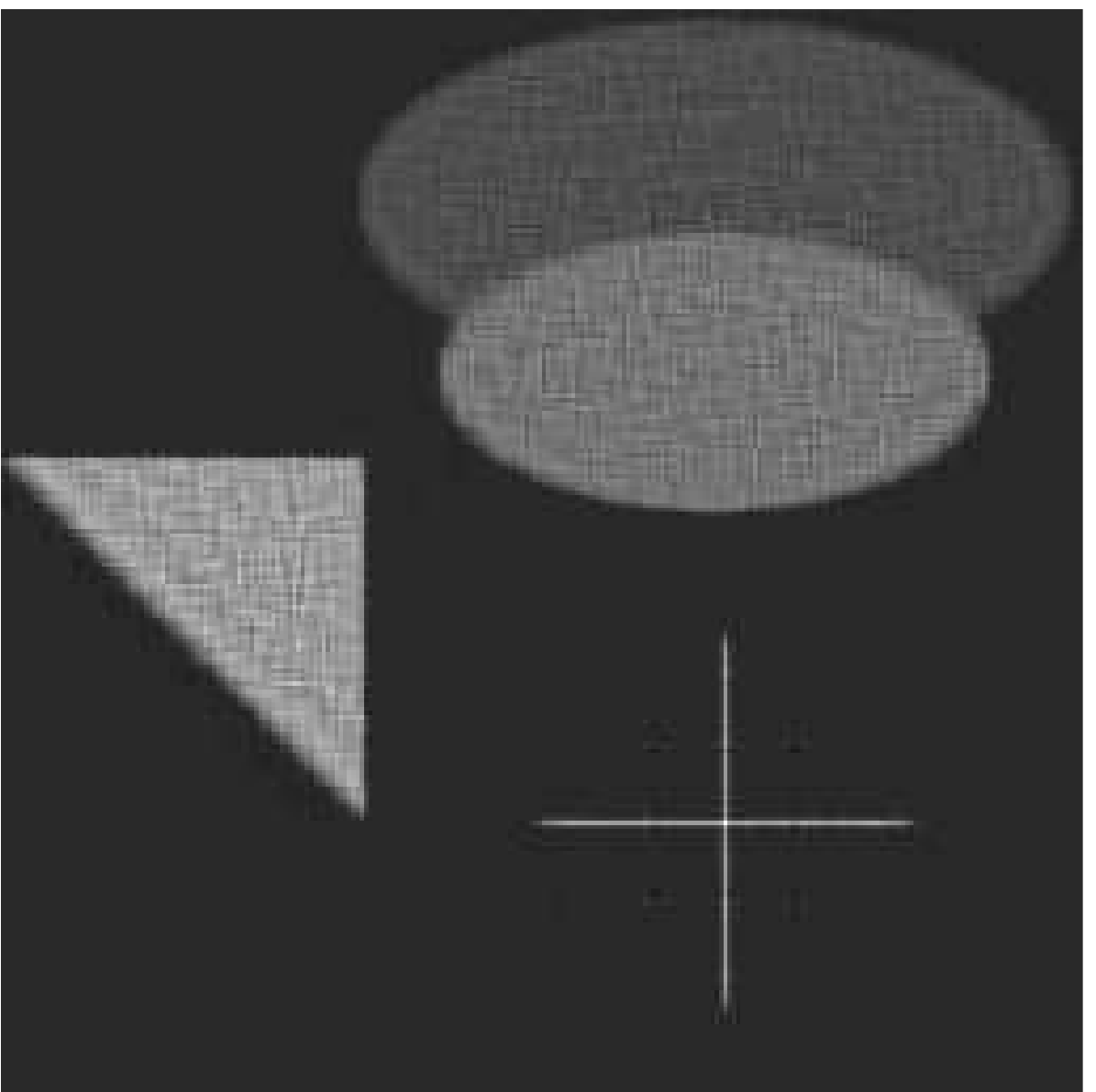}
        \end{subfigure}
        \begin{subfigure}{0.16\textwidth}
        \centering
        \includegraphics[width=1\textwidth]{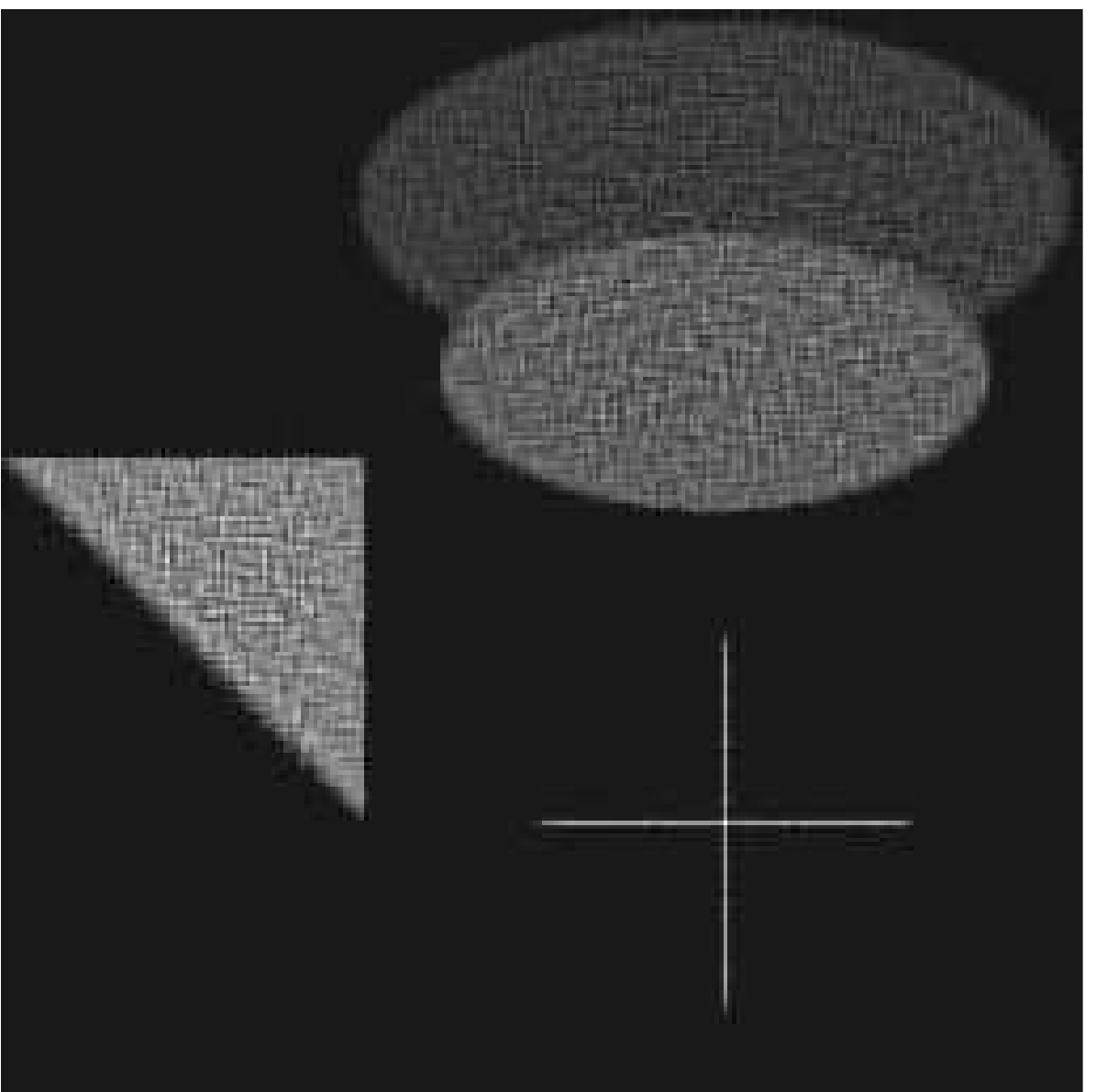}
        \end{subfigure}
        \begin{subfigure}{0.16\textwidth}
        \centering
        \includegraphics[width=1\textwidth]{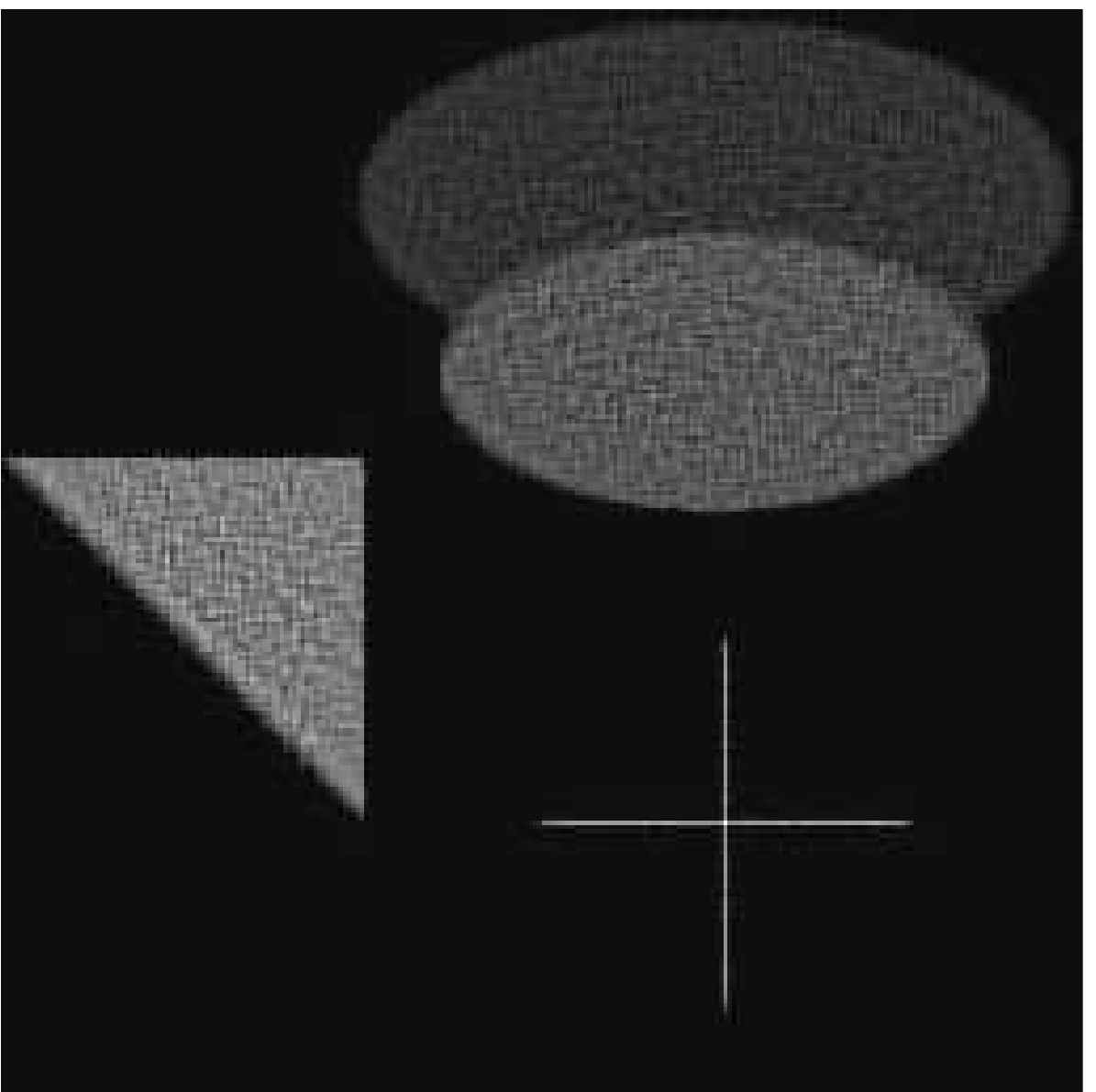}
        \end{subfigure}
        \begin{subfigure}{0.16\textwidth}
        \centering
        \includegraphics[width=1\textwidth]{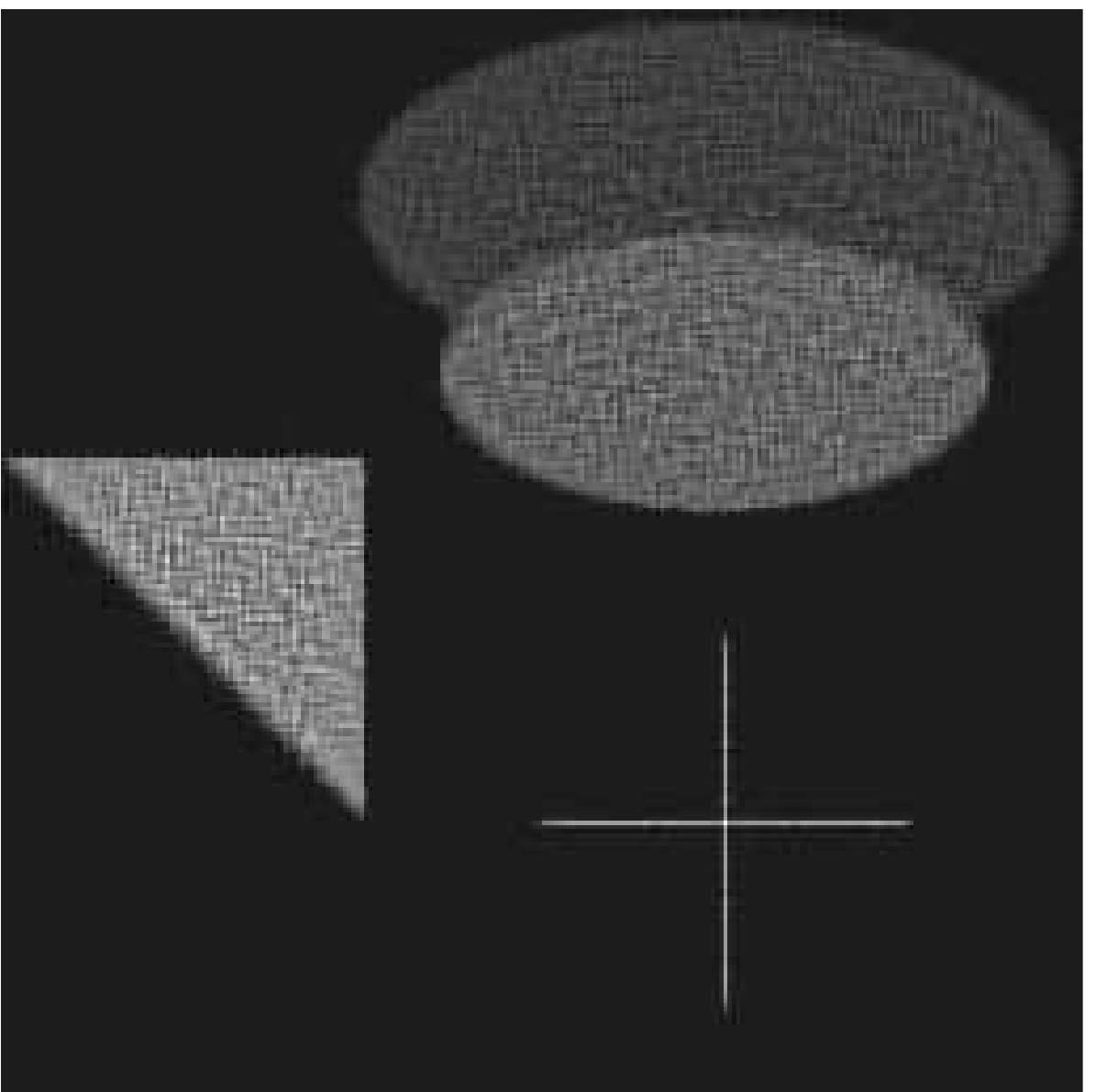}
        \end{subfigure}
   \\$T = 100$ &
        \begin{subfigure}{0.16\textwidth}
        \centering
        \includegraphics[width=1\textwidth]{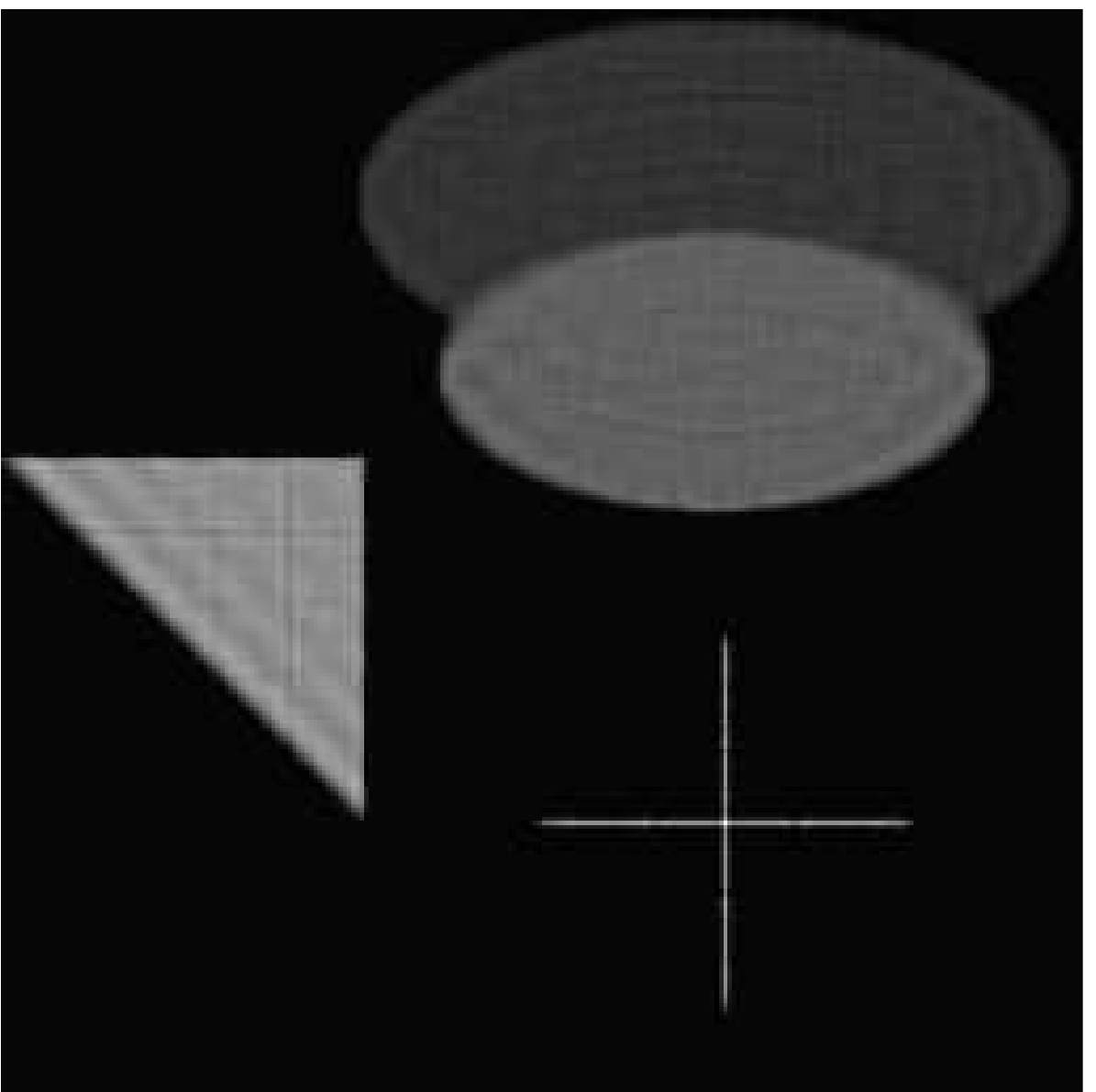}
        \end{subfigure}
        \begin{subfigure}{0.16\textwidth}
        \centering
        \includegraphics[width=1\textwidth]{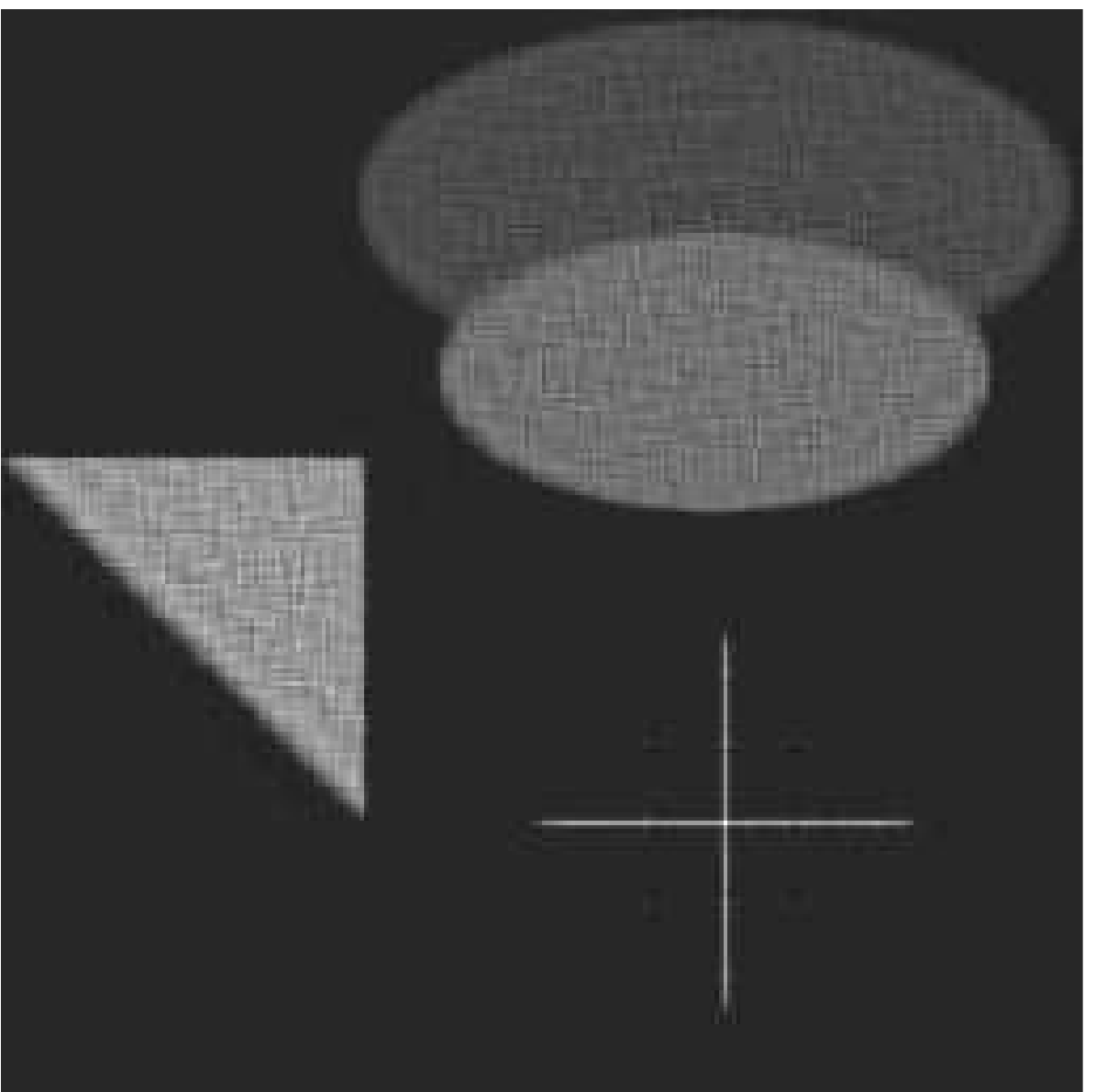}
        \end{subfigure}
        \begin{subfigure}{0.16\textwidth}
        \centering
        \includegraphics[width=1\textwidth]{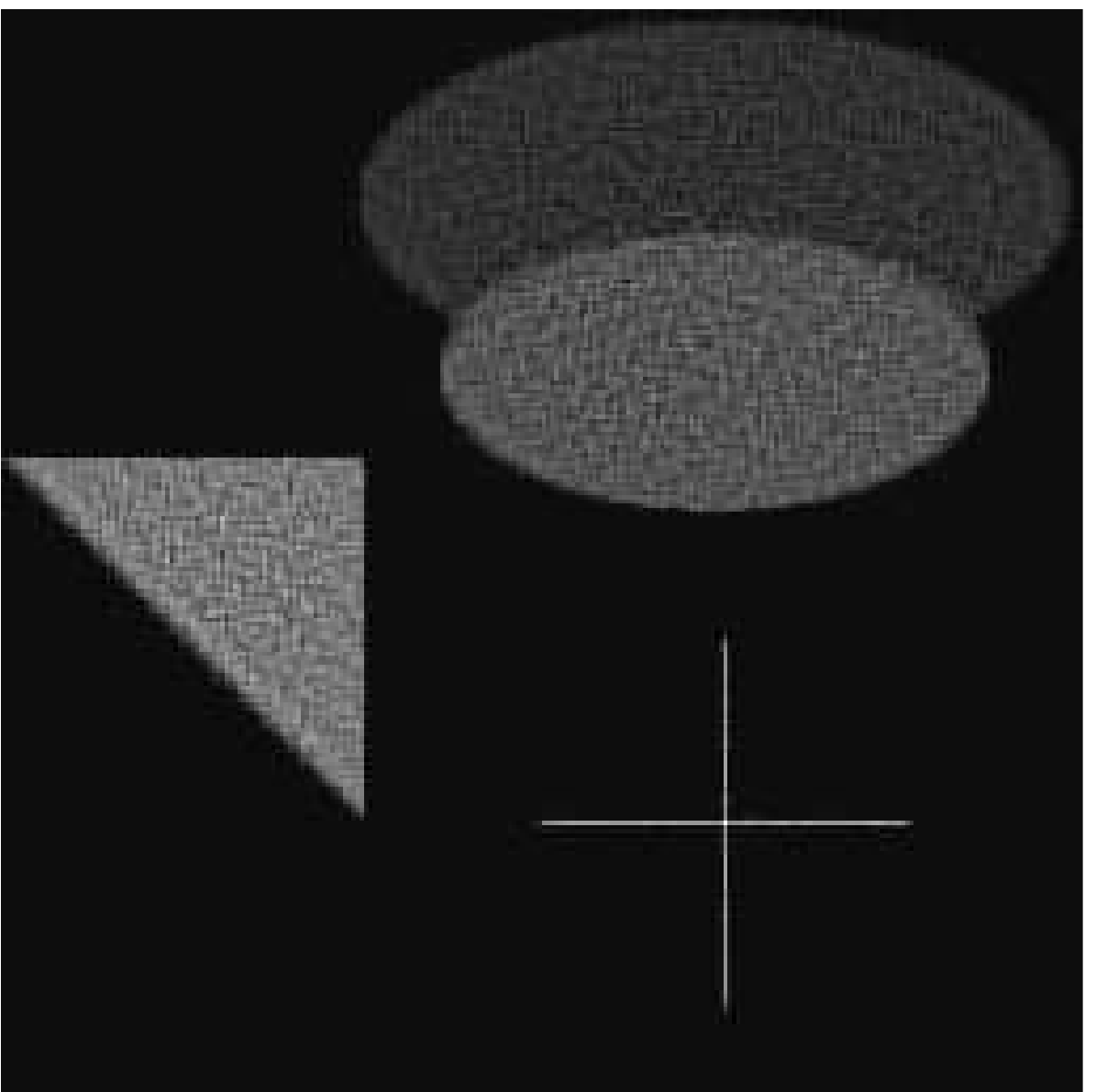}
        \end{subfigure}
        \begin{subfigure}{0.16\textwidth}
        \centering
        \includegraphics[width=1\textwidth]{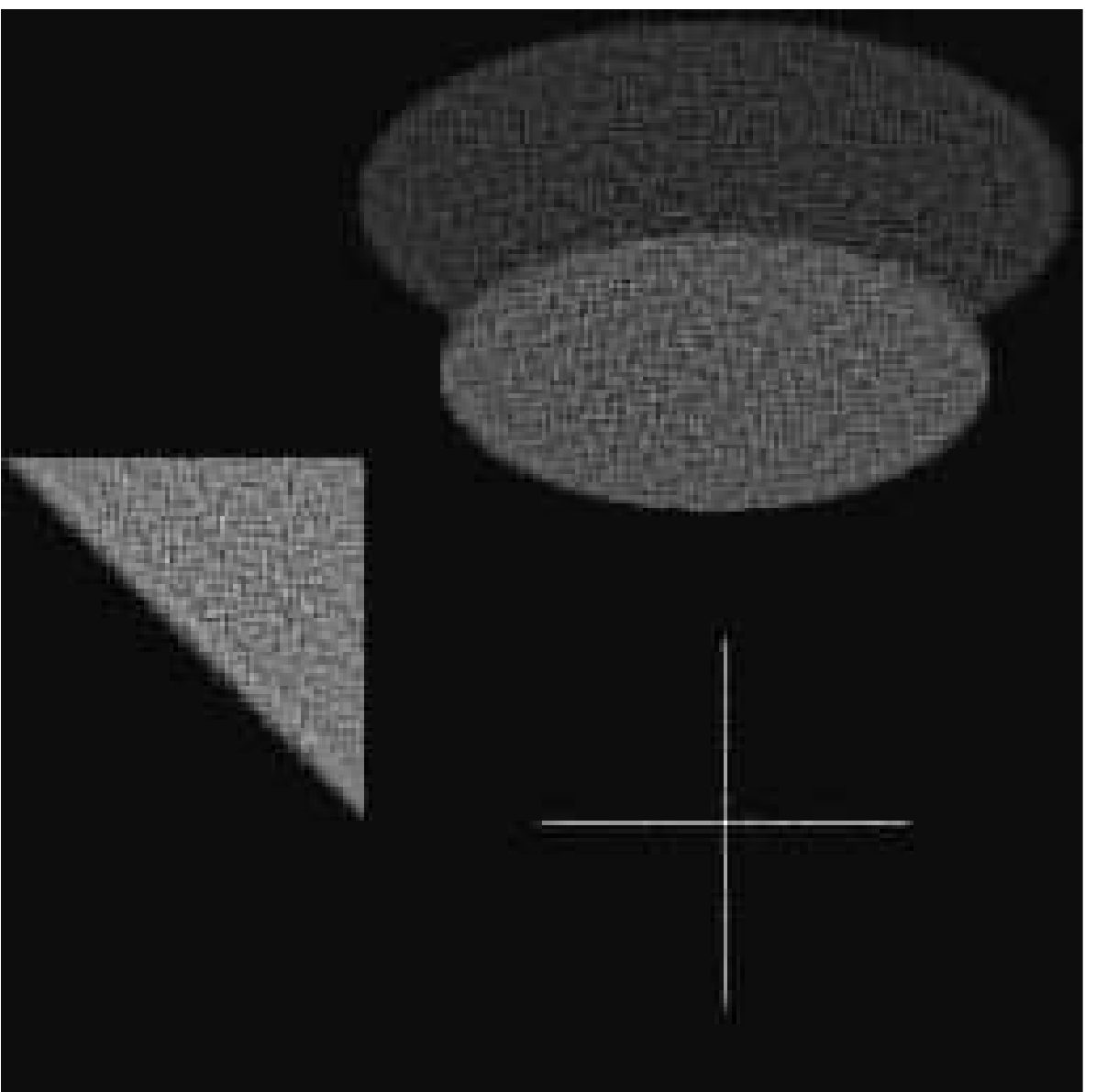}
        \end{subfigure}
        \begin{subfigure}{0.16\textwidth}
        \centering
        \includegraphics[width=1\textwidth]{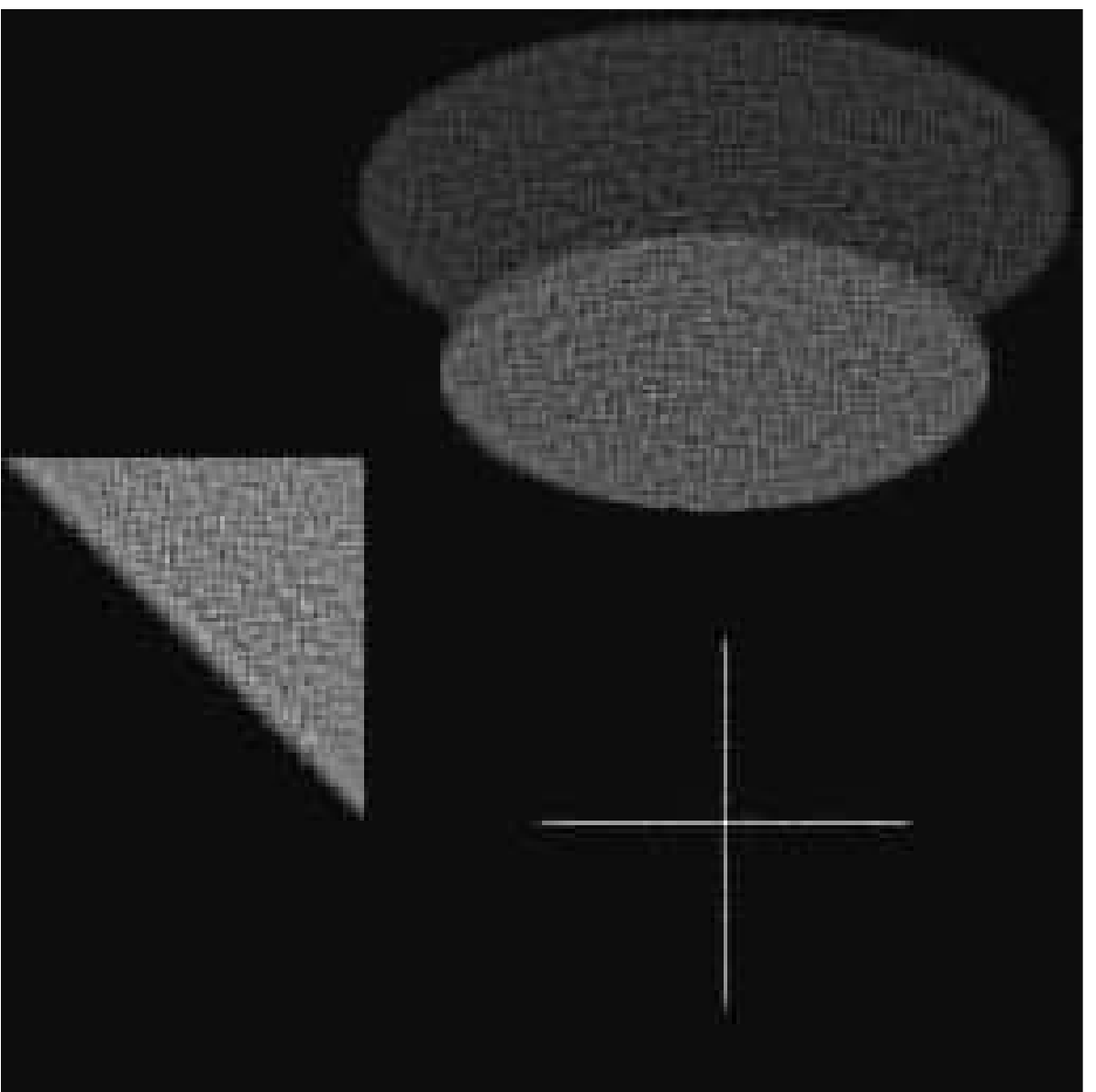}
        \end{subfigure}
  \end{tabular}
  \caption{\footnotesize Reconstruction of the  test image from Figure~\ref{fig:afcewvcawfecvwa} via FISTA, SGD+,    S2GD and mS2GD after $T=20, 60, 100$ epochs (one epoch corresponds to work equivalent to the computation of one gradient.)}
  \label{fig:adsfaewfawefa}
 \end{figure*}


\section{Conclusion}
\add{
We have proposed mS2GD---a mini-batch semi-stochastic gradient method---for minimizing a strongly convex composite function. Such optimization problems arise frequently in inverse problems in signal processing and statistics. Similarly to SAG, SVRG, SDCA and S2GD, our algorithm also outperforms existing deterministic method such as ISTA and FISTA. Moreover, we have shown that the method is by design amenable to a simple parallel implementation. Comparisons to state-of-the-art algorithms suggest that mS2GD, with a small-enough mini-batch size, is competitive in theory and faster in practice than other competing methods even without parallelism. The method can be efficiently implemented for sparse data sets.
}

\appendices

\section{Technical Results}

\begin{lemma}[Lemma 3.6 in \cite{proxsvrg}]\label{nonexpansiveness}
Let $R$ be a closed convex function on $\R^d$ and $\x,\y \in \dom(R)$, then
$
\| \prox_R(\x) - \prox_R(\y)\| \leq \|\x-\y\|.
$
\end{lemma}
Note that  contractiveness of the proximal operator is a standard result in optimization literature \cite{moreau1962,rockafellar1970}.

\begin{lemma}\label{randvar}
Let $\{\xi_i\}_{i=1}^n$ be  vectors in $\R^d$
and $\bar{\xi} \eqdef \frac1n \sum_{i=1}^n \xi_i \in \R^d$.
Let $\hat S$ be  a random subset of $\setn$ of size $\tau$, chosen uniformly at random from all subsets of this cardinality. Taking expectation with respect to $\hat{S}$, we have
\begin{equation}
\label{eq:varianceBound}
 \textstyle 
\Exp \left[ \left\|\frac1\tau \sum_{i\in \hat S} \xi_i - \bar{\xi}  \right\|^2 \right]
\leq
\frac1{n\tau}
\frac{  n-\tau}{ (n-1)}
\sum_{i=1}^n \left\| \xi_i\right\|^2.
\end{equation}

\end{lemma}

Following from the proof of Corollary 3.5  in \cite{proxsvrg}, by applying Lemma \ref{randvar} with $\xi_i := \nabla f _{i}(y_{k, t}) - \nabla f _{i}(\x_k)$, we have the bound for variance as follows.

\begin{theorem}
[Bounding Variance]\label{boundvariance} Let  $\alpha(b)\eqdef \tfrac{n-b}{b (n-1)}$. Considering the definition of $G_{k, t}$ in Algorithm \ref{alg:mS2GD}, conditioned on $y_{k,t}$, we have $\Exp[G_{k, t}]=\nabla F (y_{k, t})$ and the variance satisfies,
\begin{align}
&\Exp\left[ \|G_{k, t} - \nabla F(y_{k, t})\|^2 \right] 
\nonumber
\\ &\leq   
4L \alpha(b)    [P(y_{k, t})-P(\x_*) + P(\x_k)-P(\x_*)].
  \label{eq:fa<wvgfawvgfwagva}
\end{align}

\end{theorem}

   
\section{Proofs} 

\subsection{Proof of Lemma \ref{randvar}}
As in the statement of the lemma, by $\Exp[\cdot]$ we denote expectation with respect to the random set $\hat S$. First, note that
\begin{align*}
 \eta \eqdef & \textstyle \Exp\left[\left\|\frac1\tau \sum_{i\in \hat S} \xi_i - \bar{\xi} \right\|^2 \right]
= \Exp\left[ 
  \frac1{\tau^2} \left\|\sum_{i\in \hat S} \xi_i \right\|^2 \right]
  - \|\bar{\xi} \|^2 \\
&= \textstyle \frac1{\tau^2}
\Exp\left[ 
   \sum_{i\in \hat S}\sum_{j\in \hat S} \xi_i ^T \xi_j  \right]
  - \|\bar{\xi} \|^2.
\end{align*}
If we let $C\eqdef \|\bar{\xi} \|^2=  \frac1{n^2} \left( \sum_{i,j} \xi_i^T\xi_j \right)$, we can thus write
\begin{align*}
\eta &=
\textstyle
\tfrac1{\tau^2}
\left( 
  \tfrac{\tau (\tau-1)}{n(n-1)}  \sum_{i \neq j} \xi_i ^T \xi_j 
  + \tfrac{\tau}{n} \sum_{i=1}^n \xi_i^T\xi_i
  \right)
  - C
\\
=& 
\textstyle
\tfrac1{\tau^2}
\left( 
  \tfrac{\tau (\tau-1)}{n(n-1)}  \sum_{i, j} \xi_i ^T \xi_j 
  + \left(\tfrac{\tau}{n}-\tfrac{\tau (\tau-1)}{n(n-1)}\right) \sum_{i=1}^n \xi_i^T\xi_i
  \right)  - C
\\
=&
\textstyle
\tfrac1{n\tau}
\left[
-
\left(
-  \tfrac{(\tau-1)}{(n-1)} 
+\tfrac{\tau}{n}
\right)
\sum_{i, j} \xi_i ^T \xi_j 
  +  
  \tfrac{  n-\tau }{n-1}    \sum_{i=1}^n \xi_i^T\xi_i
\right]
\\
=&
\textstyle
\tfrac1{n\tau}
\tfrac{  n-\tau}{ (n-1)}
\left[
      \sum_{i=1}^n \xi_i^T\xi_i
      -
\tfrac1n
\sum_{i, j} \xi_i ^T \xi_j 
\right]
\leq
\textstyle
\tfrac1{n\tau}
\tfrac{  n-\tau}{ (n-1)}
\sum_{i=1}^n \left\| \xi_i\right\|^2,
\end{align*}
where in the last step we have used the bound
$
\tfrac1n
\sum_{i, j} \xi_i ^T \xi_j
= n
\left\|
\sum_{i=1}^n \tfrac1n\xi_i
\right\|^2
\geq 0.$

\subsection{Proof of Theorem \ref{s2convergence}}

The proof is following the steps in \cite{proxsvrg}.
For convenience, let us define the stochastic gradient mapping
\begin{equation*}
d_{k, t} = \tfrac1\stepsize(y_{k, t}- y_{k, t+1}) = \tfrac1\stepsize(y_{k, t} - \prox_{\stepsize R}(y_{k, t}-\stepsize G_{k, t})),
\end{equation*}
then the iterate update can be written as $y_{k, t+1} = y_{k, t} - \stepsize d_{k, t}.$ Let us estimate the change of $\|y_{k, t+1}-\x_*\|$. It holds that
\begin{align*}
&\|y_{k, t+1}-\x_*\|^2 = \|y_{k, t}-\stepsize d_{k, t} - \x_*\|^2
\\
&= \|y_{k, t}-\x_*\|^2 - 2\stepsize d_{k, t}^T(y_{k, t-1}-\x_*) + \stepsize^2\|d_{k, t}\|^2.
\tagthis\label{eq:safwavgfwafa}
\end{align*}

Applying Lemma 3.7 in \cite{proxsvrg} (this is why we need to assume that $h\leq 1/L$) with $\x=y_{k, t}$, $v=G_{k, t}$, $\x^+=y_{k, t+1}$, $g=d_{k, t}$, $\y=\x_*$ and $\Delta = \Delta_{k, t} = G_{k, t}-\nabla F (y_{k, t})$, we get
\begin{gather*}
-d_{k, t}^T(y_{k, t}-\x_*) + \tfrac\stepsize2\|d_{k, t}\|^2 
\leq  P(\x_*)-P(y_{k, t+1})
\\
-\tfrac{\mu_F}2\|y_{k, t}-\x_*\|^2 - \tfrac{\mu_R}2\|y_{k, t+1}-\x_*\|^2 - \Delta_{k, t}^T(y_{k, t+1}-\x_*),
\tagthis\label{eq:vaawfwavfwafdewafca}
\end{gather*}
and therefore,
\add{
\begin{gather*}
\|y_{k, t+1}-\x_*\|^2 \overset{\eqref{eq:safwavgfwafa},\eqref{eq:vaawfwavfwafdewafca}}{\leq}  
2\stepsize
\left( 
P(\x_*)-P(y_{k, t+1}) \right.
\\
\left. -    \Delta_{k, t}^T(y_{k, t+1}-\x_*)
\right)  + \|y_{k, t}-\x_*\|^2
\\
=
 \|y_{k, t}-\x_*\|^2 
 - 2\stepsize\Delta_{k, t}^T(y_{k, t+1}-\x_*) 
 \\ - 2\stepsize[P(y_{k, t+1})-P(\x_*)].
 \tagthis\label{eq:safawvfdwavfda}
\end{gather*}
}

In order to bound $-\Delta_{k, t}^T(y_{k, t+1}-\x_*)$, let us define the proximal full gradient update as\footnote{Note that this quantity is never computed during the algorithm. We can use it in the analysis nevertheless.}
$
\bar{y}_{k, t+1} = \prox_{\stepsize R}(y_{k, t} - \stepsize\nabla F (y_{k, t})).
$
We get

\begin{align*}
& - \Delta_{k, t}^T(y_{k, t+1}-\x_*) 
\\
&=
 - \Delta_{k, t}^T(y_{k, t+1}-\bar{y}_{k, t+1}) 
 - \Delta_{k, t}^T(\bar{y}_{k, t+1}-\x_*)
\nonumber
\\
&= 
- \Delta_{k, t}^T(\bar{y}_{k, t+1}-\x_*)
-  \Delta_{k, t}^T 
[
\prox_{\stepsize R}(y_{k, t} - \stepsize G_{k, t})
\\
&\qquad - \prox_{\stepsize R}(y_{k, t-1} - \stepsize\nabla F (y_{k, t-1}))
]
\nonumber
\end{align*}

Using Cauchy-Schwarz and Lemma~\ref{nonexpansiveness}, we  conclude that
 
\begin{align*}
& - \Delta_{k, t}^T(y_{k, t+1}-\x_*) 
\\
&\leq 
 \|\Delta_{k, t}\| \| (y_{k, t} - \stepsize G_{k, t})-(y_{k, t}-\stepsize\nabla F (y_{k, t}))\|
\\
&\qquad - \Delta_{k, t}^T(\bar{y}_{k, t+1}-\x_*), 
\nonumber
\\
&=  \stepsize \|\Delta_{k, t}\|^2- \Delta_{k, t}^T(\bar{y}_{k, t+1}-\x_*).
\tagthis\label{eq:asfaevfwavfa}
\end{align*}

\add{
Further, we obtain 
$ \|y_{k, t+1}-\x_*\|^2 
\overset{\eqref{eq:asfaevfwavfa},\eqref{eq:safawvfdwavfda}}{\leq}
  \left\|
y_{k, t}-\x_*
\right\|^2 
+  2\stepsize 
\left(   \stepsize \|\Delta_{k, t}\|^2- \Delta_{k, t}^T(\bar{y}_{k, t+1}-\x_*) 
- [P(y_{k, t+1})-P(\x_*)]\right). 
$
}
By taking expectation, conditioned on $y_{k, t}$\footnote{For simplicity, we omit the $\Exp[ \cdot \,|\, y_{k, t}]$ notation in further analysis} we obtain
\add{
\begin{gather*}
\Exp[\|y_{k, t+1} -\x_*\|^2]
  \overset{\eqref{eq:asfaevfwavfa},\eqref{eq:safawvfdwavfda}}{\leq}
  \left\|
y_{k, t}-\x_*
\right\|^2 
\\  +  2\stepsize 
\left(   \stepsize \Exp[\|\Delta_{k, t}\|^2] - \Exp[P(y_{k, t+1})-P(\x_*)]\right),
\tagthis\label{eq:asfavgfwavgfwavgfa}
\end{gather*}
}
where we have used that
$\Exp [\Delta_{k, t}] = \Exp[G_{k, t}] - \nabla F (y_{k, t})=0$
and hence $\Exp[- \Delta_{k, t}^T(\bar{y}_{k, t+1}-\x_*)] = 0$\footnote{$\bar{y}_{k, t+1}$ is constant, conditioned on $y_{k, t}$}.
Now, if we substitute \eqref{eq:fa<wvgfawvgfwagva}
into \eqref{eq:asfavgfwavgfwavgfa} and decrease index $t$ by 1, we obtain
\add{
\begin{gather*}
\Exp[\|y_{k, t} -\x_*\|^2]
\overset{\eqref{eq:asfaevfwavfa},\eqref{eq:safawvfdwavfda}}{\leq}
 \left\|
y_{k, t-1}-\x_*
\right\|^2 
\\
+  \theta  [P(y_{k, t-1})-P(\x_*)
+ P(\x_k)-P(\x_*)] \\
- 2h\Exp[P(y_{k, t})-P(\x_*)], \tagthis\label{eq:vgafrvgwaevgfa}
\end{gather*}
}
where \begin{equation}\label{eq:theta}\theta\eqdef 8 L h^2 \alpha(b)\end{equation} and $\alpha(b)=\frac{n-b}{b (n-1)}$.
\add{
Note that 
\eqref{eq:vgafrvgwaevgfa}
is equivalent to
\begin{gather*}
\Exp[\|y_{k, t}-\x_*\|^2]
+  2\stepsize 
(  \Exp[P(y_{k, t})-P(\x_*)])
\\
 \leq
  \left\|
y_{k, t-1}-\x_*
\right\|^2 
\\
+ \theta 
     \left(P(y_{k, t-1})-P(\x_*) 
     + P(\x_k)-P(\x_*)\right). \tagthis\label{eq:aswdevfwavfaawca}
\end{gather*}
}

Now, by the definition of 
$\x_k$ \add{in Algorithm~\ref{alg:mS2GD}
} 
we have that
\add{
\begin{align*}
\Exp[P(\x_{k+1})] 
&=
\tfrac1{m}
\sum_{t=1}^m 
   \Exp[P(y_{k, t})]. \tagthis\label{eqn:s2exp}
\end{align*}
}
By summing \eqref{eq:aswdevfwavfaawca}
 for $1\leq t \leq m$, we  get on the left hand side
\begin{align*}
LHS =& \sum_{t=1}^m 
   \Exp[\|y_{k, t}-\x_*\|^2] 
+      2\stepsize 
 \Exp[P(y_{k, t})-P(\x_*)]  \tagthis\label{eqn:s2LHS}
\end{align*}
and for the right hand side we have:
\begin{gather*}
RHS  = \sum_{t=1}^m 
\left\{  \Exp \|y_{k, t-1}-\x_*\|^2  + 
  \right.
 \\
\left.  \add{\theta } \Exp[P(y_{k, t-1}) - P(\x_*) + P(\x_k) - P(\x_*)] \right\}
\\
\leq \sum_{t=0}^{m-1} 
 \Exp\|y_{k, t}-\x_*\|^2 
+  \add{ \theta } 
   \sum_{t=0}^{m} \Exp[P(y_{k, t}) - P(\x_*)  ]  
\\ +  
\add{\theta}
 \Exp[  P(\x_k) - P(\x_*)]
 \add{m}. \tagthis\label{eqn:s2RHS}
\end{gather*}
Combining \eqref{eqn:s2LHS} and \eqref{eqn:s2RHS}
and using the fact that $LHS\leq RHS$, we have
\begin{gather*}
\Exp[\|y_{k, m} - \x_*\|^2]  + 
\add{2\stepsize}
\sum_{t=1}^m  \Exp[P(y_{k, t})-P(\x_*)]
\\
\leq
 \Exp\|y_{k, 0}-\x_*\|^2  +
  \add{\theta}
 \Exp[  P(\x_k) - P(\x_*)]
 \add{m}
\\ 
+  \add{\theta }
 \sum_{t=1}^{m} 
   \Exp[P(y_{k, t}) - P(\x_*)  ]
 + \add{\theta}
   \Exp[P(y_{k, 0}) - P(\x_*)  ].
  \end{gather*}  
Now, using \eqref{eqn:s2exp}, we obtain
\begin{gather*}
\Exp[\|y_{k, m}-\x_*\|^2] 
+ 
\add{2\stepsize m}
  \Exp[P(\x_{k+1}) -P(\x_*)]
\\
\leq
 \Exp\|y_{k, 0}-\x_*\|^2 
 +
  \add{\theta m} \Exp[  P(\x_k) - P(\x_*)]
\\
+ 
\add{\theta m} 
      \Exp[P(\x_{k+1}) - P(\x_*)  ]
+  
 \add {\theta} 
   \Exp[P(y_{k, 0}) - P(\x_*)  ].   
\tagthis \label{eq:vfrwavfwafaewfcwa}
  \end{gather*}


Strong convexity~\eqref{strongconv} and optimality of $\x_*$ imply that $0\in \partial P(\x_*)$, and hence for all $\x \in \R^d$ we have
\begin{equation*}
\|\x - \x_*\|^2\leq\tfrac2\mu [P(\x)-P(\x_*)]\tagthis\label{eqn:strongconvex}.
\end{equation*}
Since  $\Exp\|\y_{k, m} - \x_*\|^2\geq 0$ and $\y_{k, 0} = \x_k$, by combining \eqref{eqn:strongconvex}
and \eqref{eq:vfrwavfwafaewfcwa} we get
\begin{gather*}
\add{m(2h -\theta)}
 \Exp[P(\x_{k+1}) -P(\x_*)]
\\
\leq
( P(\x_k)-P(\x_*))
\left( \tfrac2{\mu} 
  +  \add {\theta} 
 \left(\add{m+1}  \right) \right).
\end{gather*} 
Notice that in view of our assumption on $h$ and  \eqref{eq:theta}, we have \add{$
2h>\theta$},  and hence
\begin{equation*}
\Exp[P(\x_{k+1})-P(\x_*)]\leq \decrease[P(\x_k) - P(\x_*)],
\end{equation*}
where 
\add{
$
\decrease = 
\tfrac{    2
  }
  {
 m \mu (2h-\theta)
  }
+
\tfrac{  \theta(m+1)   } { m(2h-\theta)}.
$
}
Applying the above linear convergence relation recursively with chained expectations, we finally obtain
$
\Exp [P(\x_k)-P(x_*)] \leq \decrease^k[P(\x_0) - P(x_*)].
$

 \subsection{Proof of Theorem \ref{thm:optimalM}}

Clearly,  if we choose some value of $\stepsize$ then the value of $m$ will be determined from \eqref{s2rho} (i.e. we need to choose $m$ such that we will get desired rate).
Therefore, $m$ as a function of $\stepsize$ obtained from \eqref{s2rho} is
\begin{equation}\label{eq:fsarewavgfwavgfw}
m(\stepsize) = \tfrac{1 + 4 \alpha(b) \stepsize^2 L \mu}
            {\stepsize \mu (\decrease - 4 \alpha(b) \stepsize L
   ( \decrease + 1 )  )}.
\end{equation}
Now, we can observe that the nominator is always positive
and the denominator is positive only if
$
\decrease   > 4 \alpha(b) \stepsize L
   ( \decrease + 1)
$, which implies
   $\tfrac{1}{4 \alpha(b) L
   } 
   \cdot  \tfrac{\decrease}{ \decrease + 1} >  \stepsize 
$ (note that $\tfrac{\decrease}{ \decrease + 1}\in [0,\frac12]$).
Observe that this condition is stronger than the one in the assumption of Theorem \ref{s2convergence}.
It is easy to verify that
$$\lim_{\stepsize \searrow 0} m(\stepsize) = +\infty,
\qquad
\lim_{\stepsize \nearrow \frac{1}{4 \alpha(b) L
   } 
   \cdot   \frac{\decrease}{ \decrease + 1}   } m(\stepsize) = +\infty.
$$
Also note that $m(\stepsize)$ is differentiable (and continuous) at any 
$\stepsize \in (0,\frac{1}{4 \alpha(b) L
   } 
   \cdot   \frac{\decrease}{ \decrease + 1} )=:I_\stepsize$.
The derivative of $m$ is given by
$
m'(\stepsize) = 
\tfrac{-\decrease + 4 \alpha(b) \stepsize L (2 + (2 + \stepsize \mu) \decrease)}
     {\stepsize^2 \mu (\decrease - 4 \alpha(b) \stepsize L (1 + \decrease))^2}.
$ Observe that $m'(\stepsize)$ is defined and continuous for any $\stepsize \in I_\stepsize$.
Therefore there have to be some stationary points (and in case that there is just on $I_\stepsize$) it will be the global minimum on $I_\stepsize$.
The FOC gives
\begin{align*}
\tilde \stepsize^b
&=
 \tfrac{-2 \alpha(b) L (1 + \decrease) + \sqrt{
 \alpha(b) L (\mu \decrease^2 + 4 \alpha(b) L (1 + \decrease)^2)}}
 {2 \alpha(b) L \mu \decrease}
 \\
&=
 \sqrt{ 
 \tfrac{ 
   1  }
 {4 \alpha(b) L  \mu   } 
+
 \tfrac{ 
    (1 + \decrease)^2 }
 {  \mu^2 \decrease^2} 
 }
 -
  \tfrac{    1 + \decrease   }
 {  \mu \decrease}.
\tagthis \label{eq:avfawefwafewafwafe} 
\end{align*}
If this $\tilde\stepsize^b \in I_\stepsize$
and also $\tilde\stepsize^b \leq \frac1{L}$
then this is the optimal choice and plugging 
\eqref{eq:avfawefwafewafwafe}
into \eqref{eq:fsarewavgfwavgfw}
gives us \eqref{eq:vfewdfwaefvawvgfeefewafa}.

\paragraph{Claim \#1} It always holds that  $\tilde\stepsize^b \in I_\stepsize$.
We just need to verify that
$$
 \sqrt{ 
 \tfrac{ 
   1  }
 {4 \alpha(b) L  \mu   } 
+
 \tfrac{ 
    (1 + \decrease)^2 }
 {  \mu^2 \decrease^2} 
 }
 -
  \tfrac{    1 + \decrease   }
 {  \mu \decrease}
 < \tfrac{1}{4 \alpha(b) L
   } 
   \cdot   \tfrac{\decrease}{ \decrease + 1}, 
$$
which is equivalent to 
$
\mu \decrease^2 + 4 \alpha(b) L (1 + \decrease)^2
 > 
 2 (1 + \decrease) \sqrt{\alpha(b) L (\mu \decrease^2 + 4 \alpha(b) L (1 + \decrease)^2)}. $
Because both sides are positive, we can square them to obtain the equivalent condition
$$
\mu \decrease^2 (\mu \decrease^2 + 4 \alpha(b) L (1 + \decrease)^2) > 0.
$$
\paragraph{Claim \#2} If $\tilde \stepsize^b > \frac1{L}$ then 
$\stepsize^b_* = \frac1L$.
The only detail which needs to be verified is that the denominator
of \eqref{eq:fasfawefwafewa} is positive
(or equivalently we want to show that
$ \decrease > 4 \alpha(b)   (1+\decrease)$.
To see that, we need to realize that in that case
we have $\tfrac1L \leq \tilde \stepsize^b \leq \frac{1}{4 \alpha(b) L
   } 
   \cdot   \frac{\decrease}{ \decrease + 1}$,
which implies that
$4 \alpha(b)   (1+\decrease) < \decrease.
$

\add{
\subsection{Proof of Corollary \ref{thm:minibatch}}\label{thm:proof3}

By substituting definition of $\tilde \stepsize^b$ in Theorem~\ref{thm:optimalM}, we get
\begin{equation}\label{eqn:b0}
\tilde \stepsize^b < \tfrac1L
\quad \Longleftrightarrow \quad 
b < b_0 \eqdef \tfrac{8\rho n\kappa + 8n\kappa + 4\rho n}
{\rho n\kappa + (7\rho+8)\kappa + 4\rho},
\end{equation}
where $\kappa = L/\mu$. Hence, it follows that if $b < \lceil b_0 \rceil$, then $h^b = \tilde \stepsize^b$ and $m^b$ is defined in~\eqref{eq:vfewdfwaefvawvgfeefewafa}; otherwise, $h^b=\frac1L$ and $m^b$ is defined in~\eqref{eq:fasfawefwafewa}. Let $e$ be the base of the natural logarithm. By selecting $b_0 =\frac{8 n\kappa + 8 e n\kappa + 4n}
{n\kappa + (7+8e)\kappa + 4}$, choosing mini-batch size $b<\lceil b_0 \rceil$, and running the inner loop of mS2GD for
\begin{equation}\label{eqn:maxiter}
m_b = \left\lceil 8 e\alpha(b)\kappa \left(e + 1 + \sqrt{
 \tfrac1{4 \alpha(b) \kappa} +   (1 + e)^2}\right) \right\rceil
 \end{equation}
 iterations with constant stepsize 
\begin{equation}\label{eqn:stepsizehb}
h^b = \sqrt{ \left( \tfrac{1+e}{\mu} \right)^2 + \tfrac1{4\mu\alpha(b)L}} - \tfrac{1+e}{\mu},
\end{equation}
we can achieve a convergence rate 
\begin{align*}
\rho &\overset{\eqref{s2rho}}{=}
\tfrac{    1
  }
  {
  m_b 
\stepsize^b \mu
(
1
-
  4\stepsize^b  L \alpha(b)  
)
  }
+
\tfrac{      
  4\stepsize^b  L \alpha(b) 
\left(      
 m_b
+ 
  1
\right)  
  }
  {
  m_b 
(
1
-
  4\stepsize^b   L \alpha(b)  
)
  } \overset{\eqref{eqn:maxiter}, \eqref{eqn:stepsizehb}}{=}
\tfrac{1}{e}. \tagthis\label{eqn:convergencerate}
 \end{align*}
 Since
$k = \left\lceil \log(1 / \epsilon) \right\rceil \geq \log(1 / \epsilon)$ if and only if $ 
e^k  \geq 1 / \epsilon$ if and only if
$ 
e^{-k}  \leq  \epsilon,
$
 we can conclude that $\rho^k \overset{\eqref{eqn:convergencerate}}{=} (e^{-1})^k = e^{-k} \leq \epsilon$. Therefore, running mS2GD for $k$ outer iterations achieves $\epsilon$-accuracy solution defined in \eqref{eq:epsilonaccuracy}. Moreover, since in general $\kappa \gg e, n\gg e$, it can be concluded that
$$ b_0 \overset{\eqref{eqn:b0}}{=}\tfrac{8(1+e) n\kappa + 4n}
{n\kappa + (7+8e)\kappa + 4} \approx 8 \left( e+1 \right) \approx 29.75,$$
then with the definition $\alpha(b) = \frac{(n-b)}{b(n-1)}$, we derive
 \begin{align*}
 & b m_b 
 \overset{\eqref{eqn:maxiter}}{=}
 \left\lceil 8 e\kappa\tfrac{(n-b)}{(n-1)} \left(e + 1 + \sqrt{
 \tfrac1{4 \alpha(b) \kappa} +   (1 + e)^2}\right) \right\rceil
\\&
\overset{1\leq b < 30}{\leq} 
 \left\lceil 8 e\kappa \left((e + 1) + \sqrt{
 \tfrac{b}{4 \kappa} +   (1 + e)^2}\right) \right\rceil = \mathcal{O}(\kappa),
 \end{align*}
so from \eqref{eqn:complexity}, the total complexity can be translated to
$ \mathcal{O} \left( (n + \kappa) \log(1 / \epsilon) \right).$
This result shows that we can reach efficient speedup by mini-batching as long as the mini-batch size is smaller than some threshold $b_0 \approx 29.75$, which finishes the proof for Corollary \ref{thm:minibatch}.

 }

\section{Proximal lazy updates for $\ell_1$ and $\ell_2$-regularizers}\label{section:proxl}

\subsection{Proof of Lemma~\ref{lemma:L2regularizer}}

For any $s \in \{1,2,\dots,\tau\}$
we have
$
\tilde y_s
=\prox_{hR}(\tilde y_{s-1} - hg)
= \beta 
  (\tilde y_{s-1} - hg),
$
where  $\beta \eqdef 1/(1+\lambda \stepsize)$.
Therefore,
\[
\textstyle \tilde y_{\tau} 
= \beta^\tau \tilde y_{0} 
- h\left(\sum_{j=1}^\tau \beta^j\right) g
=
\beta^{\tau} y  - \frac{\stepsize\beta}{1-\beta}
                \left[1 - \beta^{\tau} \right] g.
\]

\subsection{Proof of Lemma~\ref{lemma:L1regularizer}}

\begin{proof}
For any $s \in \{1,2,\dots,\tau\}$ and $j\in \{1,2,\dots,d\}$, 
\begin{align*}
\vc{\tilde y_{s}}{j}  &= \arg\min_{x\in\R} \frac12(x - \tilde y_{s-1}^j + hg^j)^2+ \lambda h |x| 
\\
&= 
\begin{cases}
\tilde y_{s-1}^j -  (\lambda+ g^j)h, &\text{ if } \tilde y_{s-1}^j  >(\lambda+g^j) h,\\
\tilde y_{s-1}^j + ( \lambda - g^j) h, &\text{ if }\tilde y_{s-1}^j   <-(\lambda-g^j) h,\\
0, &\text{ otherwise, }
\end{cases}
\\
&= 
\begin{cases}
\tilde y_{s-1}^j -  M, &\text{ if } \tilde y_{s-1}^j  > M,\\
\tilde y_{s-1}^j - m, &\text{ if }\tilde y_{s-1}^j   <m ,\\
0, &\text{ otherwise.}
\end{cases}
\end{align*}
where $M \eqdef (\lambda+g^j) h$, $m \eqdef -(\lambda-g^j) h$ and $M-m = 2\lambda h > 0$.  Now, we will distinguish several cases based on $\vc{g}{j}$:
\begin{enumerate}[topsep=1.5ex,itemsep=1.5ex]

\item[(1)] When $g^j \geq \lambda$, then $M>m = - (\lambda - g^j)h \geq 0$, thus by letting $p = \left\lfloor \frac{y^j}{M}\right\rfloor $, we have that: if $y^j < m$, then
$\tilde y^j_\tau = y^j - \tau m;$ if $m  \leq y^j < M$, then
$ \tilde y_{\tau}^j =  -(\tau-1)m;$ and if $y^j \geq M$, then
\begin{align*}
\tilde y_{\tau}^j &= 
\begin{cases}
y^j - \tau M, &\text{ if }\tau \leq p,\\
y^j - pM -(\tau - p)m, &\text{ if }\tau > p \text{ \& }y^j - pM < m,\\
-(\tau - p-1)m, &\text{ if }\tau > p \text{ \& }y^j - pM\geq m,
\end{cases}
\\&= 
\begin{cases}
y^j - \tau M, &\text{ if }\tau \leq p,\\
\min\{y^j - pM, m\} -(\tau - p)m, &\text{ if }\tau > p.
\end{cases}
\end{align*}
\item[(2)] When $-\lambda<g^j<\lambda$, then $M = (\lambda + g^j)h> 0,   m = - (\lambda - g^j)h  < 0$, thus we have that 
\begin{align*}
&\tilde y_{\tau}^j =
\begin{cases}
\max\{  y ^j - \tau M, 0\},&\text{ if } y^j\geq 0,\\
\min\{  y^j - \tau m, 0\},&\text{ if } y^j< 0.
\end{cases}
\end{align*}
\item[(3)] When $g^j \leq - \lambda$, then $m < M = (\lambda + g^j)h \leq 0$, thus by letting $q = \left\lfloor\frac{y^j}{m}\right\rfloor$, we have that: if $y^j \leq m$, then
 \begin{align*}
 \tilde y_\tau^j &= 
 \begin{cases}
 y^j - \tau m, &\text{ if }\tau \leq q,\\
  y^j - p m -(\tau-q)M, &\text{ if }\tau > q \text{ \& }y^j - qm > M,\\
 -(\tau-q - 1)M, &\text{ if }\tau > q \text{ \& }y^j - qm \leq M,
 \end{cases}
 \\&= 
  \begin{cases}
 y^j - \tau m, &\text{ if }\tau \leq q,\\
 \max\{y^j - q m, M\} -(\tau-q)M, &\text{ if }\tau > q;
 \end{cases}
 \end{align*}
 if $m < y^j \leq M$, then
$\tilde y_\tau^j = -(\tau-1)M;$ if $y^j >  M$, then
$\tilde y_\tau^j = y^j - \tau M.$
\end{enumerate}

Now, we will perform a few simplifications:
Case (1). When $y^j < M$, we can conclude that
$
\tilde y_\tau^j= \min\{y^j, m\} - \tau m.
$
Moreover, since the following equivalences hold if $g^j \geq \lambda$: $y^j \geq M \ \Leftrightarrow \ \tfrac{y^j}{M} \geq 1
\ \Leftrightarrow \ p\geq 1$, and
$y^j < M\ \Leftrightarrow \ \tfrac{y^j}{M} < 1
\ \Leftrightarrow \ p\leq 0$,
the situation  simplifies to
\begin{align*}
\tilde y_\tau^j&\quad =
\begin{cases}
y^j - \tau M, &\text{if }p\geq\tau,\\
\min\{y^j - pM, m\} -(\tau - p)m, &\text{if }1 \leq p<\tau,\\
\min\{y^j, m\} - \tau m, &\text{if }p\leq 0,
\end{cases}
\\&\quad =
\begin{cases}
y^j - \tau M, &\text{if }p\geq\tau,\\
\min\{y^j - [p]_+ M, m\} -(\tau - [p]_+)m, &\text{if }p<\tau,
\end{cases}
\end{align*}
where $[\cdot]_+ \eqdef \max\{\cdot, 0\}$. For Case (3), when $y^j > m$, we can conclude that
$
\tilde y_\tau^j = \max\{y^j, M\} - \tau M,
$
and in addition, the following equivalences hold when $g^j \leq -\lambda$:
\begin{align*}
y^j \leq m \ \Leftrightarrow \  \tfrac{y^j}{m} \geq 1
\  \Leftrightarrow \ q\geq 1,
\\
y^j > m\  \Leftrightarrow \  \tfrac{y^j}{m} < 1
\ \Leftrightarrow \ q\leq 0,
\end{align*}
which  summarizes the situation as follows:
\begin{align*}
\tilde y_\tau^j&\quad=
\begin{cases}
y^j - \tau m, &\text{if }q\geq\tau,\\
 \max\{y^j - q m, M\} -(\tau-q)M, &\text{if }1 \leq q<\tau,\\
 \max\{y^j, M\} - \tau M, &\text{if }q\leq 0,
\end{cases}
\\&\quad=
\begin{cases}
y^j - \tau m, &\text{if }q\geq\tau,\\
 \max\{y^j - [q]_+ m, M\} -(\tau-[q]_+)M, &\text{if }q<\tau.
\end{cases}\qedhere
\end{align*}

\end{proof}

\ifCLASSOPTIONcaptionsoff
  \newpage
\fi



%
%
%

\bibliographystyle{./IEEEtran}
\bibliography{./IEEEabrv,./literature}

\begin{thebibliography}{10}
\providecommand{\url}[1]{#1}
\csname url@samestyle\endcsname
\providecommand{\newblock}{\relax}
\providecommand{\bibinfo}[2]{#2}
\providecommand{\BIBentrySTDinterwordspacing}{\spaceskip=0pt\relax}
\providecommand{\BIBentryALTinterwordstretchfactor}{4}
\providecommand{\BIBentryALTinterwordspacing}{\spaceskip=\fontdimen2\font plus
\BIBentryALTinterwordstretchfactor\fontdimen3\font minus
  \fontdimen4\font\relax}
\providecommand{\BIBforeignlanguage}[2]{{%
\expandafter\ifx\csname l@#1\endcsname\relax
\typeout{** WARNING: IEEEtran.bst: No hyphenation pattern has been}%
\typeout{** loaded for the language `#1'. Using the pattern for}%
\typeout{** the default language instead.}%
\else
\language=\csname l@#1\endcsname
\fi
#2}}
\providecommand{\BIBdecl}{\relax}
\BIBdecl

\bibitem{nesterov2007acc}
Y.~Nesterov, ``Gradient methods for minimizing composite objective function,''
  \emph{CORE Discussion Papers}, 2007.

\bibitem{fista}
A.~Beck and M.~Teboulle, ``A fast iterative shrinkage-thresholding algorithm
  for linear inverse problems,'' \emph{SIAM J. Imaging Sciences}, vol.~2,
  no.~1, pp. 183--202, 2009.

\bibitem{RM1951}
H.~Robbins and S.~Monro, ``A stochastic approximation method,'' \emph{The
  Annals of Mathematical Statistics}, vol.~22, no.~3, pp. 400--407, 1951.

\bibitem{SAG}
N.~Le~Roux, M.~Schmidt, and F.~Bach, ``A stochastic gradient method with an
  exponential convergence rate for finite training sets,'' \emph{NIPS}, pp.
  2672--2680, 2012.

\bibitem{SDCA}
S.~Shalev-Shwartz and T.~Zhang, ``Stochastic dual coordinate ascent methods for
  regularized loss,'' \emph{JMLR}, vol.~14, no.~1, pp. 567--599, 2013.

\bibitem{SVRG}
R.~Johnson and T.~Zhang, ``Accelerating stochastic gradient descent using
  predictive variance reduction,'' \emph{NIPS}, pp. 315--323, 2013.

\bibitem{S2GD}
J.~Kone\v{c}n{\'y} and P.~Richt{\'a}rik, ``Semi-stochastic gradient descent
  methods,'' \emph{arXiv:1312.1666}, 2013.

\bibitem{nesterovRCDM}
Y.~Nesterov, ``Efficiency of coordinate descent methods on huge-scale
  optimization problems,'' \emph{SIAM J. Optimization}, vol.~22, pp. 341--362,
  2012.

\bibitem{richtarik}
P.~Richt{\'a}rik and M.~Tak{\'a}{\v{c}}, ``Iteration complexity of randomized
  block-coordinate descent methods for minimizing a composite function,''
  \emph{Mathematical Programming}, vol. 144, no. 1-2, pp. 1--38, 2014.

\bibitem{PCDM}
P.~Richt\'arik and M.~Tak\'a\v{c}, ``Parallel coordinate descent methods for
  big data optimization,'' \emph{Mathematical Programming, Series A}, pp.
  1--52, 2015.

\bibitem{necoara2014random}
I.~Necoara and A.~Patrascu, ``A random coordinate descent algorithm for
  optimization problems with composite objective function and linear coupled
  constraints,'' \emph{Comp. Optimization and Applications}, vol.~57, no.~2,
  pp. 307--337, 2014.

\bibitem{approx}
O.~Fercoq and P.~Richt{\'a}rik, ``Accelerated, parallel and proximal coordinate
  descent,'' \emph{arXiv:1312.5799}, 2013.

\bibitem{marecek2014distributed}
J.~Mare\v{c}ek, P.~Richt\'arik, and M.~Tak\'a\v{c}, ``Distributed block
  coordinate descent for minimizing partially separable functions,''
  \emph{arXiv:1406.0238}, 2014.

\bibitem{necoara2013distributed}
I.~Necoara and D.~Clipici, ``Distributed coordinate descent methods for
  composite minimization,'' \emph{arXiv:1312.5302}, 2013.

\bibitem{richtarik2013distributed}
P.~Richt{\'a}rik and M.~Tak{\'a}{\v{c}}, ``Distributed coordinate descent
  method for learning with big data,'' \emph{arXiv:1310.2059}, 2013.

\bibitem{fercoq2014fast}
O.~Fercoq, Z.~Qu, P.~Richt{\'a}rik, and M.~Tak{\'a}{\v{c}}, ``Fast distributed
  coordinate descent for non-strongly convex losses,'' in \emph{IEEE Workshop
  on Machine Learning for Signal Processing}, 2014.

\bibitem{ALPHA}
Z.~Qu and P.~Richt\'{a}rik, ``Coordinate descent with arbitrary sampling {I}:
  algorithms and complexity,'' \emph{arXiv:1412.8060}, 2014.

\bibitem{combettes2015}
P.~L. Combettes and J.-C. Pesquet, ``Stochastic quasi-fej{\'{e}}r
  block-coordinate fixed point iterations with random sweeping,'' \emph{{SIAM}
  Journal on Optimization}, vol.~25, no.~2, pp. 1221--1248, 2015.

\bibitem{zhangsgd}
T.~Zhang, ``Solving large scale linear prediction using stochastic gradient
  descent algorithms,'' in \emph{ICML}, 2004.

\bibitem{ma2015adding}
C.~Ma, V.~Smith, M.~Jaggi, M.~I. Jordan, P.~Richt{\'a}rik, and
  M.~Tak{\'a}{\v{c}}, ``Adding vs. averaging in distributed primal-dual
  optimization,'' \emph{ICML}, pp. 1973--1982, 2015.

\bibitem{COCOA}
M.~Jaggi, V.~Smith, M.~Tak{\'a}{\v{c}}, J.~Terhorst, T.~Hofmann, and M.~I.
  Jordan, ``Communication-efficient distributed dual coordinate ascent,''
  \emph{NIPS}, pp. 3068--3076, 2014.

\bibitem{takac2013ICML}
M.~Tak{\'a}\v{c}, A.~S. Bijral, P.~Richt{\'a}rik, and N.~Srebro, ``Mini-batch
  primal and dual methods for {SVM}s,'' \emph{ICML}, pp. 1022--1030, 2013.

\bibitem{nitanda}
A.~Nitanda, ``Stochastic proximal gradient descent with acceleration
  techniques,'' \emph{NIPS}, pp. 1574--1582, 2014.

\bibitem{rosasco2014}
L.~Rosasco, S.~Villa, and B.~C. V\~u, ``Convergence of stochastic proximal
  gradient algorithm,'' \emph{arXiv:1403.5074}, 2014.

\bibitem{Strohmer2009}
T.~Strohmer and R.~Vershynin, ``{A Randomized Kaczmarz Algorithm with
  Exponential Convergence},'' \emph{Journal of Fourier Analysis and
  Applications}, vol.~15, no.~2, pp. 262--278, 2009.

\bibitem{Needell2010}
D.~Needell, ``{Randomized Kaczmarz solver for noisy linear systems},''
  \emph{BIT}, vol.~50, no.~2, pp. 395--403, 2010.

\bibitem{Zouzias2012}
\BIBentryALTinterwordspacing
A.~Zouzias and N.~Freris, ``{Randomized Extended Kaczmarz for Solving
  Least-Squares},'' \emph{arXiv preprint arXiv:1205.5770}, p.~19, 2012.
  [Online]. Available: \url{http://arxiv.org/abs/1205.5770}
\BIBentrySTDinterwordspacing

\bibitem{Ma2015}
A.~Ma, D.~Needell, A.~Ramdas, and N.~A. Mar, ``{Convergence Properties of the
  Randomized Extended Gauss-Seidel and Kaczmarz methods},''
  \emph{arXiv:1503.08235}, pp. 1--16, 2015.

\bibitem{Oswald2015}
\BIBentryALTinterwordspacing
P.~Oswald and W.~Zhou, ``{Convergence analysis for Kaczmarz-type methods in a
  Hilbert space framework},'' \emph{Linear Algebra and its Applications}, vol.
  478, pp. 131--161, 2015. [Online]. Available:
  \url{http://linkinghub.elsevier.com/retrieve/pii/S0024379515001925}
\BIBentrySTDinterwordspacing

\bibitem{NeedellWard2015}
R.~{Needell, Deanna and Srebro, Nathan and Ward}, ``{Stochastic Gradient
  Descent, Weighted Sampling, and the Randomized Kaczmarz Algorithm},''
  \emph{Mathematical Programming}, 2015.

\bibitem{GowerRichtarik2015-linear}
R.~M. Gower and P.~Richt\'{a}rik, ``{Randomized Iterative Methods for Linear
  Systems},'' \emph{arXiv:1506.03296}, 2015.

\bibitem{Leventhal2008}
\BIBentryALTinterwordspacing
D.~Leventhal and A.~S. Lewis, ``{Randomized Methods for Linear Constraints:
  Convergence Rates and Conditioning},'' p.~22, 2008. [Online]. Available:
  \url{http://arxiv.org/abs/0806.3015}
\BIBentrySTDinterwordspacing

\bibitem{Spielman2006}
\BIBentryALTinterwordspacing
D.~a. Spielman and S.-H. Teng, ``{Nearly-Linear Time Algorithms for
  Preconditioning and Solving Symmetric, Diagonally Dominant Linear Systems},''
  vol.~35, no.~3, pp. 835--885, 2006. [Online]. Available:
  \url{http://arxiv.org/abs/cs/0607105}
\BIBentrySTDinterwordspacing

\bibitem{QUARTZ}
Z.~Qu, P.~Richt{\'a}rik, and T.~Zhang, ``Randomized dual coordinate ascent with
  arbitrary sampling,'' in \emph{NIPS}, 2015.

\bibitem{proxsvrg}
L.~Xiao and T.~Zhang, ``A proximal stochastic gradient method with progressive
  variance reduction,'' \emph{SIAM Journal on Optimization}, vol.~24, no.~4,
  pp. 2057--2075, 2014.

\bibitem{combettes2011}
P.~L. Combettes and J.-C. Pesquet, ``Proximal splitting methods in signal
  processing,'' in \emph{Fixed-Point Algorithms for Inverse Problems in Science
  and Engineering}.\hskip 1em plus 0.5em minus 0.4em\relax Springer, 2011,
  vol.~49, pp. 185--212.

\bibitem{parikh2014}
N.~Parikh and S.~Boyd, ``Proximal algorithms,'' \emph{Foundations and Trends in
  Opimization}, vol.~1, no.~3, pp. 127--239, 2014.

\bibitem{konecny2015mini}
J.~Kone{\v{c}}n{\'y}, J.~Liu, P.~Richt{\'a}rik, and M.~Tak{\'a}{\v{c}},
  ``Mini-batch semi-stochastic gradient descent in the proximal setting,''
  \emph{arXiv:1504.04407}, 2015.

\bibitem{nesterov2004convex}
Y.~Nesterov, \emph{Introductory Lectures on Convex Optimization: A Basic
  Course}.\hskip 1em plus 0.5em minus 0.4em\relax Kluwer, Boston, 2004.

\bibitem{langford2009}
J.~Langford, L.~Li, and T.~Zhang, ``Sparse online learning via truncated
  gradient,'' \emph{Journal of Machine Learning Research}, vol.~10, pp.
  777--801, 2009.

\bibitem{carpenter2008}
B.~Carpenter, ``Lazy sparse stochastic gradient descent for regularized
  multinomial logistic regression,'' \emph{Technical Report}, April 2008.

\bibitem{hansen2007regularization}
P.~C. Hansen, ``Regularization tools version 4.0 for matlab 7.3,''
  \emph{Numerical algorithms}, vol.~46, no.~2, pp. 189--194, 2007.

\bibitem{moreau1962}
J.~J. Moreau, ``Fonctions convexes duales et points proximaux dans un espace
  hilbertien,'' in \emph{Reports of the Paris Academy of Sciences}, ser. A,
  vol. 255, 1962, pp. 2897--2899.

\bibitem{rockafellar1970}
R.~T. Rockafellar, \emph{Convex Analysis}.\hskip 1em plus 0.5em minus
  0.4em\relax Princeton University Press, 1970.

\end{thebibliography}

\end{document}